\documentclass{article}

 \usepackage[preprint]{neurips_2026}

\usepackage[utf8]{inputenc} 
\usepackage[T1]{fontenc}    
\usepackage{hyperref}       
\usepackage{url}            
\usepackage{booktabs}       
\usepackage{amsfonts}       
\usepackage{nicefrac}       
\usepackage{microtype}      
\usepackage{xcolor}         
\usepackage{mathtools}
\usepackage{amsthm}
\usepackage{dsfont}
\usepackage{thmtools}
\usepackage{thm-restate}
\usepackage{amssymb}
\usepackage{amsfonts} 
\usepackage{subcaption}
\usepackage{enumitem}

\usepackage{todonotes}

\def\x#1{\hat x_{#1}}
\def\y#1{\hat y_{#1}}
\def\Dh#1{\Delta^h_{#1}}
\def\Dg#1{\Delta^g_{#1}}
\def\tildex#1{\tilde x_{#1}}
\def\tildey#1{\y{#1}}

\def\Rx{\mathbb R^{d_x}}
\def\Ry{\mathbb R^{d_y}}
\def\Reals{ \mathbb{R}}

\def\domX{\mathcal{X}}
\def\domY{\mathcal{Y}}

\def\F#1{f(x_{#1},y_{#1},\xi_{#1 +1} )}
\def\G#1{g(x_{#1},y_{#1},\xi_{#1 +1} )}
\def\H#1{h(y_{#1})}

\def\vf#1{v_f(x_{#1},y_{#1},\xi_{{#1}+1})}
\def\poissf#1#2{v_f(x_{#1},y_{#1},\xi_{{#2}})}

\def\vg#1{v_g(x_{#1},y_{#1},\xi_{{#1}+1})}
\def\poissg#1#2{v_g(x_{#1},y_{#1},\xi_{{#2}})}

\def\expvf#1#2#3{\mathbb{E}_{\xi' \sim K_{#1}(\xi_{{#2}}, \cdot)}[v_f(x_{#3},y_{#3},\xi')]}
\def\expvg#1#2#3{\mathbb{E}_{\xi' \sim K_{#1}(\xi_{{#2}}, \cdot)}[v_g(x_{#3},y_{#3},\xi')]}

\def\barF#1{\overline f(x_{#1},y_{#1})}
\def\barG#1{\overline g(x_{#1},y_{#1})}
\def\GH#1{\overline g(\H{#1},y_{#1})}

\def\nablafx{\nabla_{x} \overline f}

\def\nablagy{\nabla_{y} \overline g}

\def\boundF{\mathrm{M}_f}
\def\boundG{\mathrm{M}_g}
\def\boundH{\mathrm{M}_h}
\def\boundXi{\mathrm{M}_\xi}
\def\boundXY{\mathrm{M}_{\mathcal S}}

\def\muf{\mu_f}
\def\mug{\mu_g}
\def\Lf {L_f}
\def\Lg {L_g}
\def\Lh {L_h}

\def\Lvf {L_{v_f}}
\def\Lvg {L_{v_g}}
\def\Lk {L_{K}}

\def\Kernel#1{K_{(x_{#1},y_{#1})}}
\def\Stationary#1{\pi_{(x_{#1},y_{#1})}}

\newcommand\proba[1]{\mathbb{P}[#1]}
\def\Exp{\mathbb{E}}
\def\opnorm#1{\|#1\|_{\text{op}}}
\def\indicator#1{\mathds{1} _{\{#1\}}}
\def\Expect#1{\Exp \left [ #1 \right ]} 
\def\brackets#1{\left [ #1 \right ]}

\def\parentheses#1{\left ( #1\right ) }
\def\abs#1{\left | #1\right | }
\def\vector#1#2{\brackets{\begin{array}{c} 
         #1 \\
         #2
    \end{array}}}
\def\blockMatrix#1#2#3#4{\brackets{\begin{array}{cc}
         #1 & #2 \\
         #3 & #4
    \end{array}}}

\def\inner#1#2{\left \langle #1  , #2 \right\rangle}

\def\norm#1{\left \| #1 \right\|}
\def\opnorm#1{\norm{#1}_{\text{op}}}
\def\inducednorm#1#2{\left \| #1 \right\|_{#2}}
\def\LyapF{\mathrm L}
\def\LyapG{\mathrm G}
\def\LyapFNorm#1{\inducednorm{#1}{\LyapF}}
\def\LyapGNorm#1{\inducednorm{#1}{\LyapG}}

\def\jointbias#1{\brackets{\begin{array}{c}
         \Expect{\tildex{#1}}  \\
         \Expect{\tildey{#1}} 
    \end{array}}}
\def\bias#1{\brackets{\begin{array}{c}
         b^x_{#1}  \\
         b^y_{#1}
    \end{array}}}

\def\ricattibias#1{\brackets{\begin{array}{c}
        r^x_{#1}  \\
        r^y_{#1}
\end{array}}}

\def\sylvesterbias#1{\brackets{\begin{array}{c}
        s^x_{#1}  \\
        s^y_{#1}
\end{array}}}
\def\dynamicsmatrix{\brackets{\begin{array}{cc}
        I-\alpha J_{xx} &  -\alpha J_{xy} \\
        -\beta J_{yx} & I-\beta J_{yy}
    \end{array}}}

\def\jxxab{J_{xx}^{\alpha, \beta}}
\def\jyyab{J_{yy}^{\alpha, \beta}}

\def\scondL{\sqrt{\kappa(\LyapF)}}
\def\scondG{\sqrt{\kappa(\LyapG)}}

\def\epsStar#1{\varepsilon_{#1}^*}
\def\Rbound{\radiusU + \opnorm{U_0}}

\def\J#1{J_{#1}}
\def\eps{\varepsilon}

\def\Lsim#1{\LyapF^{-1/2} #1 \LyapF^{-1/2}}
\def\LsimSym#1{\LyapF^{-1/2} #1 \LyapF^{-1/2}}

\def\radiusX{R_{syl}}
\def\radiusU{R_{ric}}
\def\radiusUfixed{{1\over 4 \opnorm{G} \scondG M_J }}

\def\ballU{B_{\radiusU}^{U_0}}
\def\ballX{B_{\varepsilon \radiusX}}

\def\MSEx{\mathrm{MSE}^{x}_\infty(\alpha, \beta)}
\def\MSEy{\mathrm{MSE}^{y}_\infty(\alpha, \beta)}

\def\boundMSEFastPoisson{
\abs{\alpha  \Expect{ \sum_{j=0}^{k} (1-\alpha\muf)^{k-j} \inner{\x{j}}{\psi_j}}} \leq &\alpha \parentheses{  \boundXi \parentheses{{3 - \alpha \muf \over 1 - \alpha \muf} + {\newcontent{\boundF(1 + |\Xi| \boundXY \Lk )} \over \muf}} + {\Lvf\boundF (\boundXY +  \boundH)\over \muf} }   \\
& + \beta \parentheses{{\boundXi \boundH + \Lvf \boundG (\boundXY +  \boundH) + \newcontent{|\Xi| \boundXY \Lk \boundXi \boundG} \over \muf} }
}

\def\boundMSESlowPoisson{        \abs{\beta  \Expect{ \sum_{j=0}^{k} (1-\beta\mug)^{k-j} \inner{\y{j}}{\phi_j}}} \leq&  \beta \parentheses{ {(\Lvg + \newcontent{\boundXi  \Lk  |\Xi|}) \boundXY \boundG  \over \mug} +   \boundXi \parentheses{{3 - \beta \mug \over 1 - \beta \mug} + {\boundG \over \mug} }} \\
        & + \alpha {(\Lvg + \newcontent{\boundXi  \Lk  |\Xi|}) \boundXY \boundF\over \mug}}

\newtheorem{assumption}{A.}
\newtheorem{assumptionB}{B.}
\newtheorem{lemma}{Lemma}
\newtheorem{theorem}{Theorem}

\newtheorem{proof*}{Proof}

\newcommand\newcontent[1]{#1}

\title{The Bias of Nonlinear Two-Time-scale Stochastic Approximation under Constant Step-Sizes}

\author{%
  Djamel Rassem Lamouri \qquad Dorian Baudry \qquad Nicolas Gast\\~\\
  Univ. Grenoble Alpes, CNRS, Inria\\
  Grenoble INP, LIG, 38000 Grenoble\\
  France
}

\begin{document}

\maketitle

\begin{abstract}
Two-timescale stochastic approximation (TTSA) is a fundamental tool for analyzing coupled iterative algorithms in reinforcement learning, optimization, and stochastic control. However, finite-time guarantees for nonlinear two-timescale schemes remain difficult to obtain, especially under constant step-sizes. In this paper, we study nonlinear TTSA with step-sizes $\alpha\gg\beta$. Under standard stability, regularity, and Markovian noise assumptions, we upper bound the mean-squared error and the bias of both iterates around their limiting equilibria. Our bounds scale as $O(\alpha+\beta^2/\alpha^2)$, which we prove to be tight when $\beta\le\alpha^{3/2}$. The analysis separates the contributions of initial conditions, fast-timescale tracking error, Markovian dependence, and timescale coupling, thereby clarifying the origin of the $\beta^2/\alpha^2$ term. Our results reveal qualitative differences from the linear TTSA setting previously studied, showing that nonlinear dynamics introduce additional finite-time effects that are absent in the linear case. 

\end{abstract}

\section{Introduction}

A broad family of machine learning and optimization algorithms can be formulated within the framework of Stochastic Approximation (SA) \cite{Borkar2008StochasticAA, Kushner1997StochasticAA, Benveniste1990AdaptiveAA}. In the two-timescale setting \cite{Konda_2004, Mokkadem_2006}, two sequences of parameter vectors, $\{x_k\}_{k\geq 0}$ and $\{y_k\}_{k\geq 0}$, are generated from the following recursive update rule: 
\begin{equation}\label{eq: TTSA update}
\begin{aligned}
    x_{k+1} &= x_k - \alpha_k \F{k} \\
    y_{k+1} &= y_k - \beta_k \G{k},
\end{aligned}
\end{equation}
where, for a measurable space $(\Xi,\mathcal{S})$, the $\mathcal{S}$-measurable functions $f : \Rx \times \Ry \times \Xi \to \Rx$ and $g :  \Rx \times \Ry \times \Xi \to \Ry$ are the noisy dynamics of the system, $\{\xi_{k}\}_{k\geq 1}$ is the source of noise modeled as a stochastic process defined on $\Xi$ and $\{\alpha_k\}_{k\geq 0}$, $\{\beta_k\}_{k \geq 0}$ are positive sequences representing the step-sizes.
Some notable examples of algorithms that naturally exhibit this scheme of recursive updates include AdaGrad \cite{JMLR:v12:duchi11a}, Generative adversarial networks \cite{goodfellow2014generativeadversarialnetworks}, Bi-level Optimization \cite{ji2021bilevel} and Reinforcement Learning \cite{NIPS2008_e0c64119, NIPS1999_6449f44a}.

To study the dynamics of SA, a classical tool is to use the ODE method \cite{Borkar2000TheOM}. It consists in viewing \eqref{eq: TTSA update} as a noisy Euler discretization of the singularly perturbed ODE
\begin{equation}\label{eq: ODE TTSA}
    \begin{aligned}
        \dot x(t) &= - \frac{1}{\epsilon}\overline f (x(t),y(t)) \\
        \dot y(t) &= -  \overline g(x(t), y(t)) \;,
    \end{aligned}
\end{equation}
with $\epsilon = \beta_k/\alpha_k \approx 0$ \cite[Chapter~6]{Borkar2008StochasticAA}. The small parameter \(\epsilon\) induces a separation between the two timescales: The variable \(x\) evolves on the fast timescale, whereas \(y\) evolves slowly and is therefore seen by the \(x\)-dynamics as quasi-static. 
This motivates the standard decoupling of the two ODEs: for each fixed \(y\), the fast dynamics are described by $
    \dot x(t) = -\overline f(x(t),y)$.
Under suitable stability assumptions, the trajectory solutions of this fast ODE converge to the asymptotically stable equilibrium \(h(y)\) satisfying
    $\overline f(h(y),y)=0$.
Once this tracking property is established, the slow ODE is $\dot y(t) = -\overline g(h(y(t)),y(t))$. Under suitable assumptions, the solution of the ODE converges to a point $y^*$. The ODE method consists in assuming that $x_k\approx h(y_k)$ and $y_k\approx y^*$.

To guarantee the accuracy of the ODE method, it is therefore important to characterize the error that is made when approximating \eqref{eq: TTSA update} by the solutions of the ODE \eqref{eq: ODE TTSA}.  The classical convergence results for SA and TTSA typically rely on diminishing step-sizes. 
For two time-scale schemes, one usually assumes the step-sizes to have infinite  $\ell^1$-norm, finite $\ell^2$-norm and that $\lim_{k\to \infty} {\frac{\beta_k}{\alpha_k}}=0$ which ensures the time-scale separation of the associated limiting ODEs \cite{Borkar2008StochasticAA}. 
In many applications, however, constant step-sizes \(\alpha_k\equiv \alpha\) and \(\beta_k\equiv \beta\) are preferred, notably because they allow an exponential rate of forgetting initial conditions \cite{needell2015stochasticgradientdescentweighted}. 
In this regime, one should not expect exact convergence to the equilibrium; instead, the iterates typically approach a neighborhood whose size depends primarily on the step-sizes and the noise.

This motivates a finite-time analysis of the constant-step-size TTSA under Markovian noise. 
Previous works have quantified the bias induced by constant step-sizes for nonlinear single time-scale SA \cite{allmeier2024computingbiasconstantstepstochastic} and for linear TTSA \cite{huo2023biasextrapolationmarkovianlinear}. 
The nonlinear TTSA setting remains more delicate, because the fast tracking error and the slow dynamics interact through the nonlinear manifold \(x=h(y)\). 
Our goal is to characterize this interaction quantitatively. We will measure the performance of the algorithm through two quantities: the \emph{mean-squared error} (MSE), which is defined as 
\begin{equation*}
    \mathrm{MSE}^x_k \triangleq \Expect{\norm{x_k-h(y_k)}^2}\qquad \text{ and } \qquad\mathrm{MSE}^y_k \triangleq \Expect{\norm{y_k-y^*}^2},
\end{equation*}
and the \emph{bias} of the iterates around their limiting equilibrium, which we formally define as
\begin{equation*}
    \mathrm{Bias}^x_k \triangleq \Expect{x_k} - h(y^*) \qquad\text{ and }\qquad\mathrm{Bias}^y_k\triangleq\Expect{y_k}-y^*.
\end{equation*}

Our goal in this paper is to provide a finite-time analysis of nonlinear two time-scale stochastic approximation with constant step-sizes under a quite general Markovian noise. One original feature of our model is that we allow the kernel of the noise process $\xi_k$ to depend on the iterates $(x_k, y_k)$. This is particularly relevant for applications to reinforcement learning where the exploration policy might depend on the current values of the parameters. 


\textbf{Contributions.}
Our first main result, Theorem~\ref{thm : mse}, proves that the mean-squared errors of both time-scale components are bounded, up to exponentially decaying initialization terms, by $O\!\left(\alpha+\beta+{\beta^2}/{\alpha^2}\right)$, which reduces to \(O(\alpha+\beta^2/\alpha^2)\) when \(\beta\leq\alpha\).  The proof of Theorem~\ref{thm : mse} combines a contraction argument for both iterates $x_k$ and $y_k$ with a decomposition of the Markovian noise based on the problem-specific Poisson equation which allows us to handle the dependencies of the transition kernel on the model's parameters. 
The bound we prove separates the exponentially vanishing contribution of the initialization from the persistent error induced by constant step-sizes. 
The resulting recursion makes explicit the different sources of error: the intrinsic fluctuations of the fast iterate, the slow motion of the quasi-stationary target \(h(y_k)\), and the bias induced by the noise. 
This extends the approaches developed for nonlinear single time-scale SA \cite{allmeier2024computingbiasconstantstepstochastic} and nonlinear TTSA with diminishing step-sizes \cite{doan2021finitetimeconvergenceratesnonlinear} to the constant-step-size non-linear TTSA regime. 
When $\beta\le\alpha^{3/2}$, our MSE bounds are of order $O(\alpha)$ for both the fast and the slow iterate. In Theorem~\ref{thm : lower bound MSE}, we show that this bound cannot be improved by exhibiting an example whose MSE is of the order $\Omega(\alpha)$. This shows that the non-linear TTSA is significantly different from the linear TTSA for which the MSE of the slow iterate was shown to be $O(\beta+\alpha^2)$ in \cite{kwon2024twotimescalelinearstochasticapproximation}.


Our second main result, Theorem~\ref{thm : bias}, establishes matching-order bounds on the bias of the fast and slow iterates. 
The bias analysis is substantially more delicate than the MSE analysis because the analysis of the first-order dynamics of the two bias components remains coupled. 
We handle this difficulty by linearizing the nonlinear drift around the limiting equilibrium and then applying a sequence of coordinate transformations that separates the dominant fast and slow modes. 
In these transformed coordinates, the bias satisfies a recursion formula, which leads to upper bounds involving the Markovian-noise residuals and the MSE bounds from Theorem~\ref{thm : mse}. 
Transforming back to the original variables yields bias bounds of order $O\!\left(\alpha+{\beta^2}/{\alpha^2}\right)$. When $\beta\le\alpha^{3/2}$, this again reduces to \(O(\alpha)\) which we prove to be tight in Theorem~\ref{thm : lower bound}. 

\textbf{Outline.}
The remainder of the paper is organized as follows. In  
Section~\ref{sec::related_work}, we detail related work on single and two time-scale stochastic approximation, focusing on finite-time guarantees obtained with constant step-sizes under Markovian noise. 
Section~\ref{sec::assumptions} introduces the nonlinear TTSA model and states the assumptions used throughout the paper. 
Section~\ref{sec::MSE} presents Theorem~\ref{thm : mse}, a finite-time MSE bound, together with a proof overview and a discussion of tightness. 
Section~\ref{sec::bias} presents Theorem~\ref{thm : bias}, a finite-time bias bound, and outlines the main ideas behind its proof. 
Finally, Section~\ref{sec::exp} illustrates the theoretical bounds by creating an example of nonlinear TTSA showing how the term $\beta^2 /\alpha^2$ persists in the bounds of our MSE.



\section{Related Work}\label{sec::related_work}

The study of stochastic approximation (SA) algorithms traces back to the seminal work of \cite{Robbins1951ASA}, which introduced a framework for finding a root $x$ that satisfies
$M(x) = b$ for a given $b \in \Reals$. Here, when referring to SA, we mean \emph{one} time-scale stochastic approximation, where only the variable $x$ is present (as opposed to to TTSA where two variables are present and evolve at different time-scales $\beta\ll\alpha$).  In this classical setting, the monotonic function $M : \Reals \rightarrow \Reals$ is assumed to be unknown, requiring the algorithm to rely strictly on noisy observations. A pivotal development in the theoretical understanding of SA is through the ordinary differential equation (ODE) method \cite{Borkar2000TheOM}. This technique characterizes the asymptotic convergence properties of discrete SA updates by analyzing their continuous-time limit, represented by the analogous ODE $\dot{x} = -\bar{f}(x)$. The ODE method is the method of choice for analyzing SA dynamics, with its underlying theory and applications extensively detailed across numerous foundational texts \cite{Borkar2008StochasticAA, Kushner1997StochasticAA, Benveniste1990AdaptiveAA}.

\textbf{Constant step-size SA.} A line of work has been developed around the constant step-size assumption in SA due to practical advantages in algorithm design. Starting with Linear SA (LSA), \cite{lakshminarayanan2017linearstochasticapproximationconstant} have shown that the MSE of the Polyak-Ruppert (PR) averages \cite{polyak_judistky_92} decreases with order $O(\frac{1}{k})$ under i.i.d noise. Previously, in \cite{bach2013nonstronglyconvexsmoothstochasticapproximation} averaging was used on non-strongly convex SA to achieve the same rate. Under the same settings \cite{durmus2021tighthighprobabilitybounds} 
proved that with probability at least $1-\delta$, $|x_k -x^*|$ is bounded by $O(\sqrt{\alpha \log(1/\delta)})$, where $x^*$ denotes the equilibrium of $x$. Later, the work of  \cite{durmus2023finitetimehighprobabilityboundspolyakruppert} generalized the high probability bounds of the latter result to PR averages under Markovian noise. In \cite{huo2023biasextrapolationmarkovianlinear}, it is shown that LSA under constant step size and Markovian noise is characterized by a bias of order $O(\alpha)$ that can be improved to higher orders by applying Richardson-Romberg extrapolation. Switching to non-linear SA, \cite{chen2022finitesampleanalysisnonlinearstochastic} showed an MSE of order $O(\alpha \log(1 / \alpha))$. This bound is improved in \cite{allmeier2024computingbiasconstantstepstochastic} to $O(\alpha)$, with a characterization of the bias vector by $\alpha V + O(\alpha ^2)$, where $V$ is a solution vector of a problem-specific Lyapunov equation. 

\textbf{Linear TTSA.} In the linear regime, the foundational work \cite{Konda_2004} established the asymptotic variance and normality of the iterates under i.i.d. noise and decreasing step-sizes. Subsequent research shifted toward finite-time analyses. In \cite{kaledin2020finite}, the MSE of the fast iterate was shown to be of order $O(\alpha_k)$ and that of the slow iterate $O(\beta_k)$ for step-sizes satisfying $\alpha_k =O (k^{-v})$ for some $v<1$ and $\beta_k = O(k^{-1})$. A convergence rate of $O(1 / k^{2/3})$ under Markovian noise was derived in \cite{doan2020finitetimeanalysisrestartingscheme}, which was later tightened to $O(1/k)$ in \cite{haque2025tightfinitetimebounds}. The constant step-sizes setting was studied in \cite{kwon2024twotimescalelinearstochasticapproximation},  which established an upper bound of order $O(\alpha)$ and $O(\beta + \alpha^2)$ for the fast and slow iterates, respectively. They further proved upper bounds of order $O(\alpha + \beta)$ for the bias of both iterates. Our results indicate that these bounds are specific to this setting, since they degrade after removing the linearity assumption.

\textbf{Nonlinear TTSA.} Parallel advancements have been made in non-linear TTSA. Extending asymptotic normality to this setting, \cite{Mokkadem_2006} proved a Central Limit Theorem (CLT) for the iterates, a property recently expanded upon by \cite{hu2024centrallimittheoremtwotimescale} under decreasing step-sizes. Regarding finite-time guarantees, \cite{doan2021finitetimeconvergenceratesnonlinear, doan2020finitetimeanalysisrestartingscheme} achieved an MSE rate of $O(1/k^{2/3})$ under both i.i.d. and Markovian noise. More recently, \cite{chandak2025finitetimeboundstwotimescalestochastic} analyzed contractive non-linear TTSA under Markovian noise, demonstrating that the MSE for both iterates is bounded by $O(\alpha_k + \beta_k^2 / \alpha_k^2)$. All of those works assume decreasing step-sizes.

The remaining gap, which we address in this work, is to obtain finite-time bounds on both the MSE and the bias of nonlinear TTSA with constant step-sizes, under Markovian noise.



\section{Assumptions}\label{sec::assumptions}
First, to guarantee algorithmic stability, we assume that the algorithm's iterates remain within a compact set. This is often enforced in practice via explicit projections \cite[Chapter 5]{Kushner1997StochasticAA}, \cite{NIPS2009_3a15c7d0}. 
\begin{assumption}\label{assumption: Boundedness}
    There exist two compact sets $\domX \subset  \Rx$ and $\domY \subset \Ry$ such that for all $k\geq 0$, $x_k\in\domX$ and $y_k\in\domY$ a.s. 
\end{assumption}
Our second assumption concerns the stochastic process $(\xi_k)_{k\geq 1}$, which we assume to be a \emph{positive recurrent Markov chain}. This allows us to account for temporally correlated data, in contrast with the classical i.i.d.\ or martingale-difference noise models, and is particularly relevant for applications such as reinforcement learning. The price of this added flexibility is that the analysis must handle the bias and dependence induced by the Markovian dynamics.  
\begin{assumption}\label{assumption: All Kernel Properties}
    $(\xi_k)_{k\geq 1}$ evolves as a Markov chain on a finite state space $\Xi$ with a transition kernel \newcontent{$\Kernel{k} $} that is allowed to depend on $x_k,y_k$, that is: for every $\xi',\xi\in \Xi$ and $x\in \domX, y\in \domY$ we have: 
    \newcontent{
    \begin{align*}
    \mathbb{P}(\xi_{k+1} = \xi' | \xi_k = \xi, \mathcal{F}_k) &=\mathbb{P}(\xi_{k+1} = \xi' | \xi_k = \xi, x_k=x,y_k=y)\triangleq \Kernel{}(\xi,\xi'),
    \end{align*}
    }
    where $\{\mathcal F_k\}_{k\geq 1}$ is the natural filtration of the process $(x_k,y_k,\xi_k)_{k\geq1}$.  We further assume that the kernel \newcontent{$\Kernel{}$} corresponds to a unichain Markov chain for all $(x,y)\in\Rx\times\Ry$. Also, we assume that for all $\xi, \xi' \in \Xi$ the quantity \newcontent{$K_{(.,.)}(\xi,\xi')$} is $\Lk$-Lipschitz continuous:
    \begin{align*}
        \newcontent{\abs{\Kernel{1}(\xi,\xi') - \Kernel{2}(\xi,\xi')}} &\leq \Lk (\|x_1 - x_2\|+\|y_1 - y_2\|).
    \end{align*}
\end{assumption}
The assumption of having a Markovian noise is standard in stochastic approximation and covers a broad class of algorithms
\cite{kwon2024twotimescalelinearstochasticapproximation,hu2024centrallimittheoremtwotimescale,doan2021finitetimeconvergenceratesnonlinear,chen2022finitesampleanalysisnonlinearstochastic}. The fact that we allow the kernel at time $k$ to depend on $(x_k,y_k)$ is more original and is relevant in reinforcement learning algorithms where the exploration policy might depend on the current parameters \cite{allmeier2024computingbiasconstantstepstochastic}. Assumption A.\ref{assumption: All Kernel Properties} guarantees the existence of a unique stationary distribution \newcontent{$\Stationary{}$} for each $(x,y)$, thereby allowing us to define the averaged functions $\bar f$ and $\bar g$ by integrating $f$ and $g$ with respect to $\pi$:
\begin{align}
    \label{eq:def barF barG}
    \barF{} &\triangleq \underset{\xi\sim\Stationary{} \hspace{0.7cm}}{\Exp[f(x,y,\xi) ]} = \sum_{\xi\in\Xi} \Stationary{}(\xi) f(x,y,\xi)\quad\text{and}\quad
    \barG{} \triangleq \underset{\xi\sim\Stationary{} \hspace{0.7cm}}{\Exp[g(x,y,\xi) ]}.
\end{align}
Furthermore, the assumption on Lipschitzness of $K$ implies that the stationary distribution \newcontent{$\Stationary{}$} is also Lipschitz-continuous in $(x,y)$, see \emph{e.g.} \cite{allmeier2024computingbiasconstantstepstochastic}.
Third, to ensure the existence and uniqueness of the limiting trajectories of the equivalent ODE, we need to assume
Lipschitzness of $f,g$, which, combined with Lipschitzness of $K$ and the definitions of $\bar{f}$ and $\bar{g}$ of \eqref{eq:def barF barG} implies that the functions $\bar{f}$ and $\bar{g}$ are also Lipschitz continuous which is a fundamental prerequisite of the ODE method, ensuring the existence and uniqueness of solutions of the ODE. We start with our assumptions on $f$:
\begin{assumption}\label{assumption: Lip and Mono of f}
    For all $\xi \in \Xi$, the function $f(., ., \xi)$ is $L_f$-Lipschitz continuous:
    \begin{align*}
        \norm{f(x_1, y_1, \xi) - f(x_2, y_2,\xi) } &\leq \Lf (\|x_1 - x_2\|+\|y_1 - y_2\|)
    \end{align*}
 In addition, there exists $\muf > 0$ such that $\overline f$ is $\muf$-strongly monotonic in its first argument, \emph{i.e.}, for every $x_1,x_2 \in \domX$ and $y \in \domY$ we have
 \begin{align*}
    \langle \overline f(x_1,y)-\overline f(x_2,y),  x_1-x_2 \rangle  \geq \muf \|x_1-x_2\|^2.
    \end{align*}    
\end{assumption}
This strong monotonicity assumption guarantees that for a fixed $y$, the ODE $\dot{x}(t)= - \bar{f}(x(t), y)$ possesses a unique, exponentially stable equilibrium, which we denote by $h(y)$. The equilibrium map $h$ is $\frac{\Lf}{\muf}$-Lipschitz from Lipschitzness and strong monotonicity of $f$: 
\begin{align*}
    \muf \norm{h(y_1) - h(y_2)} ^2 &\leq \inner{\bar f (h(y_1), y_1)- \bar f (h(y_2), y_1)}{h(y_1)- h(y_2)}\\
    & = \inner{\bar f (h(y_2), y_2)- \bar f (h(y_2), y_1)}{h(y_1)- h(y_2)} \\
    &\leq \norm{\bar f (h(y_2), y_2)- \bar f (h(y_2), y_1)}\norm{h(y_1)- h(y_2)}\\
    & \leq \Lf \norm{y_2 - y_1}\norm{h(y_1)- h(y_2)}
\end{align*}
The result follows when $h(y_1)\neq h(y_2)$, we denote $\Lh \triangleq \frac{\Lf}{\muf}$. Moving to the slow iterate, we need Lipschitzness of $g$: 
\begin{assumption}\label{assumption: Lip and Mono of g}
    For all $\xi \in \Xi$, the function $g(., ., \xi)$ is $\Lg$-Lipschitz continuous:
    \begin{align*}
        \|g(x_1, y_1, \xi) - g(x_2, y_2,\xi) \| &\leq \Lg (\|x_1 - x_2\|+\|y_1 - y_2\|)
    \end{align*}
In addition, for $\mug > 0$, $\overline g$ is 1-point $\mug$-strongly monotonic with respect to a point $y^* \in \domY$ \newcontent{that satisfies $\bar{g}(h(y^*),y^*)=0$}: for all $y \in \domY$:
    \begin{align*}
        \langle \GH{} ,y-y^*\rangle \geq \mug \|y-y^*\|^2.
    \end{align*}    
\end{assumption}
This assumption implies that the function $\bar{g}(h(y),y)$ is also Lipschitz continuous and in particular implies that the ODE $\dot{y}=-\bar{g}(h(y),y)$ has a unique solution. To guarantee that this ODE has a unique stationary point $y^*$, we added the 1-point strong monotonicity at $y^*$, for instance, as in \cite{doan2021nonlineartwotimescalestochasticapproximation, chen2025convergence}. It can be viewed as a generalization of the non-linear case of having Hurwitz matrices in the linear setting \cite{kwon2024twotimescalelinearstochasticapproximation, huo2023biasextrapolationmarkovianlinear, durmus2021tighthighprobabilitybounds}, which are the conditions required for the ODE trajectories to exhibit global convergence.
\paragraph{Notation} Let $x\in \mathbb R ^d$, $\|x\|$ is the Euclidean norm $\|x\| = (\sum_{i=1}^d x_i^2)^{1/2}$. For $y \in \mathbb R ^d$ we denote the inner product between $x$ and $y$ by $\langle x, y \rangle $. For a square matrix $M \in \mathbb R^{n\times n }$ we define the operator norm $\opnorm{M} = \sup_{\|x\| =1}\|Mx\|$. For a positive definite matrix $L$ we define the condition number of $L$ as $\kappa(L)=\tfrac{\lambda_{\max}(L)}{\lambda_{\min}(L)}$ where $\lambda_{\max}(L)$ and $\lambda_{\min}(L)$ are the largest and smallest eigenvalues of $L$. For a multivariate function $\bar f : \Rx \times \Ry \rightarrow \Rx$, $\nablafx(x,y)$ denotes the partial Jacobian with respect to the first argument $(\nablafx(x,y))_{i,j} = {\partial \overline f_i(x,y)\over \partial x_j}$ and $\nabla_y \overline f(x,y)$ denotes the partial Jacobian with respect to the second argument.
\section{The MSE Analysis of TTSA}\label{sec::MSE}
Our first main result is a bound on the Mean Squared Error (MSE) for both the fast and the slow iterates.
\subsection{Main Result}
The following theorem provides non-asymptotic bounds on $\mathrm{MSE}^x_k$ and $\mathrm{MSE}^y_k$, explicitly separating the transient initial condition from the steady-state residual error caused by the Markovian noise and the constant step-sizes.

\begin{restatable}{theorem}{mseThm}\label{thm : mse}
Assuming that A.\ref{assumption: Boundedness}-\ref{assumption: Lip and Mono of f} hold, and that $\alpha \muf<1$ and $\beta\leq\alpha$, then for any $k\geq 1$ we obtain the following upper bound on the MSE of the fast iterate, 
  \begin{align}
        \Exp[\|x_k-h(y_k)\|^2] \leq  & \underbrace{(1-\alpha\muf)^{k}\norm{x_{0}-h(y_0)}^2}_{\text{Decaying initial condition}} + \underbrace{\alpha \;\left(\frac{\mathrm{B}_1}{1-\alpha \mu_f} + \mathrm{B}_2\right)  + \beta \; \mathrm{B}_3 + {\beta^2 \over \alpha^2 }\mathrm{B}_4}_{\MSEx}   \;,
        \label{eq:MSE_x}
    \end{align}
    where the constants $(B_k)_{k\in [4]}$ are fully defined in appendix (Eq.~\eqref{eq::B_1}-\eqref{eq::B_4}), and are independent of $\alpha, \beta$ and $k$: they only depend on the properties of the Markovian noise, the range, monotonicity and Lipschitzness parameters of the functions $f, g, h$. 
Further, assuming A.\ref{assumption: Lip and Mono of g}, and $\beta \mu_g <1$, for any $k \geq 1$ the MSE of the slow iterate is upper bounded by 
\begin{align*}
        \Exp[\|y_k - y^*\|^2] \leq  
        &\underbrace{(1- \beta \mug )^{k}\norm{y_{0}-y^*}^2  + D_1 \beta \; k(1-(\alpha \muf \wedge \beta \mug))^{k} \norm{x_0-h(y_0)}^2}_{\text{Decaying initial conditions}} \\
        & + \underbrace{D_2 \; \MSEx + \beta \; \left(\frac{D_3}{1-\beta \mug}+ D_4\right) + \alpha  D_5 }_{\MSEy}  \;. 
\end{align*}
where the constants $(D_k)_{k\in [5]}$ are also formally defined in appendix (Eq.~\eqref{eq::D_1}-\eqref{eq::D_5}) and depend again on problem parameters but not on $\alpha, \beta, k$.
\end{restatable}
\begin{proof}[Proof Overview.] 
    To prove the MSE bounds for the fast iterate $x_k$, we first use the strong monotonicity assumption A.\ref{assumption: Lip and Mono of f} to extract the following recursion on the tracking errors between steps $k$ and $k+1$,
    \begin{align}
        \label{eq:fast recurrence}
        \norm{x_{k+1}-h(y_{k+1})}^2 \leq (1-\alpha \mu_f) \norm{x_k-h(y_{k})}^2 + r^x_k,
    \end{align}
    where $r^x_k$ is a residual term containing the deviation of the noisy dynamics from their stationary means in addition to higher order terms of $\alpha$ and $\beta$ that we detail in the next paragraph. The goal of this recursion is to dissociate the impact of the initial conditions from the cumulative impact of the noise.  Repeatedly applying the above inequality, we get
    \begin{align*}
        \norm{x_{k+1}-h(y_{k+1})}^2 \leq (1-\alpha \mu_f)^{k+1} \norm{x_{0}-h(y_{0})}^2 + \sum_{j=0}^k (1-\mu_f\alpha)^{k-j}r^x_j.
    \end{align*}
    This first step is similar to the approach used for the decreasing step sizes in \cite{doan2021finitetimeconvergenceratesnonlinear}, except in the latter case the sum of the residuals under expectation vanishes as $k\to \infty$ which is not the case here due to the assumption of fixed step-sizes.
    
    To bound this sum of residuals, we then use A.\ref{assumption: Lip and Mono of f} to show that each $r^x_j$ is bounded by $2\alpha\inner{x_{j}-h(y_j)}{f(x_j,y_j,\xi_j)-\barF{j}}$, plus a term $M_{\alpha,\beta}$ that is of order $O(\alpha^2+\beta^2+\alpha\beta+\beta^2/\alpha)$. The sum $\sum_{j=0}^k (1-\mu_f\alpha)^{k-j}M_{\alpha,\beta}$ leads to a $O(\alpha+\beta+\beta^2/\alpha^2)$ error term.  The main technical difficulty is therefore to control the sum induced by the first term.  To do so, we follow the Poisson-equation technique used in
    \cite[Lemma 8]{allmeier2024computingbiasconstantstepstochastic} to rewrite the noise as a Martingale difference term, a bounded term that scales linearly with the step-sizes, and a weighted telescoping term. Because of the geometric weights \((1-\alpha\muf)^{k-j}\), the telescoping term is not exact and leaves additional terms. Additionally, since the kernel is depending on the current iterates, the Martingale term creates a drift in the kernel that we control due to Lipschitzness of $K$ from A.\ref{assumption: All Kernel Properties}. Lemma~\ref{lemma : poisson mse} shows that all these contributions are of the required order.

    
    To prove the MSE bound for the slow iterate, we apply a recipe similar to the one followed for the fast iterate. From assumption A.\ref{assumption: Lip and Mono of g} we obtain an equation for $y$ similar to \eqref{eq:fast recurrence} with a contraction factor $1-\beta \mug$, along with an additional term capturing the influence of the fast iterate:
    \begin{align*}
        \norm{y_{k+1}-y^*}^2 \leq (1-\beta \mug) \norm{y_{k}-y^*}^2 +\beta \norm{x_{k}-h(y_k)}^2 +r^y_k.
    \end{align*}
    We then use the results of the fast iterate to bound $\norm{x_{k}-h(y_k)}^2$. The treatment of $r^y_k$ is then similar to the analysis done for the fast iterate. The full proof can be found in Appendix~\ref{section : mse detailed proofs}.
\end{proof}

\subsection{Discussion and tightness}

Under the constant step-size regime and the strong monotonicity assumptions 
A.\ref{assumption: Lip and Mono of f} and A.\ref{assumption: Lip and Mono of g}, the initial conditions are forgotten at an exponential rate, respectively 
\((1-\alpha\muf)^k\) for the fast iterate and \((1-\beta\mug)^k\) for the slow iterate. 
This illustrates one of the main practical advantages of constant step-sizes: the iterates rapidly enter a neighborhood of the equilibrium, although the size of this neighborhood is controlled by the persistent stochastic error. 
The two time-scale structure also creates a transient coupling between the initial fast error and the slow iterate, which appears through the term
\[
    k\bigl(1-(\alpha\muf \wedge \beta\mug)\bigr)^k\norm{\x{0}}^2,
\]
which may initially increase because of the prefactor \(k\), but still vanishes exponentially as \(k\to\infty\).

Once the transient terms have vanished, the iterates remain, in expectation, within a neighborhood of the equilibrium whose squared radius scales as
\(  O\!\left(\alpha+\beta^2/\alpha^2\right).
\)
The presence of the term \(\beta^2/\alpha^2\) in both bounds is one of the main features of Theorem~\ref{thm : mse}. 
It is consistent with finite-time results for nonlinear TTSA with decreasing step-sizes, where bounds of order 
\(O(\alpha_k+\beta_k^2/\alpha_k^2)\) are obtained for both iterates \cite{chandak2025finitetimeboundstwotimescalestochastic}. 
This term quantifies the cost of coupling the fast and slow recursions: without sufficient time-scale separation, it may dominate the steady-state error.

However, when the slow step-size is sufficiently small relative to the fast one, the \(\alpha\)-term becomes dominant. 
In particular, if \(\beta\leq \alpha^{3/2}\), then \(\beta^2/\alpha^2\leq \alpha\), and Theorem~\ref{thm : mse} immediately yields that, for any TTSA satisfying its assumptions, there exist constants \(C>0\) and \(\alpha_0>0\) such that, for all \(\alpha<\alpha_0\) and all \(\beta\leq \alpha^{3/2}\),
\begin{align}
    \label{eq:corollary MSE}
    \limsup_{k\to \infty} \Exp[\|x_k-h(y_k)\|^2] 
    \leq C\alpha,
    \qquad\text{and}\qquad
    \limsup_{k\to \infty} \Exp[\|y_{k}-y^*\|^2] 
    \leq C\alpha.
\end{align}



We emphasize that the steady-state MSE bound for the slow iterate is substantially larger than what is obtained in the linear TTSA setting, where the corresponding rate is \(O(\beta+\alpha^2)\) \cite{kwon2024twotimescalelinearstochasticapproximation}. 
This raises a natural question: is the \(O(\alpha)\) term in the bound for the non-linear case unavoidable, or is it merely a proof artifact? 
The next result shows that it is not. It constructs a nonlinear TTSA instance satisfying the assumptions of Theorem~\ref{thm : mse} for which both the fast tracking error and the slow error have asymptotic MSE at least of order \(\alpha\). 
Thus, in the nonlinear setting, the rate in \eqref{eq:corollary MSE} cannot be improved in general: there is a qualitative separation from the linear TTSA setting, where sharper rates are possible for the slow iterate.


\begin{theorem}\label{thm : lower bound MSE}
    There exist constants \(C>0\), \(\alpha_0>0\), and a TTSA instance satisfying all the assumptions of Theorem~\ref{thm : mse} such that, for all step-sizes \(\beta\leq\alpha<\alpha_0\),
    \begin{align*}
        \liminf_{k\to\infty} \Expect{\|x_k-h(y_k)\|^2} 
        \ge C\alpha,
        \qquad \text{and} \qquad 
        \liminf_{k\to\infty} \Expect{\|y_{k}-y^*\|^2} 
        \ge C\alpha .
    \end{align*}
\end{theorem}
The proof is given in Appendix~\ref{apx:lower bound} and relies on a simple one-dimensional construction for both variables \(x_k\) and \(y_k\). 




\section{The Bias Analysis of TTSA}\label{sec::bias}
Analyzing the bias requires a finer characterization of the dynamics. We therefore assume additional differentiability, which allows us to linearize the averaged update around the equilibrium. The first-order terms drive the bias recursion, while the second-order Taylor remainders are controlled by the MSE bounds from Theorem~\ref{thm : mse}.
We use the following notation: $J_{xx} \triangleq  \nablafx(h(y^*), y^*), J_{xy} \triangleq  \nabla_y \overline f(h(y^*), y^*), J_{yx} \triangleq  \nabla_x \overline g(h(y^*), y^*)$  and $J_{yy} \triangleq  \nablagy (h(y^*), y^*)$. 
\begin{assumptionB}\label{assump : derivability}
    $\overline f$ and $\overline g$ are at least two times differentiable and their second derivatives are bounded operators. In addition, the matrices $-J_{xx}$ and $-\Delta  \triangleq -J_{yy} + J_{yx} J_{xx}^{-1} J_{xy}$ are Hurwitz.
\end{assumptionB}
The differentiability part of B.\ref{assump : derivability} is standard in analyses that rely on a local expansion of the stochastic-approximation dynamics. It is used, for instance, to characterize the constant-step-size bias in nonlinear SA \cite{allmeier2024computingbiasconstantstepstochastic}, and it is also a natural assumption in CLT analyses for nonlinear TTSA \cite{hu2024centrallimittheoremtwotimescale}. In our proof, this assumption allows us to linearize the averaged dynamics around the equilibrium: the first-order terms determine the leading bias recursion, while the second-order remainders are controlled through the MSE bounds.

The Hurwitz condition on $-J_{xx}$ and $-\Delta$ is essential when studying TTSA settings \cite{kwon2024twotimescalelinearstochasticapproximation, haque2025tightfinitetimebounds, hu2024centrallimittheoremtwotimescale, kaledin2020finite, Konda_2004}.  In particular, for the linear settings, it is equivalent (up to a linear transformation) to monotonicity in A.\ref{assumption: Lip and Mono of f} and A.\ref{assumption: Lip and Mono of g}.  It is needed to ensure that the trajectories under the approximate linear model converge exponentially fast to the equilibrium.


\subsection{Main result}


We now introduce the second main result of this paper, which provides an upper bound on the bias of TTSA. The result shows that, once the initial conditions vanish, the bias is of the same order as the bound on the MSE obtained in Theorem~\ref{thm : mse}.

\begin{restatable}{theorem}{biasThm}\label{thm : bias}
    Assuming A.\ref{assumption: Boundedness}-\ref{assumption: Lip and Mono of g} and B.\ref{assump : derivability}, \newcontent{there exists $\alpha^*$ defined in \eqref{eq: alpha star condition for bias thm}, $\beta^*$ defined in \eqref{eq: beta star condition for bias thm}, and $ \omega^*$ defined in \eqref{eq: omega star condition for bias thm} such that for $\alpha < \alpha ^*$, $\beta < \beta^*$ and $\frac{\beta}{\alpha} < \omega^* $} we have the following bound on the bias of both iterates:
    \begin{align*}
        \max\{ \norm{\Expect{ x_k - h(y^*)}}, \norm{\Expect{y_k -y^*}} \} \leq& \opnorm{T^{-1} S^{-1}} \parentheses{\norm{s_k^x } +\norm{s_k^y}}
    \end{align*}
    for some upper triangular matrix $T$ defined in \eqref{eq:riccati-transformation} and a lower triangular matrix $S$ defined in \eqref{eq:sylvester-transformation}, and where vectors $s^x_k, s^y_k$ are some transformed bias vectors defined in \eqref{eq:sylvester-fast-bias}, \eqref{eq:sylvester-slow-bias} respectively. Furthermore,
    \begin{align}
            \norm{s_{k}^x} +  \norm{s^y_{k}}  \leq & \underbrace{e^{-k c\alpha}\scondL\norm{s^x_0} + e^{-k d\beta } \scondG \norm {s^y_0} + a e^{-\alpha \hat c (k-1) }\norm{\x{0}}^2 + b  e^{-\beta \hat d (k-1) } \norm{\y{0}}^2 }_{\text{Decaying initial conditions}} \nonumber \\
            & + \underbrace{ F_1(\alpha +\beta) + F_2 \parentheses{ \mathrm{MSE}_{x}^\infty(\alpha, \beta)+ \mathrm{MSE}_{y}^\infty(\alpha, \beta)}   + R_{\alpha, \beta} }_{\text{Asymptotic bias}} \;,     
        \label{eq:bias bound}
    \end{align}
    where $F_1,F_2$ are positive constants independent of 
    $\alpha,\beta$ and $k$, whose values may depend on the problem parameters and on 
    $\omega^*$, $R_{\alpha, \beta}$ contains terms of higher order $(\alpha + \beta )^2$. 
    Additionally, $\LyapF, \LyapG$ are symmetric positive definite matrices solutions of the Lyapunov equations \eqref{eq:lyapunov-equation-fast}, \eqref{eq:Lyapunov-equation-slow}, $c = (8 \opnorm{\LyapF})^{-1}$ and $d = (8 \opnorm{\LyapG})^{-1}$, constants $a,b,\hat{c},\hat{d}$ are related to decaying initial errors from the MSE and can be found in \eqref{eq:sequence-sum-mse}, and the MSE terms are defined in Theorem~\ref{thm : mse}.
\end{restatable}
We detail the proof of the theorem in Appendix~\ref{section: bias detailed proofs}, and present its main arguments in the following proof overview.

\emph{Proof Overview.}
Recall that $\mathrm{Bias_k^x} \triangleq \Expect{x_k - h(y^*)}$ and $\mathrm{bias^y_k} \triangleq \Expect{y_k -y^*}$ are the bias vectors. As used in \cite{Mokkadem_2006, doan2021finitetimeconvergenceratesnonlinear}, we define the concatenation of the two bias vectors as $b_k \triangleq(\mathrm{Bias_k^x} , \mathrm{bias_k^y})^T$. 
This manipulation allows us to analyze the coupled dynamics and to control cross impacts of both iterates on each other. Using assumption B.\ref{assump : derivability} we use a Taylor expansion of the dynamics $\bar{f}$ and $\bar{g}$ around $(h(y^*),y^*)$ to obtain 
\begin{align*}
    b_{k+1} = H b_{k} + R_k,
\end{align*}
where $H$ represents the linear dynamics around the equilibrium and the remainder term $R_k$ contains the higher order terms of the Taylor expansion plus the expected impact of the Markovian noise.

As we will see later, the remainder term $R_k$ can be controlled somehow similarly to what we did for the MSE. The intricate part is to find a way to analyze the dynamics induced by $H$ where the off diagonal term represents the cross influence between the biases of both iterates.  To treat this term, we use a linear transformations of coordinates to decouple the fast and slow part, following techniques from control theory, see \emph{e.g.} \cite[Chapter 2]{kokotovic1999singular}. First, we construct an upper triangular matrix $T$ with identity on the diagonal and an unknown upper right block denoted by $U$, such that $\Pi \triangleq T H T^{-1}$ is lower triangular with a zero in the upper right block. This leads to solving a Riccati equation for $U$ \cite[Eq. 2.13]{kokotovic1999singular}. 
Second, we search for a transform of the matrix $\Pi$ that would make it block-diagonal, by removing the lower left block. To do that, we build a lower triangular matrix $S$ that has an unknown block $X$, and consider $\Lambda = S \Pi S^{-1}$ with the goal to make it diagonal matrix. We do that by solving a Sylvester equation in $X$ \cite[Lemma 8.3]{POZNYAK2008133}. Using both transformations, we get
\begin{align*}
    s_{k+1} = \Lambda s_{k} + ST R_k,
\end{align*}
where $s_k \triangleq ST b_k $. By having the matrix $\Lambda$ being block diagonal, with two blocks $\Lambda _1, \Lambda_{2}$, we can unroll the recursion
to give us access to individual updates of each component:
\begin{align*}
    s_{k+1}^x = \Lambda_1^k s_0^x + \sum_{j=0}^k \Lambda_1^{k-j} R^x_j ,&& s_{k+1}^y = \Lambda_2^k s_0^y + \sum_{j=0}^k \Lambda_2^{k-j} R^y_j.
\end{align*}
The last challenging part is to find the right norm so to ensure that the leading matrices behave as contractions for small enough step sizes.  This is done by solving a set of Lyapunov equations to get symmetric positive definite matrices that contain information about the linear dynamics of the system and such that the leading matrices are contractions under the Lyapunov induced norms. We can then propagate back the bounds to the original bias vectors by applying the inverse operation $b_k = (ST)^{-1} s_k$. \qed 

\subsection{Discussion}

The asymptotic bias provided in Theorem~\ref{thm : bias} is of the order of $O(\alpha+\beta^2/\alpha^2)$. This is not a direct consequence of the MSE bound of Theorem~\ref{thm : mse} which, by Jensen's inequality, would only provide a $O(\sqrt{\alpha+\beta^2/\alpha^2})$ on the bias. To prove Theorem~\ref{thm : bias}, we follow a different path from the one we used to prove the MSE: the approach we used was to treat the fast and slow iterates separately. We first obtained a bound on the MSE of $x$ and then plugged this term in the analysis of $y$. For the bias, we believe that this approach does not work and that the two iterates must be treated at the same time. Hence, what we do in the proof of Theorem~\ref{thm : bias} is to couple the analysis of the fast and slow iterates by using a change of variables linked to linearization of $\bar{f}$ and $\bar{g}$ around $(h(y^*),y^*)$. 

The first term in the bias bound \eqref{eq:bias bound} is an exponentially fast vanishing term. This term is similar to the one that we obtained for the MSE but is there because of a different reason. In the MSE analysis, this term came from the strong monotonicity assumption. Here, this term is linked to the Hurwitz stability of first order dynamics of the TTSA (Assumption B.\ref{assump : derivability}). The proof overview above explains how a transformation of our bias vectors into new coordinates allowed us to get a contraction. Obtaining this contraction allows us to derive the two remaining terms that form the asymptotic bound on the bias. These terms emerge from two different sources: the first source is the Markovian noise and induces the term $F_1(\alpha + \beta)$. The other terms $\MSEx +\MSEy$ are inherited errors from the MSE due to the high order terms from the Taylor approximation.

Taken together, the bias of both iterates has a leading dependence of order $O(\alpha+\beta)$, which matches the results of linear TTSA under Markovian noise  \cite{kwon2024twotimescalelinearstochasticapproximation}. In addition to that, we get an extra term $\beta^2 \over \alpha ^2$ due to non linearity of the dynamics. This term is injected from the MSE results and can not be improved unless we improve the bound on the MSE. 

When assuming a clear separation of timescales such that the slow step-size $\beta$ satisfies $\beta\leq \alpha^{3/2}$ the bound on the bias reduces to $O(\alpha)$ for both iterates.  Similarly to the MSE bound, we can show that this bound is tight: We prove in Theorem~\ref{thm : lower bound} that there exists a TTSA that has a lower bound on the bias of order $\Omega(\alpha)$.  The exact statement of the theorem along with its proof can be found in Appendix~\ref{apx:lower bound}.  Note that, contrary to the MSE case where linear and non-linear TTSA behave very differently, the bias in non-linear TTSA is essentially of the same order as the bias in the linear setting of \cite{kwon2024twotimescalelinearstochasticapproximation} where the order $O(\alpha+\beta)$ is proved to be tight.

\section{Illustrations}\label{sec:illustrations}
\begin{figure}[t]
    \centering  
    \begin{tabular}{@{}c@{}c@{}c@{}}
        \includegraphics[width=0.32\linewidth]{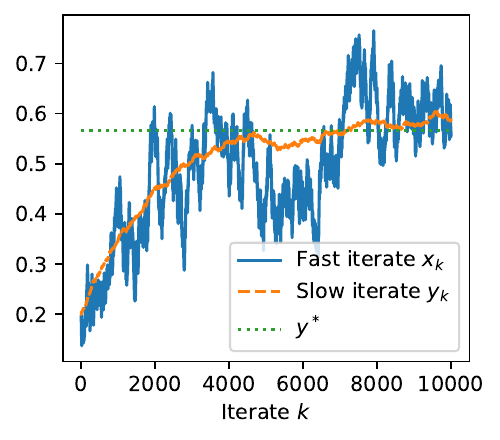}
        &\includegraphics[width=0.32\linewidth]{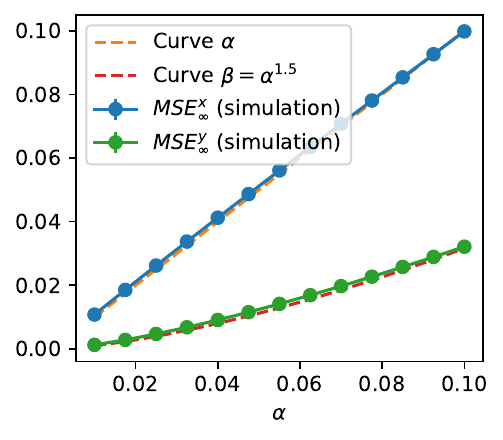}
        &\includegraphics[width=0.32\linewidth]{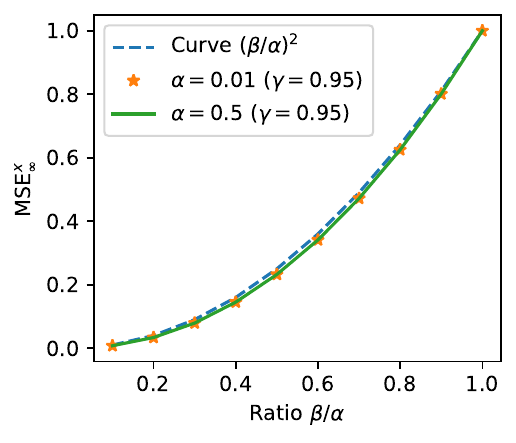}\\
        (a) Transient regime of \eqref{eq:example_smooth_apx} & (b) $\mathrm{MSE}^{\cdot}_\infty$ of \eqref{eq:example_smooth_apx} & (c) $\mathrm{MSE}^x_\infty$ of Example~\eqref{eq: beta^2/alpha^2}. 
    \end{tabular}
    \caption{Illustration of the convergence results.}
    \label{fig:various illustrations}
\end{figure}

In this section, we illustrate some implications of the main theorems of this paper. More numerical results are provided in Appendix~\ref{apx : Additional illustrations}, which also contains the detailed definitions of the examples. 

In a first example, we consider the TTSA defined in Eq.~\eqref{eq:example_smooth_apx}  of Appendix~\ref{apx : Additional illustrations}. In Figure~\ref{fig:various illustrations}(a), we plot a trajectory of the iterates $(x_k,y_k)$ as a function of $k$ for $\alpha=0.005$ and $\beta=\alpha^{1.5}$. In this example, $h(y)=y$ and $y^*\approx 0.57$. As observed on this figure, the fast iterate $x_k$ tracks $h(y_k)=y_k$ but with a relatively large error. When $k$ goes to infinity, both iterates get close to $y^
*\approx 0.57$. This illustrates the vanishing term present in Theorems~\ref{thm : mse} and \ref{thm : bias}.  In Figure~\ref{fig:various illustrations}(b), we show the MSE of $x$ and $y$ as a function of $\alpha$, for the case $\beta=\alpha^{1.5}$. We observe that the MSE of $x$ is close to $\alpha$ and that the MSE of $y$ is close to $\beta=\alpha^{1.5}$. Hence, the MSE of $y$ is smaller than what is suggested by Theorem~\ref{thm : mse}. We believe that this is due to the fact that the dynamics of $\bar{f}$ and $\bar{g}$ are differentiable in $y^*$ (and not just Lipschitz) and that for this particular case stronger guarantees for the MSE of $y$ could be obtained.  

In our second example, we consider the TTSA where the fast variable $x_k$ tries to track $y_k$ by applying $x_{k+1} = x_k + \alpha (y_k - x_k)$, while $y_k$ tries to escape from $x_k$ by applying $y_{k+1} = y_k + \beta \frac{(y_k - x_k)}{|y_k-x_k|^\gamma}$ for some $\gamma<1$.  As shown in Figure~\ref{fig:various illustrations}(c), for $\gamma=0.95$, the MSE of $x$ in this example is essentially equal to $\beta^2/\alpha^2$, independently of $\alpha$ (remark that the figure contains two curves for each of the values $\alpha=0.01$ and $\alpha=0.5$, but they are superposed). This shows that the dependence on $\beta^2/\alpha^2$ cannot be avoided.  \newcontent{To be more precise, this example does not satisfy A.\ref{assumption: Lip and Mono of g} (because $g$ is not Lipschitz nor strong monotonic), but our method to obtain the bound on $\mathrm{MSE^x}$ (in \eqref{eq:MSE_x}) does not use A.\ref{assumption: Lip and Mono of g}. This means that the term $\beta^2/\alpha^2$ that is present in \eqref{eq:MSE_x} cannot be improved without using Asssumption~A.\ref{assumption: Lip and Mono of g}. The bounds of all of our Theorems contain terms in $\beta^2/\alpha^2$ that all come from $\mathrm{MSE^x}$. As we were unable to find an example that has this $\beta^2/\alpha^2$ term and that satisfies A.\ref{assumption: Lip and Mono of g}, it is possible that the dependence on $\beta^2/\alpha^2$ could be removed under A.\ref{assumption: Lip and Mono of g} but this would require a radically different analysis.  We refer to Appendix~\ref{apx : Additional illustrations} for a detailed discussion.}


\label{sec::exp}
\section{Conclusion}
In this paper, we carry out a finite time analysis on the bias and MSE of nonlinear two time-scale stochastic approximation (TTSA) under constant step-sizes. We consider a general model with Markovian noise where the transition kernel can depend on the iterates $(x_k,y_k)$, this allows for broad applications, for instance in reinforcement learning. We show that the bias and MSE are upper bounded by vanishing initialization terms plus a residual term that scales as $O(\alpha + {\beta^2 / \alpha^2})$ and we show that our analysis is tight. 
While our current framework establishes robust theoretical limits for TTSA, there are several directions for future work. The first is to improve the bounds by (slightly) strenghening the assumptions: for instance, could the MSE bounds be improved if we add B.\ref{assump : derivability} to Theorem~\ref{thm : mse}? 
Also, one may want to relax the finite state-space restriction on the Markovian noise. This will require careful technical adjustments, particularly in validating the Poisson equation solutions required to bound the noise impact in continuous or unbounded spaces. Furthermore, replacing the strict compactness assumption on the iterate trajectory with non-uniform bounds would capture a larger class of practical, unconstrained algorithms. 

\bibliographystyle{plain}
\bibliography{reference.bib}


\appendix


\clearpage 


\section{Main notations}


In this section, we recall some notation and main variables that are used in all appendix. 

\paragraph{Spaces and norms}
Let $x\in \mathbb R ^d$. The quantity $\|x\|$ is the Euclidean norm $\|x\| = (\sum_{i=1}^d x_i^2)^{1/2}$. For $y \in \mathbb R ^d$ we denote the inner product between $x$ and $y$ by $\langle x, y \rangle $. For a square matrix $M \in \mathbb R^{n\times n }$ we define the operator norm $\opnorm{M} = \sup_{\|x\| =1}\|Mx\|$. For a positive definite matrix $L$ we define the norm induced by $L$ as $\|x\|_L = \sqrt{x^T L x}$. We define the operator norm induced by $L$ by $ \|X\|_L = \sup_{\|u\|_L =1}\|Xu\|_L$. The quantity $\sigma_{\max}(L)$ denotes the maximal singular value of $L$ and $\lambda_{\max}(L)$ is its maximal eigenvalue, and the condition number of $L$ is $\kappa(L)=\tfrac{\lambda_{\max}(L)}{\lambda_{\min}(L)}$. 

\paragraph{Probabilities}
For a discrete random variable $\xi$ in a countable space $\Xi$, we have for a probability measure $\pi$ on $\Xi$ the notation $\Exp_{\xi \sim \pi}[f(x,\xi)] = \sum_{\xi\in \Xi} f(x,\xi) \pi(\xi)$ and for a kernel $P : \Xi \times \Xi \rightarrow [0,1]$ we have for $\epsilon \in \Xi$, $\Exp_{\xi \sim P(.|\epsilon)}[f(x,\xi)] = \sum_{\xi\in \Xi} f(x,\xi) P(\xi|\epsilon)$. For an event $E$ the quantity $\indicator{E}$ equals $1$ if $E$ is true and $0$ otherwise.

\paragraph{Derivatives} For a multivariate function $\bar f : \Rx \times \Ry \rightarrow \Rx$, $\nablafx(x,y)$ denotes the partial Jacobian with respect to the first argument $(\nablafx(x,y))_{i,j} = {\partial \overline f_i(x,y)\over \partial x_j}$ and $\nabla_y \overline f(x,y)$ denotes the partial Jacobian with respect to the second argument. 

\paragraph{Important variables.} In what follows, we will use the notation:
\begin{align*}
    \x{k} &\triangleq x_k - h(y_k) && \text{ Tracking error of $x_k$ with respect to $h(y_k)$. Note that $\mathrm{MSE}^x_k = \Expect{\norm{\x{k}}^2}$} \\
    \tildex{k} &\triangleq x_k-h(y^*)  && \text{ Tracking error of $x_k$ with respect to $h(y^*)$. Note that $\mathrm{Bias}^x_k = \Expect{\tildex{k}}$}\\
    \y{k} &\triangleq y_k - y^* && \text{ Tracking error of $y_k$. Note that $\mathrm{MSE}^y_k = \Expect{\norm{\y{k}}^2}$ and $\mathrm{Bias}^y_k = \Expect{\y{k}}$}
\end{align*}

\newcontent{
\paragraph{Step sizes} To obtain our result on the bias (Theorem~\ref{thm : bias}), we require a few conditions on the step sizes from the contraction under the Lyapunov norms passing by the existence and boundedness of the solutions of the Riccati and Sylvester equations to the invertibility conditions on the bias dynamics. We compile all the conditions here so we can refer to them easily in the statement of our Theorem~\ref{thm : bias}. In what follows, we define $\radiusU = \radiusUfixed$ where $M_J$ is a positive constant satisfying $\opnorm{\J{yx}} \leq M_J$.
}

\newcontent{
We start with the conditions on $\alpha^*$:
\begin{equation}
    \begin{aligned}
    \alpha^* \triangleq \min \bigg \{ & {1\over \muf}, \\
    & {1\over 4 \opnorm{\LyapF}(\LyapFNorm{ \J{xx} }^2+ {\epsStar{4}}^2 \kappa(L)(\Rbound)^2\LyapFNorm{\J{yx} }^2)}, \\
    & {1\over \opnorm{\J{xx}} + (\Rbound)\opnorm{\J{yx}} } \bigg \}
    \end{aligned}
    \label{eq: alpha star condition for bias thm}
\end{equation}
where the first upper bound comes from the MSE Theorem, the second bound gives us the contraction of the fast bias dynamics and can be found in \eqref{eq: alpha condition fast Lyapunov} where $\epsStar{4} = \min\{ \eqref{eq: Riccati eps for the ball}, \eqref{eq: Riccati eps for the contraction}\}$ and the last condition is for the invertibility of the fast bias dynamics and can be found in \eqref{eq: alpha for invertible fast bias}.
}

\newcontent{
Second, we enumerate the conditions on $\beta^*$: 
\begin{equation}
    \begin{aligned}
    \beta^* \triangleq \min \bigg \{ & {1\over \mug}, \\
    &{1 \over 4 \opnorm{G} ( 2\kappa(G) \opnorm{\Delta} M_J\radiusU+  \LyapGNorm{\Delta }^2 + \kappa(G) M_J^2 \radiusU^2)}, \\
    & {1 \over \opnorm{\Delta} + \radiusU \opnorm{\J{yx} }}  \bigg \}
    \end{aligned}
    \label{eq: beta star condition for bias thm}
\end{equation}
where the first condition is required for the results of the MSE theorem to hold, the second condition is reqruired by the contraction of the slow bias dynamics and can be found in \eqref{eq: condition beta contraction}, the last condition is about the invertibility of the slow dynamics and can be found in \eqref{eq: beta invertibility condition for slow bias}.}

\newcontent{
Finally, we define $\omega^*$ to be: 
\begin{equation}
    \begin{aligned}
    \omega^* \triangleq \min \bigg\{ & 1 , \\
     &{\radiusU \over  \opnorm{  \J{xx}^{-1}}  (\Rbound) ( \opnorm{\J{yy}} +(\Rbound) \opnorm{\J{yx}})}, \\
     & \frac{1}{\opnorm{ \J{xx}^{-1} } \left( \opnorm{ \J{yy} } + 2 \opnorm{ \J{yx} } (\Rbound) \right)}, \\
     &{1  \over  \opnorm{\J{yy}}  \opnorm{\J{xx}^{-1}} + 2 (\Rbound) \opnorm{ \J{yx}} \opnorm{ \J{xx}^{-1}}},\\
     & {1 \over 2 \opnorm{\LyapF} (\Rbound) (2   \tilde \alpha^* \kappa (L) \opnorm{\J{yx}}\opnorm{\J{xx}}  + 2 \kappa(L)^{(1/2)} M_J)}\bigg\}
    \end{aligned}
    \label{eq: omega star condition for bias thm}
\end{equation}
where the first condition is to make the results of the MSE hold, the second and third conditions come from the existence and boundedness of the Riccati equation and can be found in \eqref{eq: Riccati eps for the ball}, \eqref{eq: Riccati eps for the contraction}, the fourth condition, alongside the third, is for the existence and boundedness of the solution of the Sylvester equation and can be found in \eqref{eq: eps condition sylvester ball}, the fifth condition is needed for the contraction of the fast bias dynamics and can be found in \eqref{eq: eps condition for fast lyapunov contraction} with noting that $\tilde \alpha^*$ in the denominator is independent of $\alpha$ and $\beta$ and can be found in \eqref{eq: alpha condition fast Lyapunov}.
}

\section{Proof of Theorem~\ref{thm : mse}: MSE of TTSA}\label{section : mse detailed proofs}

In this section we prove the upper bound of Theorem~\ref{thm : mse} on the MSE of TTSA. We restate the theorem, before detailing the proof.

\mseThm*

\begin{proof} We divide the proof into two main parts, the first part is dedicated to the fast iterate MSE and the second part is for the slow iterate.

\paragraph{Upper bounding the MSE of the fast iterate} We recall that    $\x{k} \triangleq x_k - \H{k}$. We also introduce the shorthand
\begin{equation*}
\psi_k \triangleq \F{k}  - \barF{k}, \quad \Dh{k} \triangleq \H{k+1} - \H{k} \; .
\end{equation*}
Using the update \(x_{k+1}=x_k-\alpha \F{k}\), we obtain
\[
    \x{k+1}
    = x_{k+1}-\H{k+1}
    = \x{k}-\alpha \F{k}-\Dh{k},
\]
so expanding the squared norm gives
\begin{align*}
\norm{\x{k+1}}^2 =& \norm{x_{k+1} - \H{k+1}}^2\\ 
= & \norm{\x{k} -\alpha \F{k} -\Dh{k}}^2 \\
=& \norm{\x{k}}^2 + \alpha^2 \norm{\F{k}}^2 - 2\alpha \inner{\x{k}}{\barF{k}} \\
& + \|\Dh{k}\|^2 -2 \inner{\x{k}}{\Dh{k}} + 2\alpha \inner{\F{k}}{\Dh{k}} -2\alpha \inner{\x{k}}{\psi_{k}}
\end{align*}
Rearranging the terms to facilitate the organization of the bound:
\begin{align}
  \norm{\x{k+1}}^2  =& 
  \norm{\x{k}}^2 - 2\alpha \inner{\x{k}}{\barF{k}}  - 2\ \inner{\x{k}}{\Dh{k}}
 \label{eq: decay}\\ 
& +
\alpha^2 \norm{\F{k}}^2
 \label{eq: noisy dynamics}\\
&  +
\|\Dh{k}\|^2 + 2\alpha \inner{\F{k}}{\Dh{k}}
\label{eq: high orders terms}\\ 
& - 2\alpha \inner{\x{k}}{\psi_{k}}
 \label{eq: hard term} \; .
\end{align}
We first control the bounded second-order terms. By Lemma~\ref{lemma: bounds},
there exist constants \(\boundF,\boundH>0\) such that
\[
    \norm{\F{k}}\leq \boundF,
    \qquad
    \norm{\Dh{k}}\leq \beta \boundH .
\]
Therefore, by the Cauchy--Schwarz inequality,
\begin{align*}
\eqref{eq: noisy dynamics} + \eqref{eq: high orders terms} &\leq \alpha ^2 \norm{\F{k}}^2 + \norm{\Dh{k}}^2 + 2\alpha \norm{\F{k}}\norm{\Dh{k}}\\
& \leq \alpha ^2 \; \boundF^2  + \beta ^2 \; \boundH^2 + 2 \alpha \beta \; \boundF \boundH \;.
\end{align*}

We now extract the contraction term. By property of \(h\), for every
\(y\in\domY\),
\[
    \overline f(h(y),y)=0.
\]
Hence, recalling that \(\x{k}=x_k-h(y_k)\), we can write
\begin{align*}
-2\alpha\inner{\x{k}}{\barF{k}}
&=
-2\alpha
\inner{x_k-h(y_k)}
{\overline f(x_k,y_k)-\overline f(h(y_k),y_k)}
\\
&\leq
-2\alpha\muf\norm{x_k-h(y_k)}^2
\\
&=
-2\alpha\muf\norm{\x{k}}^2,
\end{align*}
where the inequality follows from the strong monotonicity assumption
A.\ref{assumption: Lip and Mono of f} applied to the first argument of
\(\overline f\), with \(y=y_k\).

The third term in line \eqref{eq: decay} is controlled by Cauchy--Schwarz and Young's inequality. Precisely:
\[
    -2\inner{\x{k}}{\Dh{k}}
    \leq 2\norm{\x{k}}\norm{\Dh{k}}
    \leq \alpha\muf\norm{\x{k}}^2
        + {1\over \alpha\muf}\norm{\Dh{k}}^2
    \leq \alpha\muf\norm{\x{k}}^2
        + {\beta^2\boundH^2\over \alpha\muf},
\]
where the last inequality follows from Lemma~\ref{lemma: bounds}. Combining the previous bounds on
\eqref{eq: decay}--\eqref{eq: high orders terms}, we obtain
\begin{align*}
    \norm{\x{k+1}}^2
    &\leq
    (1-\alpha\muf)\norm{\x{k}}^2
    -2\alpha\inner{\x{k}}{\psi_k}
    + \alpha^2\boundF^2
    + \beta^2\boundH^2
    + 2\alpha\beta\boundF\boundH
    + {\beta^2\boundH^2\over \alpha\muf}.
\end{align*}
For brevity, set
\[
    \mathrm{M}_{\alpha,\beta}
    \triangleq
    \alpha^2\boundF^2
    + \beta^2\boundH^2
    + 2\alpha\beta\boundF\boundH
    + {\beta^2\boundH^2\over \alpha\muf}.
\]
Iterating the previous inequality gives
\begin{align*}
    \norm{\x{k+1}}^2
    &\leq
    (1-\alpha\muf)^{k+1}\norm{\x{0}}^2
    -2\alpha
    \sum_{j=0}^{k}
    (1-\alpha\muf)^{k-j}
    \inner{\x{j}}{\psi_j}
    + \sum_{j=0}^{k}
    (1-\alpha\muf)^{k-j}
    \mathrm{M}_{\alpha,\beta}.
\end{align*}
Taking expectations yields
\begin{align}
    \Expect{\norm{\x{k+1}}^2}
    \leq&
    (1-\alpha\muf)^{k+1}\norm{\x{0}}^2
    \label{eq: init vanish fast mse}
    \\
    &+
    \mathrm{M}_{\alpha,\beta}
    \sum_{j=0}^{k}(1-\alpha\muf)^j
    \label{eq: easy geometric}
    \\
    &-
    2\alpha
    \Expect{
    \sum_{j=0}^{k}
    (1-\alpha\muf)^{k-j}
    \inner{\x{j}}{\psi_j}
    }.
    \label{eq: poisson term}
\end{align}
The first term \eqref{eq: init vanish fast mse} is the contribution of the initial condition and decays
geometrically. The two remaining terms correspond to the steady-state error induced by the constant
step-size. We first bound the deterministic geometric term \eqref{eq: easy geometric}:
\begin{align*}
    \mathrm{M}_{\alpha,\beta}
    \sum_{j=0}^{k}(1-\alpha\muf)^j
    &=
    \mathrm{M}_{\alpha,\beta}
    {1-(1-\alpha\muf)^{k+1}\over \alpha\muf}
    \leq
    {\mathrm{M}_{\alpha,\beta}\over \alpha\muf}.
\end{align*}
It remains to control the Markovian-noise term \eqref{eq: poisson term}. To do so, we invoke 
Lemma~\ref{lemma : poisson mse}, that we detail and prove in Appendix \ref{section: appendix lemmas}.
\begin{align*}
     \boundMSEFastPoisson
\end{align*}
Combining the above results provide
\begin{align*}
    \Expect{\norm{\x{k+1}}^2}
    \leq&
    (1-\alpha\muf)^{k+1}\norm{\x{0}}^2
    \nonumber\\
    &+
    \alpha
    \left(
        {4\boundXi\over 1-\alpha \muf}
        +2\boundXi
        +{\boundF^2
        +2(\boundXi \newcontent{(1 + |\Xi|\boundXY \Lk)}+\Lvf(\boundXY+\boundH))\boundF
        \over \muf}
    \right)    \\
    &+
    \beta
    \left(
        {2(\boundXi+\boundF)\boundH
        +2\Lvf\boundG(\boundXY+\boundH) + \newcontent{|\Xi|\boundXY \Lk  \boundXi \boundG}
        \over \muf}
    \right)
    +
    {\beta^2\over \alpha}
     {\boundH^2\over \muf}
        +
    {\beta^2\boundH^2\over \alpha^2\muf^2}.
\end{align*}
Using that $\frac{\beta^2}{\alpha} \leq \beta$ because $\beta\leq \alpha$, we are ready to formally define the constants 
\begin{align}
    B_1 &\coloneqq 4 \boundXi \;, \label{eq::B_1} \\
    B_2 &\coloneqq 2 \boundXi
        +{\boundF^2
        +2(\boundXi\newcontent{(1 + |\Xi|\boundXY \Lk)}+\Lvf(\boundXY+\boundH))\boundF
        \over \muf} \;,\label{eq::B_2} \\
    B_3 &\coloneqq 
        {2(\boundXi+\boundF)\boundH
        +2\Lvf\boundG(\boundXY+\boundH)+ \newcontent{2|\Xi|\boundXY \Lk  \boundXi \boundG} + \boundH^2
        \over \muf} \;, \; \text{and} \label{eq::B_3} \\
    B_4 &\coloneqq     {\boundH^2\over \muf^2}\;. \label{eq::B_4} 
\end{align}
This proves the desired MSE bound for the fast iterate, and allows us to define the non-vanishing term of the bound, 
\[\MSEx \coloneqq \alpha \;\left(\frac{\mathrm{B}_1}{1-\alpha \mu_f} + \mathrm{B}_2\right)  + \beta \; \mathrm{B}_3 + {\beta^2 \over \alpha^2 }\mathrm{B}_4 \;.\]


\paragraph{MSE of the slow iterate}
We now turn to the MSE of the slow iterate. The argument follows the same
structure as for the fast iterate, with one additional difficulty: the update of
\(y_k\) is evaluated at \((x_k,y_k)\), whereas the contraction assumption is
available for the reduced dynamics evaluated at \((h(y_k),y_k)\). We therefore
need to control the tracking error of the fast iterate, which enters the slow
recursion as an additional residual term. 

For convenience, introduce
\newcontent{
\[
    \Dg{k} \triangleq \barG{k}- \GH{k},
    \qquad
    \phi_k \triangleq \G{k}- \barG{k}.
\]}
Thus, \(\Dg{k}\) measures the error due to replacing \(x_k\) by \(h(y_k)\) in
the noisy \newcontent{averaged} dynamics, while \(\phi_k\) is the Markovian-noise fluctuations
around the \newcontent{averaged} slow dynamics.

Using the update \(y_{k+1}=y_k-\beta\G{k}\), we have
\[
    \y{k+1}=\y{k}-\beta\G{k}.
\]
Expanding the squared norm and adding/subtracting the reduced dynamics gives
\begin{align}
    \norm{\y{k+1}}^2
    =&
    \norm{\y{k}-\beta\G{k}}^2
    \nonumber\\
    =&
    \norm{\y{k}}^2
    -2\beta\inner{\y{k}}{\G{k}}
    +\beta^2\norm{\G{k}}^2
    \nonumber\\
    =&\newcontent{
    \norm{\y{k}}^2
    -2\beta\inner{\y{k}}{\GH{k}}
    -2\beta\inner{\y{k}}{\barG{k}-\GH{k}}}
    \nonumber\\
    &\quad
    \newcontent{+\beta^2\norm{\G{k}}^2
    -2\beta\inner{\y{k}}{\G{k}-\barG{k}}}
    \nonumber\\
    =&
    \norm{\y{k}}^2
    -2\beta\inner{\y{k}}{\GH{k}}
    -2\beta\inner{\y{k}}{\Dg{k}}
    \label{eq: norm error slow 1}
    \\
    &\quad
    +\beta^2\norm{\G{k}}^2
    \label{eq: norm error slow 2}
    \\
    &\quad
    -2\beta\inner{\y{k}}{\phi_k}.
    \label{eq: norm error slow 3}
\end{align}

By the strong monotonicity assumption A.\ref{assumption: Lip and Mono of g}, the
first two terms in \eqref{eq: norm error slow 1} yield a contraction. The
additional term involving \(\Dg{k}\) measures the error induced by evaluating
the slow update at \(x_k\) instead of \(h(y_k)\). More precisely,
\begin{align*}
\eqref{eq: norm error slow 1}   &=\norm{\y{k}}^2
    -2\beta \inner{\y{k}}{\GH{k}}
    -2\beta\inner{\y{k}}{\Dg{k}}
    \\
    &\leq
    \norm{\y{k}}^2
    -2\beta \mug \norm{\y{k}}^2
    +2\beta \norm{\y{k}}\norm{\Dg{k}}
    \\
    &\leq
    \norm{\y{k}}^2
    -2\beta \mug \norm{\y{k}}^2
    +\beta \mug\norm{\y{k}}^2
    + {\beta \over \mug}\norm{\Dg{k}}^2
    \\
    &\leq
    (1-\beta \mug)\norm{\y{k}}^2
    +{\beta \over \mug}\norm{\Dg{k}}^2
    \\
    &\leq
    (1-\beta \mug)\norm{\y{k}}^2
    +\beta { \Lg^2 \over \mug}\norm{\x{k}}^2 .
\end{align*}
In the first inequality, we used the strong monotonicity of the reduced slow
dynamics and Cauchy--Schwarz. In the second inequality, we used Young's
inequality \(2ab\leq \epsilon a^2+b^2/\epsilon\) with \(\epsilon=\mug\).
The last inequality follows from the Lipschitz property of \(\overline g(\cdot,\cdot)\):
\[
    \newcontent{\norm{\Dg{k}}
    =
    \norm{\barG{k}-\GH{k}}
    \leq
    \Lg \norm{x_k-h(y_k)}
    =
    \Lg\norm{\x{k}}}
\]
We take $\beta < \mug^{-1}$ so that the factor \(1-\beta\mug\) is strictly contractive and
bounded away from zero. The remaining deterministic term in \eqref{eq: norm error slow 2} is bounded
using Lemma~\ref{lemma: bounds}: there exists \(\boundG>0\) such that
\[
    \beta^2\norm{\G{k}}^2 \leq \beta^2\boundG^2 .
\]
Combining the previous bounds on
\eqref{eq: norm error slow 1}--\eqref{eq: norm error slow 3}, we obtain
\[
    \norm{\y{k+1}}^2
    \leq
    (1-\beta\mug)\norm{\y{k}}^2
    + { \Lg^2 \over \mug}\beta\norm{\x{k}}^2
    + \beta^2\boundG^2
    -2\beta\inner{\y{k}}{\phi_k}.
\]
Iterating this recursion from time \(0\) to time \(k\), and then taking
expectations, gives
\begin{align}
    \Expect{\norm{\y{k+1}}^2}
    \leq&
    (1-\beta \mug)^{k+1}\norm{\y{0}}^2
    + \boundG^2 \beta^2
    \sum_{j=0}^k (1-\beta \mug)^{j}
    \label{eq: mse slow 1}
    \\
    &+
    { \Lg^2 \over \mug} \beta
    \sum_{j=0}^k
    (1-\beta \mug)^{k-j}
    \Expect{\norm{\x{j}}^2}
    \label{eq: mse slow 2}
    \\
    &-
    2\beta
    \Expect{
    \sum_{j=0}^k
    (1-\beta \mug)^{k-j}
    \inner{\y{j}}{\phi_j}
    }.
    \label{eq: mse slow 3}
\end{align}

The first term in \eqref{eq: mse slow 1} is the contribution of the initial
condition and decays geometrically. The second term in \eqref{eq: mse slow 1}
is a geometric sum, and therefore
\begin{align*}
    \boundG^2 \beta^2
    \sum_{j=0}^k (1-\beta \mug)^j
    &=
    \boundG^2 \beta^2
    {1-(1-\beta \mug)^{k+1} \over \beta \mug}
    \\
    &\leq
    {\boundG^2 \over \mug}\beta .
\end{align*}

We now control the term \eqref{eq: mse slow 2}, which is the geometrically
weighted accumulation of the fast-iterate MSEs. 
Using the upper bound we obtained for the fast iterate, we get
\begin{align*}
    { \Lg^2 \over \mug} \beta
    \sum_{j=0}^k
    (1-\beta \mug)^{k-j}
    \Expect{\norm{\x{j}}^2}
    \leq&
    { \Lg^2 \over \mug} \beta
    \sum_{j=0}^k
    (1-\beta \mug)^{k-j}
    \Big(
        (1-\alpha\muf)^{j+1}\norm{\x{0}}^2
        + \MSEx
    \Big)
    \\
    \leq&
    {\Lg^2 \over \mug^2}\MSEx
    \\& +
    { \Lg^2 \over \mug} \beta (1-\alpha \muf)  
    \norm{\x{0}}^2\underbrace{\sum_{j=0}^k
    (1-\beta \mug)^{k-j}
    (1-\alpha\muf)^{j}}_{S_k} \;.
    \\
\end{align*}
For the term $S_k$, we simply remark that, for $j\in [k]$, it holds that 
\[(1-\beta \mug)^{k-j}
    (1-\alpha\muf)^{j} \leq (1-(\alpha \muf \wedge \beta \mug))^k. \]
Therefore, we can upper bound each term of the sum uniformly and obtain that 
\begin{align*}
    S_k \leq (k+1) (1-(\alpha \muf \wedge \beta \mug))^k \;. 
\end{align*}
Note that the bound is tight for $\alpha \muf = \beta \mug$. This leads to 

\begin{align*}
    { \Lg^2 \over \mug} \beta
    \sum_{j=0}^k
    (1-\beta \mug)^{k-j}
    \Expect{\norm{\x{j}}^2}
    \leq&
    {\Lg^2 \over \mug^2}\MSEx
    + \beta \; {\Lg^2 \over \mug}  
    \norm{\x{0}}^2 (k+1) (1- (\alpha \muf \wedge \beta \mug))^{k+1}  .
\end{align*}

It remains to control the Markovian-noise term \eqref{eq: mse slow 3}. As in
the fast-iterate analysis, we use the Poisson-equation decomposition from
Lemma~\ref{lemma : poisson mse}. In particular,
\begin{align*}
    -2\beta
    \Expect{
    \sum_{j=0}^k
    (1-\beta\mug)^{k-j}
    \inner{\y{j}}{\phi_j}
    }
    \leq&
    2\beta
    \abs{
    \Expect{
    \sum_{j=0}^k
    (1-\beta\mug)^{k-j}
    \inner{\y{j}}{\phi_j}
    }}
    \\
    \leq & 2 \beta \parentheses{ {(\Lvg + \newcontent{\boundXi \Lk |\Xi|} )\boundXY \boundG  \over \mug} +   \boundXi \parentheses{{3 - \beta \mug \over 1 - \beta \mug} + {\boundG \over \mug} }} \\
    & + 2 \alpha {(\Lvg + \newcontent{\boundXi \Lk |\Xi|} ) \boundXY \boundF\over \mug} 
\end{align*}
Combining this bound with the bounds on
\eqref{eq: mse slow 1} and \eqref{eq: mse slow 2}, we obtain
\begin{align*}
    \Expect{\norm{\y{k+1}}^2}
    \leq&
    (1-\beta\mug)^{k+1}\norm{\y{0}}^2
    + {\boundG^2\over \mug}\beta
    \\
    &+\beta \; {\Lg^2 \over \mug}  
    \norm{\x{0}}^2 (k+1) (1- (\alpha \muf \wedge \beta \mug))^{k+1}
    \\
    &+
    {\Lg^2\over \mug^2}\MSEx
    +
    2\beta
    \parentheses{
        {(\Lvg + \newcontent{\boundXi \Lk |\Xi|} )\boundXY \boundG + \boundXi \boundG \over \mug}
        + { 2 \boundXi \over  1 - \beta \mug}
        + \boundXi
    } \\ 
    &+ 2 \alpha {(\Lvg + \newcontent{\boundXi \Lk |\Xi|} ) \boundXY \boundF\over \mug}.
\end{align*}
It remains to identify the constants, 
\begin{align}
    D_1 &\coloneqq \frac{L_g^2}{\mug}, \label{eq::D_1} \\
    D_2 &\coloneqq \frac{L_g^2}{\mug^2} \;,\label{eq::D_2} \\
    D_3 &\coloneqq 
        4 \boundXi \;, \; \text{and} \label{eq::D_3} \\
    D_4 &\coloneqq \frac{\boundG^2}{\mug}  +   2
    \parentheses{
        {(\Lvg + \newcontent{\boundXi \Lk |\Xi|} )\boundXY \boundG + \boundXi \boundG \over \mug}
        +  \boundXi
    }\;. \label{eq::D_4} \\
    D_5 &\coloneqq 2{(\Lvg + \newcontent{\boundXi \Lk |\Xi|} ) \boundXY \boundF\over \mug}\;. \label{eq::D_5}
\end{align}
This proves the desired MSE bound for the slow iterate and concludes the proof
of the theorem.\end{proof}


\section{Proof of Theorem~\ref{thm : lower bound MSE}: Lower bounds}

Theorem~\ref{thm : lower bound MSE} shows that there exists a TTSA whose MSE is larger than $\Omega(\alpha)$ for both the fast and the slow iterates. This result is a consequence of Theorem~\ref{thm : lower bound} below where we exhibit an example that attains this lower bound, and where we also provide an example has an MSE of order $\Omega(\alpha)$ for both $x$ and $y$. 

\begin{theorem}\label{thm : lower bound}
    There exist a constant $C$ and a constant $\alpha_0>0$ and:
    \begin{enumerate}[label=(\roman*)]
        \item a TTSA satisfying all Assumptions of Theorem~\ref{thm : mse} such that for all $\beta\le\alpha\le\alpha_0$:
            \begin{align*}
                \liminf_{k\to\infty} \Expect{\|x_k-h(y_k)\|^2} \ge C\alpha\qquad \text{ and }\qquad 
                \liminf_{k\to\infty} \Expect{\|y_{k}-y^*\|^2} \ge C\alpha;
            \end{align*}
        \item a TTSA satisfying all Assumptions of Theorem~\ref{thm : bias} such that for all $\beta\le\alpha\le\alpha_0$:
            \begin{align*}
            \liminf_{k\to\infty} \|\Expect{x_k}-h(y^*)\| \ge C(\alpha + \beta)\qquad \text{ and }\qquad 
            \liminf_{k\to\infty} \|\Expect{y_k}-y^*\| \ge C(\alpha + \beta).
            \end{align*}
    \end{enumerate}
\end{theorem}

\begin{proof}
    The proof of (i) and (ii) uses the same fast iterate $(x_k)$ that is defined as: 
    \begin{align*}
        x_{k+1} &= x_k - \alpha(f(x_k) + \xi_{k+1}),
    \end{align*}
    where \newcontent{$f(x) = {x \over{1+x}}$. The sequence of noises $(\xi_{k})$ is a sequence of zero-mean \emph{i.i.d.} scaled Rademacher random variables (\emph{i.e.}, such that $\proba{\xi_k=1/2} = \proba{\xi_k=-1/2}=0.5$)}. Note that the function $f$ is Lipschitz continuous with a Lipschitz-continunous derivative. We assume that $\alpha\in(0,1)$ and we set $x_0=0$. 
    
    We first show by induction on $k$ that the sequence $x_k$ is bounded almost surely, and more precisely that $x_k\in[-1/2, 1]$. Assume that this holds for some $k$ and recall that
    \begin{align*}
        x_{k+1} = x_k - \alpha\left(\frac{x_k}{1+x_k} + \xi_{k+1}\right).
    \end{align*}
    We distinguish two cases:
    \begin{enumerate}
        \item If $\frac{x_k}{1+x_k} + \xi_k\ge0$, then 
        \begin{align*}
            x_{k+1} &\le x_k \le 1\\
            x_{k+1} &\ge x_k - \left(\frac{x_k}{1+x_k} + \xi_k\right) = \frac{x_k^2}{1+x_k} - \frac12 \ge -\frac12,
        \end{align*}
        where we used that $\alpha\le 1$ in the second line and the fact that $x^2/(1+x)\ge0$.
        \item If  $\frac{x_k}{1+x_k} + \xi_k\le0$, then 
        \begin{align*}
            x_{k+1} &\ge x_k\ge -\frac12\\
            x_{k+1} &\le x_k - \left(\frac{x_k}{1+x_k} + \xi_k\right) = \frac{x_k^2}{1+x_k} + \frac12 \le 1,
        \end{align*}
    \end{enumerate}

    As the behavior of $x_k$ does not depend on the slow iterates $(y_k)_k$, the sequence $(x_k)$ can be analyzed with standard tools from one-time-scale stochastic approximation. Here, we have $x^*=h(y_k)=0$. In particular, \newcontent{since near 0 this function has a second order approximation of $x-x^2$},from \cite{allmeier2024computingbiasconstantstepstochastic}, we know that the bias and the MSE of $x_k$ are of order $\alpha$ and asymptotically equal to:
    \begin{align*}
        \lim_{k\to\infty}\Expect{x_k} = \lim_{k\to\infty}\Expect{(x_k)^2} = \frac\alpha{2c} + O(\alpha^2).
    \end{align*}
    \newcontent{for some positive constant $c$}. Moreover, from \cite[Chapter~10]{Kushner1997StochasticAA}, when $\alpha$ is small and $k$ is large, the variable $x_k$ is close to a Gaussian distribution of variance $\alpha/2c$, which implies 
    \begin{align}
        \label{eq:apx proof lower bound 2}
        \lim_{k\to\infty}\Expect{|x_k|} = \sqrt{\alpha / d \pi}  + O(\alpha)
    \end{align}

    \newcontent{for some positive constant $d$}. For the variable $y_k$, we consider two different cases to prove either (i) or (ii):
    \begin{enumerate}[label=(\roman*)]
        \item To prove \emph{(i)}, we set
        \begin{align*}
            y_{k+1} = y_k - \beta ( y_k - |x_k|).
        \end{align*}
        Unrolling the recurrence equation $y_{k+1} = y_k(1-\beta)+\beta|x_k|$ shows that
        \begin{align*}
            y_{k+1} &= \beta \sum_{i=0}^{k} (1-\beta)^{k-i}|x_i|.
        \end{align*}
        Combined with \eqref{eq:apx proof lower bound 2}, this shows that
        \begin{align*}
            \lim_{k\to\infty} \Expect{y_{k}} &= \beta \sum_{i=0}^{\infty} (1-\beta)^{i} \lim_{k\to\infty} \Expect{|x_k|} \\
            &= \lim_{k\to\infty} \Expect{|x_k|}\\
            &= \sqrt{\frac{\alpha}{d \pi}} + O(\alpha).
        \end{align*}
        By using Jensen's inequality, this shows that: 
        \begin{align*}
            \lim_{k\to\infty}\Expect{(y_k)^2} &\ge \lim_{k\to\infty}(\Expect{y_k})^2\\
            &=\frac{\alpha}{d \pi} + O(\alpha^{3/2}).
        \end{align*}
        \item To prove \emph{(ii)}, we define $y_k$ as:
        \begin{align*}
            y_{k+1} = y_k - \beta ( y_k - x_k)
        \end{align*}
        This shows that $\lim_{k\to\infty}\Expect{y_k} = \lim_{k\to\infty}\Expect{x_k} = \alpha/2c + O(\alpha^2) $.
    \end{enumerate}
\end{proof}

The proof of the above theorem exhibits an example where the MSE of $y$ can be of order $\Omega(\alpha)$. To construct this example, we used a function $g$ such that $\bar{g}(x,y)=y-|x|$ that satisfies all assumptions of Theorem~\ref{thm : mse} but not the assumptions of Theorem~\ref{thm : bias} because it is not differentiable. In fact, we believe that having a non-differentiable function is needed to obtain the $\Omega(\alpha)$ lower bound for $\mathrm{MSE_y^\infty}$, and that under a stronger assumption of differentiability, one may obtain a tighter bound for the MSE. This is an interesting question for future work. 
\label{apx:lower bound}
\section{Proof of Theorem~\ref{thm : bias} : Bias of TTSA} \label{section: bias detailed proofs}
We start by restating the theorem, before detailing its proof.

\biasThm*
\begin{proof} 
We start by detailing the general outline of the proof. We begin by concatenating the TTSA updates to jointly analyze the bias of the fast and slow variables. Exploiting the differentiability of the underlying dynamics (Assumption B.~\ref{assump : derivability}), we apply a first-order Taylor expansion to linearize the updates, which yields a highly coupled linear dynamical system. To mathematically decouple these dynamics, we introduce a two-stage coordinate transformation. First, we block-triangularize the system by solving an algebraic Riccati equation. Then, we fully diagonalize the system by solving a Sylvester equation.

This transformation projects the bias into a new coordinate space governed by a diagonalized recursion. While this recursion features strong contracting terms, it is corrupted by higher-order perturbation residuals arising from the diagonalization process and the multiplication by the transformation matrices. However, the order of these perturbations will not affect the nature of the contractions due to our assumptions on the step sizes. The diagonal structure then allows us to unroll the recursion and bound each transformed component independently. Finally, because our transformation matrices possess finite norms, and by a simple triangular inequality we get the final bounds on the bias.

\textbf{Concatenating the iterates.} We start by deriving a recursion for the joint bias vector. 
In this proof, we work with deviations from the equilibrium and write
\[
    \tildex{k} \triangleq x_k-h(y^*),
    \qquad
    \text{ and we recall that }
    \tildey{k} = y_k-y^*.
\]
We also introduce the centered Markovian-noise terms
\[
    \psi_k \triangleq \F{k}-\barF{k},
    \qquad
    \varphi_k \triangleq \G{k}-\barG{k}.
\]
For the fast iterate, we obtain
\begin{align*}
    \Expect{\tildex{k+1}}
    &= \Expect{\tildex{k}}
       -\alpha \Expect{\barF{k}}
       -\alpha \Expect{\psi_k}
    \\
    &= \Expect{\tildex{k}}
       -\alpha
       \Expect{
       \overline{f}\bigl(h(y^*)+\tildex{k},\,y^*+\tildey{k}\bigr)}
       -\alpha \Expect{\psi_k}
    \\
    &= \Expect{\tildex{k}}
       -\alpha
       \Expect{
       \overline{f}(h(y^*),y^*)
       + J_{xx}\tildex{k}
       + J_{xy}\tildey{k}
       + R_f(\tildex{k},\tildey{k})}
       -\alpha \Expect{\psi_k}
    \\
    &= \Expect{\tildex{k}}
       -\alpha J_{xx}\Expect{\tildex{k}}
       -\alpha J_{xy}\Expect{\tildey{k}}
       -\alpha \Expect{R_f(\tildex{k},\tildey{k})}
       -\alpha \Expect{\psi_k},
\end{align*}
where we used $\overline{f}(h(y^*),y^*)=0$ in the last line.

Similarly, for the slow iterate,
\begin{align*}
    \Expect{\tildey{k+1}}
    &= \Expect{\tildey{k}}
       -\beta \Expect{\barG{k}}
       -\beta \Expect{\varphi_k}
    \\
    &= \Expect{\tildey{k}}
       -\beta
       \Expect{
       \overline{g}\bigl(h(y^*)+\tildex{k},\,y^*+\tildey{k}\bigr)}
       -\beta \Expect{\varphi_k}
    \\
    &= \Expect{\tildey{k}}
       -\beta
       \Expect{
       \overline{g}(h(y^*),y^*)
       + J_{yx}\tildex{k}
       + J_{yy}\tildey{k}
       + R_g(\tildex{k},\tildey{k})}
       -\beta \Expect{\varphi_k}
    \\
    &= \Expect{\tildey{k}}
       -\beta J_{yx}\Expect{\tildex{k}}
       -\beta J_{yy}\Expect{\tildey{k}}
       -\beta \Expect{R_g(\tildex{k},\tildey{k})}
       -\beta \Expect{\varphi_k},
\end{align*}
where we used $\overline{g}(h(y^*),y^*)=0$.

In both expansions, we applied a first-order Taylor expansion of the averaged
dynamics around the equilibrium $(h(y^*),y^*)$. Since the second derivatives
of $\bar f$ and $\bar g$ are bounded by Assumption~B.\ref{assump : derivability},
there exists constant $L_R>0$ such that, for all $k\geq 0$,
\[
    \norm{R_f(\tildex{k},\tildey{k})}
    \leq
    L_R\bigl(\norm{\tildex{k}}+\norm{\tildey{k}}\bigr)^2,
    \qquad
    \norm{R_g(\tildex{k},\tildey{k})}
    \leq
    L_R\bigl(\norm{\tildex{k}}+\norm{\tildey{k}}\bigr)^2 .
\]

Concatenating the two bias vectors yields the joint recursion
\begin{align}
    \jointbias{k+1}
    =
    \underbrace{\dynamicsmatrix}_{H_{\alpha,\beta}}
    \jointbias{k}
    +
    \underbrace{
    \begin{bmatrix}
          -\alpha \Expect{R_f(\tildex{k}, \tildey{k})}
          -\alpha \Expect{\psi_k}
          \\
          -\beta \Expect{R_g(\tildex{k}, \tildey{k})}
          -\beta \Expect{\varphi_k}
    \end{bmatrix}
    }_{\mathcal R_k}.
    \label{eq:concatenated-bias-recursion}
\end{align}
The matrix $H_{\alpha,\beta}$ is the linearized discrete-time dynamics of the
joint bias. Its block structure shows that the two components are still coupled:
the fast bias depends on the slow bias through the block $J_{xy}$, while the
slow bias depends on the fast bias through the block $J_{yx}$. This coupling is
the main reason why one cannot simply analyze the two bias coordinates
separately.

The next step is therefore to find a change of coordinates adapted to this
two-time-scale linear system. More precisely, following ideas from singular
perturbation theory \cite[Chapter~2]{kokotovic1999singular}, we construct
invertible transformations that turn $H_{\alpha,\beta}$ into a block-diagonal
matrix. In the transformed coordinates, the fast and slow bias components can
then be controlled separately, with their respective contraction rates. In the following, we drop the subscript $\alpha, \beta$ from the matrix $H$ for conciseness of notation.


\textbf{Triangularization through a Riccati equation.}
The first step is to remove the upper-right coupling in the linearized dynamics.
We look for an invertible change of coordinates $T$ such that
\[
    \Pi \triangleq T H T^{-1}
\]
is block lower triangular, or equivalently $H=T^{-1}\Pi T$.

We choose $T$ to be upper triangular with identity blocks on the diagonal:
\begin{align}
    T
    =
    \blockMatrix{I}{U}{0}{I},
    \qquad
    T^{-1}
    =
    \blockMatrix{I}{-U}{0}{I},
    \label{eq:riccati-transformation}
\end{align}
where $U\in\mathbb R^{d_x\times d_y}$ is to be determined. To simplify the
notation, write the block decomposition of $H$ as
\[
    H
    =
    \blockMatrix{A}{B}{C}{D}.
\]
Then a direct computation gives
\begin{align*}
    THT^{-1}
    &=
    \blockMatrix{I}{U}{0}{I}
    \blockMatrix{A}{B}{C}{D}
    \blockMatrix{I}{-U}{0}{I}
    \\
    &=
    \blockMatrix{A+UC}{R(U)}{C}{D-CU},
\end{align*}
where
\[
    R(U)
    \triangleq
    -(A+UC)U+B+UD .
\]
Thus, the upper-right block is zero if and only if $U$ solves
\begin{align}
    R(U)
    =
    -(A+UC)U+B+UD
    =
    0.
    \label{eq:riccati-equation-U}
\end{align}
This is an algebraic Riccati equation for the unknown block $U$ \cite[Eq. 2.8]{kokotovic1999singular}. Any solution
of~\eqref{eq:riccati-equation-U} yields the block lower-triangular form
\[
    \Pi
    =
    THT^{-1}
    =
    \blockMatrix{A+UC}{0}{C}{D-CU}.
\]

Writing the Riccati equation in terms of the blocks of the TTSA dynamics gives
\begin{align}
     \frac{\beta}{\alpha} U J_{yx} U
     + J_{xx} U
     - \frac{\beta}{\alpha} U J_{yy}
     =
     J_{xy}.
     \label{eq:riccati-equation}
\end{align}
This equation is the singular-perturbation Riccati equation associated with the
linearized two-time-scale dynamics. Under Assumption~B.\ref{assump : derivability}, it is shown in \cite[Chapter~2, Eq.~2.13]{kokotovic1999singular} that when $\beta/\alpha$ small enough, this Riccati equation admits a solution of the form
\[
    U
    =
    \sum_{\ell=0}^{\infty}\left(\frac{\beta}{\alpha}\right)^\ell \tilde{U}_\ell ,
\]
where the matrices $(\tilde{U}_\ell)_{\ell\geq 0}$ are independent of $\alpha$ and $\beta$. The matrix $\tilde{U}_0$ can be found by setting $\beta/\alpha=0$ in~\eqref{eq:riccati-equation} which gives
\[
    J_{xx}\tilde{U}_0=J_{xy}.
\]
Since $-J_{xx}$ is Hurwitz by Assumption~B.\ref{assump : derivability}, the matrix
$J_{xx}$ is invertible and $\tilde{U}_0=J_{xx}^{-1}J_{xy}$.
This shows that
\begin{align*}
    U
    &= \tilde{U}_0 
    +\sum_{\ell=1}^\infty \left(\frac{\beta}{\alpha}\right)^\ell \tilde{U}_\ell \\
    &=J_{xx}^{-1}J_{xy} + \frac{\beta}{\alpha}\hat{U}_{\alpha,\beta},
\end{align*}
where $\hat{U}_{\alpha,\beta}\triangleq\sum_{\ell=1}^\infty \left(\frac{\beta}{\alpha}\right)^{\ell-1} \tilde{U}_\ell$. 

\newcontent{From Assumption B.\ref{assump : derivability} and since we are assuming that ${\beta \over \alpha} < \min\{\epsStar{0}, \epsStar{1} \}$ where $\epsStar{0}, \epsStar{1}$ are defined in \eqref{eq: Riccati eps for the ball}, \eqref{eq: Riccati eps for the contraction} and are dependent on the predefined radius $\radiusU= \radiusUfixed$ we can apply Lemma~\ref{lemma: Uniform Bound Riccati} to show that the Riccati equation accepts a unique solution that lives inside the closed ball $\ballU = \{U : \|U - \tilde U_0\| \leq \radiusU\}$. This implies that for the selected range of the ratio $\beta \over \alpha $ we have : 
\[
\opnorm{U} \leq \Rbound
\]
}
We now apply the transformation to the linear recursion. Recall that
\[
    \bias{k}
    =
    \begin{pmatrix}
        \Expect{\tildex{k}}\\
        \Expect{\tildey{k}}
    \end{pmatrix}.
\]
Using $H=T^{-1}\Pi T$, the recursion~\eqref{eq:concatenated-bias-recursion} becomes
\begin{align*}
    \bias{k+1}
    &=
    \blockMatrix{I}{-U}{0}{I}
    \blockMatrix{
        I-\alpha J_{xx}-\beta UJ_{yx}
    }{0}{
        -\beta J_{yx}
    }{
        I-\beta(J_{yy}-J_{yx}U)
    }
    \blockMatrix{I}{U}{0}{I}
    \bias{k} +\mathcal R_k
    \\
\end{align*}
Define the transformed bias
\[
    \ricattibias{k}
    \triangleq
    T\bias{k}
    =
    \blockMatrix{I}{U}{0}{I}\bias{k}.
\]
Multiplying the bias recursion by $T$ gives
\begin{align}
    \ricattibias{k+1}
    &=
    \blockMatrix{
        I-\alpha \jxxab
    }{0}{
        -\beta J_{yx}
    }{
        I-\beta \jyyab
    }
    \ricattibias{k}
    +
    T \mathcal R_k,
    \label{eq:riccati-transformed-recursion}
\end{align}
where
\begin{align*}
    \jxxab
    &\triangleq
    J_{xx}
    +
    \frac{\beta}{\alpha}U \J{yx}
    \\
    \jyyab
    &\triangleq
    J_{yy}
    -
    J_{yx}J_{xx}^{-1}J_{xy}
    -
    \frac{\beta}{\alpha}J_{yx}\hat{U}_{\alpha,\beta}.
\end{align*}
In particular, the leading-order matrix for the slow variable's bias is
\[
    \Delta
    \triangleq
    J_{yy}
    -
    J_{yx}J_{xx}^{-1}J_{xy}.
\]
This matrix is the effective first-order drift of the slow dynamics. The assumption that $-\Delta$ is Hurwitz is standard
in two-time-scale stochastic approximation and appears, for instance, in the linear
and CLT analyses of TTSA
\cite{kwon2024twotimescalelinearstochasticapproximation, haque2025tightfinitetimebounds, hu2024centrallimittheoremtwotimescale, kaledin2020finite}.

For lighter notation, we write~\eqref{eq:riccati-transformed-recursion} as
\[
    r_{k+1}
    =
    \Pi r_k
    +
    T\mathcal R_k,
\]
where $r_k\triangleq \ricattibias{k}$, and
 $   \Pi
    \triangleq
    \blockMatrix{
        I-\alpha \jxxab
    }{0}{
        -\beta J_{yx}
    }{
        I-\beta \jyyab
    }$.

\paragraph{Block diagonalization through a Sylvester equation.}
After the Riccati transformation, the linear part of the recursion is block lower
triangular. The remaining coupling is the bottom left block $-\beta J_{yx}$, which still lets
the fast transformed bias influence the slow transformed bias. We now remove this
coupling by a second change of coordinates. Let
\begin{align}
    S
    \triangleq
    \blockMatrix{I}{0}{X}{I},
    \qquad
    S^{-1}
    =
    \blockMatrix{I}{0}{-X}{I},
    \label{eq:sylvester-transformation}
\end{align}
where $X\in\mathbb R^{d_y\times d_x}$ is to be determined. A direct computation gives
\begin{align*}
    S\Pi S^{-1}
    =
    \blockMatrix
    {I-\alpha \jxxab}
    {0}
    {X(I-\alpha \jxxab)-(I-\beta \jyyab)X-\beta J_{yx}}
    {I-\beta \jyyab}.
\end{align*}
Hence, the bottom left block vanishes if and only if
\begin{align*}
    X(I-\alpha \jxxab)
    -
    (I-\beta \jyyab)X
    -
    \beta J_{yx}
    =
    0.
\end{align*}
Equivalently, after simplifying and dividing by $\alpha$,
\begin{align}
    X\jxxab
    -
    \frac{\beta}{\alpha}\jyyab X
    =
    -\frac{\beta}{\alpha}J_{yx}.
    \label{eq:sylvester-equation-X}
\end{align}
This is a Sylvester equation for the unknown block $X$ \cite[Chapter~2, Eq~4.3]{kokotovic1999singular}, \cite[Lemma 8.3]{POZNYAK2008133}. From \cite[Lemma 8.3]{POZNYAK2008133}. 
\newcontent{Since we are assuming B.\ref{assump : derivability} and ${\beta \over \alpha } < \min\{ \epsStar{0}, \epsStar{1} , \epsStar{2}, \epsStar{3} \}$ that are dependent on $\radiusU$ and are defined in \eqref{eq: Riccati eps for the ball}, \eqref{eq: Riccati eps for the contraction}, \eqref{eq: eps condition sylvester ball}, \eqref{eq: eps condition sylvester contraction} then the above Sylvester equation accepts a unique solution $X$ that lives inside a closed ball $\ballX = \{X : \opnorm{X} \leq \eps \radiusX\}$ where $\radiusX = 2 \opnorm{\J{yx}} \opnorm{ \J{xx}^{-1}}$, this directly implies that:
\begin{align}
    \opnorm{X} \leq {\beta \over \alpha } \radiusX 
    \label{eq:alpha-X-bound}
\end{align}
}
Furthermore, the solution admits the expansion 
\begin{align*}
    X  = \sum_{\ell = 1}^\infty \parentheses{\beta \over \alpha }^\ell \tilde X_\ell, 
\end{align*}
We can solve for $\tilde X _1$ by replacing the sum into \eqref{eq:sylvester-equation-X} then dividing by $\beta \over \alpha$ then setting it to zero to get:
\begin{align*}
    \tilde X _1 = -J_{yx}J_{xx}^{-1}
\end{align*}
With this choice of $X$, the transformed matrix $ \Lambda
    \triangleq
    S\Pi S^{-1}$ is block diagonal:
\[
    \Lambda
    =
    \blockMatrix
    {I-\alpha \jxxab}
    {0}
    {0}
    {I-\beta \jyyab}.
\]

We now apply the second change of coordinates to the transformed bias vector
$r_k$. Define
\begin{align*}
    \sylvesterbias{k}
    \triangleq
    S\ricattibias{k}
    =
    \blockMatrix{I}{0}{X}{I}\ricattibias{k}.
\end{align*}
By the choice of $X$, the linear part of the recursion is block diagonal. Hence,
using \[S T = \blockMatrix{I}{U}{X}{XU+I},\] we obtain
\begin{align*}
    \sylvesterbias{k+1}
    &= \Lambda 
    \sylvesterbias{k}   -
    \blockMatrix{I}{U}{X}{XU+I}
    \begin{bmatrix}
        \alpha\bigl(
            \Expect{R_f(\tildex{k},\tildey{k})}
            +
            \Expect{\psi_k}
        \bigr)
        \\
        \beta\bigl(
            \Expect{R_g(\tildex{k},\tildey{k})}
            +
            \Expect{\varphi_k}
        \bigr)
    \end{bmatrix}.
\end{align*}

Hence, the concatenated transformed bias $s_k$ satisfies the compact recursion
\begin{align}
    s_{k+1}
    =
    \Lambda s_k
    +
    ST\mathcal R_k.
    \label{eq:compact-transformed-bias-recursion}
\end{align}


\paragraph{Recursion on the transformed system.}
We now roll out the transformed recursion~\eqref{eq:compact-transformed-bias-recursion}. 
Since the matrix $\Lambda$ is block diagonal, we obtain
\begin{align*}
    s_{k+1}
    =
    \Lambda^{k+1}s_0
    +
    \sum_{j=0}^k
    \Lambda^{k-j}ST\mathcal R_j .
\end{align*}
This representation is the main payoff of the two transformations: the linear part
of the recursion is now decoupled into a fast block and a slow block. Writing the two
components separately gives
\begin{align}
    s^x_{k+1}
    =&\,
    (I-\alpha \jxxab)^{k+1}s^x_0
    -\alpha
    \sum_{j=0}^k
    (I-\alpha \jxxab)^{k-j}
    \parentheses{
        \Expect{R_f(\tildex{j},\tildey{j})}
        +
        \Expect{\psi_j}
    }
    \nonumber\\
    &\,
    -\beta
    \sum_{j=0}^k
    (I-\alpha \jxxab)^{k-j}
    U
    \parentheses{
        \Expect{R_g(\tildex{j},\tildey{j})}
        +
        \Expect{\varphi_j}
    },
    \label{eq:sylvester-fast-bias}
    \\
    s^y_{k+1}
    =&\,
    (I-\beta \jyyab)^{k+1}s^y_0
    -\alpha
    \sum_{j=0}^k
    (I-\beta \jyyab)^{k-j}
    X
    \parentheses{
        \Expect{R_f(\tildex{j},\tildey{j})}
        +
        \Expect{\psi_j}
    }
    \nonumber\\
    &\,
    -\beta
    \sum_{j=0}^k
    (I-\beta \jyyab)^{k-j}
    (XU+I)
    \parentheses{
        \Expect{R_g(\tildex{j},\tildey{j})}
        +
        \Expect{\varphi_j}
    } .
    \label{eq:sylvester-slow-bias}
\end{align}

The original bias recursion was fully coupled: the fast and slow biases influenced each
other already at the level of the linear dynamics. After the changes of coordinates
defined by $T$ and $S$, this coupling has been pushed into the residual terms, while
the leading linear dynamics are block diagonal. This is precisely what allows us to
study the two components separately.

The next step is to show that the two diagonal blocks are contractions in suitable
Lyapunov norms. More precisely, for our choice on the step sizes, the matrices describing the perturbed dynamics:
\[
    I-\alpha \jxxab
    \qquad\text{and}\qquad
    I-\beta \jyyab
\]
inherit the stability of \newcontent{$I- \alpha J_{xx}$ and $I - \beta \Delta$}, respectively. Their powers therefore
produce the exponentially decaying initial-condition terms in the final bias bound,
with decay rates of order $\alpha$ for the fast component and $\beta$ for the slow
component.

\paragraph{Bound of the fast transformed component.}
We start from the recursion~\eqref{eq:sylvester-fast-bias}. For readability, write
\[
    A_x \triangleq I-\alpha J_{xx}^{\alpha,\beta}.
\]
Applying the triangle inequality, submultiplicativity of the operator norm, and
Jensen's inequality gives
\begin{align*}
\norm{s^x_{k+1}}
\leq\;&
    \norm{A_x^{k+1}s^x_0}
    +
    \alpha
    \sum_{j=0}^k
    \opnorm{A_x^{k-j}}
    \Expect{\norm{R_f(\tildex{j},\tildey{j})}}
\\
&\quad
    +
    \beta
    \sum_{j=0}^k
    \opnorm{A_x^{k-j}}
    \opnorm{U}
    \Expect{\norm{R_g(\tildex{j},\tildey{j})}}
\\
&\quad
    +
    \alpha
    \norm{
        \sum_{j=0}^k
        A_x^{k-j}
        \Expect{\psi_j}
    }
    +
    \beta
    \norm{
        \sum_{j=0}^k
        A_x^{k-j}
        U\Expect{\varphi_j}
    } .
\end{align*}
In the two terms involving $R_f$ and $R_g$, we used
$\norm{\Expect{Z}}\leq \Expect{\norm{Z}}$. By contrast, we keep the Markovian-noise
terms as weighted sums, since these will be controlled later through the
Poisson equation argument.

The Euclidean operator norm does not directly reveal the contraction of
$A_x=I-\alpha J_{xx}^{\alpha,\beta}$. We therefore use a Lyapunov norm adapted to
the unperturbed fast linear dynamics. Since $-J_{xx}$ is Hurwitz by
Assumption~B.\ref{assump : derivability}, there exists a unique symmetric positive
definite matrix $\LyapF$ solving
\begin{align}
    \LyapF J_{xx} + J_{xx}^\top \LyapF = I .
    \label{eq:lyapunov-equation-fast}
\end{align}
We denote the associated norm by
\[
    \norm{z}_{\LyapF}
    \triangleq
    \sqrt{z^\top \LyapF z}.
\]
\newcontent{Since we are assuming B.\ref{assump : derivability}, $\alpha < \alpha^*(\radiusU)$ as defined in \eqref{eq: alpha condition fast Lyapunov} and the ratio satisfies ${\beta \over \alpha} < \min\{\epsStar{0}, \epsStar{1},\epsStar{5}\}$ where $\epsStar{0}, \epsStar{1},\epsStar{5}$ depend on $\radiusU$ and are fully defined in \eqref{eq: Riccati eps for the ball}, \eqref{eq: Riccati eps for the contraction}, \eqref{eq: eps condition for fast lyapunov contraction} then after applying submultiplicativity of the norm we can apply Lemma~\ref{lemma: Lyapunov's fast contraction} to get:
we have:
\begin{align*}
    \norm{
        (I-\alpha J_{xx}-\beta UJ_{yx})^{k}
    }_{\LyapF}\leq \norm{
        I-\alpha J_{xx}-\beta UJ_{yx}
    }_{\LyapF}^k
    \leq
    e^{-kc\alpha}
\end{align*}
where $c = (8 \opnorm{\LyapF})^{-1}$, see Lemma~\ref{lemma: Lyapunov's fast contraction} and the proof for full derivation. Equivalently,
\[
    \norm{A_x^{k}
    }_{\LyapF}
    \leq
    e^{-kc\alpha}
\]}
which yields
for every $\ell\geq 0$,
\[
    \opnorm{A_x^\ell}
    \leq \scondL \LyapFNorm{A_x^\ell}\leq 
    \scondL e^{-c\alpha \ell}.
\]
Applying this bound to~\eqref{eq:sylvester-fast-bias} gives
\begin{align}
\norm{s^x_{k+1}}
\leq\;&
    \scondL e^{-(k+1)c\alpha}\norm{s^x_0}
    +
    \alpha\scondL
    \sum_{j=0}^k
    e^{-c\alpha(k-j)}
    \Expect{\norm{R_f(\tildex{j},\tildey{j})}}
    \label{eq:fast-bias-term-1}
\\
&\quad
    +
    \beta\scondL\opnorm{U}
    \sum_{j=0}^k
    e^{-c\alpha(k-j)}
    \Expect{\norm{R_g(\tildex{j},\tildey{j})}}
    \label{eq:fast-bias-term-2}
\\
&\quad
    +
    \alpha
    \norm{
        \sum_{j=0}^k
        A_x^{k-j}\Expect{\psi_j}
    }
    +
    \beta
    \norm{
        \sum_{j=0}^k
        A_x^{k-j}U\Expect{\varphi_j}
    } .
    \label{eq:fast-bias-term-3}
\end{align}
We now bound the terms on the right hand side separately. The first term in
\eqref{eq:fast-bias-term-1} is the exponentially decaying contribution of the
initial condition. The second term involves the Taylor remainder. Since
$R_f$ is a second order remainder, there exists a constant $L_R>0$ such that:
\begin{align*}
    \Expect{\norm{R_f(\tildex{j},\tildey{j})}}
    &\leq
    L_R
    \Expect{
        \bigl(\norm{\tildex{j}}+\norm{\tildey{j}}\bigr)^2
    }
    \\
    &\leq
    2L_R
    \Expect{
        \norm{\tildex{j}}^2+\norm{\tildey{j}}^2
    }
    \\
    &=
    2L_R
    \Expect{
        \norm{x_j-h(y_j)+h(y_j)-h(y^*)}^2
        +
        \norm{\tildey{j}}^2
    }
    \\
    &\leq
    2L_R
    \Expect{
        2\norm{x_j-h(y_j)}^2
        +
        2\norm{h(y_j)-h(y^*)}^2
        +
        \norm{\tildey{j}}^2
    }
    \\
    &\leq
    2L_R
    \Expect{
        2\norm{x_j-h(y_j)}^2
        +
        (2L_h^2+1)\norm{\tildey{j}}^2
    } .
\end{align*}
Setting:
\[
    d_1 \triangleq \max\{2,\,2L_h^2+1\},
\]
we obtain:
\[
    \Expect{\norm{R_f(\tildex{j},\tildey{j})}}
    \leq
    2L_Rd_1
    \Expect{
        \norm{x_j-h(y_j)}^2
        +
        \norm{\tildey{j}}^2
    } .
\]
Plugging this into the second term of~\eqref{eq:fast-bias-term-1} yields:
\begin{align*}
    \alpha\scondL
    \sum_{j=0}^k
    e^{-c\alpha(k-j)}
    \Expect{\norm{R_f(\tildex{j},\tildey{j})}}
    \leq
    2\alpha\scondL L_R d_1
    \sum_{j=0}^k
    e^{-c\alpha(k-j)}
    \Expect{
        \norm{x_j-h(y_j)}^2
        +
        \norm{\tildey{j}}^2
    } .
\end{align*}

By Lemma~\ref{lemma : Exp Sum Of the MSEs}, and using the notation
\[
    \mathcal M_{\alpha,\beta}
    \triangleq
    \mathrm{MSE}_{x}^{\infty}(\alpha,\beta)
    +
    \mathrm{MSE}_{y}^{\infty}(\alpha,\beta),
\]
we obtain
\begin{align*}
    \alpha \scondL
    \sum_{j=0}^k
    e^{-c\alpha(k-j)}
    \Expect{\norm{R_f(\tildex{j},\tildey{j})}}
    \leq\;&
    \alpha \underbrace{2\scondL L_R d_1}_{\triangleq C_{x,1}}
    \left(
        \mathcal M_{\alpha,\beta}
        +
        u_k^{c,\alpha}
        +
        w_k^{c,\alpha}
    \right)
    \nonumber\\
    &\quad
    +
    \underbrace{2\scondL L_R d_1 c^{-1}}_{ \triangleq C_{x,2}}
    \mathcal M_{\alpha,\beta}.
\end{align*}
The term involving $R_g$ is controlled in the same way. We denote $M_U \triangleq \Rbound$. Thus,
\begin{align}
    \eqref{eq:fast-bias-term-2}
    \leq\;&
    \beta \underbrace{2\scondL M_U L_R d_1}_{\triangleq C_{x,3}}
    \left(
        \mathcal M_{\alpha,\beta}
        +
        u_k^{c,\alpha}
        +
        w_k^{c,\alpha}
    \right)
    \nonumber +\frac{\beta}{\alpha}
    \underbrace{2\scondL M_U L_R d_1 c^{-1}}_{C_{x,4}}
    \mathcal M_{\alpha,\beta}.
\end{align}
It remains to control the two weighted Markovian-noise terms in
\eqref{eq:fast-bias-term-3}. By Lemma~\ref{lemma:bias-poisson},
\begin{align*}
    \eqref{eq:fast-bias-term-3}
    \leq\;&
    \alpha A_1
    +
    \beta(A_2+B_1)
    +
    \alpha^2 A_3
    +
    \beta^2 B_2
    +
    \alpha\beta(A_4+B_3)
    +
    \frac{\beta^2}{\alpha}B_4 .
\end{align*}
where the constants are defined in detail in the proof of the lemma. Merging all the bounds from above we obtain
\begin{align*}
    \norm{s^x_{k+1}}
    \leq\;&
    e^{-(k+1)c\alpha}\scondL\norm{s^x_0}
    +
    (\alpha + \beta )F_x^{(1)}
    +
    F_x^{(2)}\mathcal M_{\alpha,\beta}
    +
    (\alpha +\beta ) F_x^{(3)}  (u_k^{c,\alpha}+w_k^{c,\alpha}) \nonumber \\
    & + R^x_{\alpha, \beta} 
\end{align*}
where the quantities $F_x^{(1)},F_x^{(2)},F_x^{(3)}, R^x_{\alpha, \beta}$ are defined by:
\begin{align*}
    F_x^{(1)}&\coloneqq A_1 \vee  (A_2 + B_1) \vee \omega^* B_4, \\
    F_x^{(2)} &\coloneqq C_{x,2}+ \omega ^* C_{x,4} + F_x^{(3)}(\alpha^* +\beta^*) \;,\\
    F_x^{(3)}&\coloneqq C_{x,3} \vee C_{x,1} \;,  \\
    R^x_{\alpha, \beta}&\coloneqq \alpha^2 A_3
    +
    \beta^2 B_2
    +
    \alpha\beta(A_4+B_3)
    \;.
\end{align*}

\paragraph{Bound of the slow transformed component.}
We now turn to the slow component. Let
\[
    A_y \triangleq I-\beta J_{yy}^{\alpha,\beta},
    \qquad
    I' \triangleq I+XU .
\]
Applying the triangle inequality, submultiplicativity, and Jensen's inequality to
\eqref{eq:sylvester-slow-bias} gives
\begin{align*}
    \norm{s^y_{k+1}}
    \leq\;&
    \norm{A_y^{k+1}s^y_0}
    +
    \alpha
    \sum_{j=0}^k
    \opnorm{A_y^{k-j}}
    \opnorm{X}
    \Expect{\norm{R_f(\tildex{j},\tildey{j})}}
    \\
    &\quad
    +
    \beta
    \sum_{j=0}^k
    \opnorm{A_y^{k-j}}
    \opnorm{I'}
    \Expect{\norm{R_g(\tildex{j},\tildey{j})}}
    \\
    &\quad
    +
    \alpha
    \norm{
        \sum_{j=0}^k
        A_y^{k-j}X\Expect{\psi_j}
    }
    +
    \beta
    \norm{
        \sum_{j=0}^k
        A_y^{k-j}I'\Expect{\varphi_j}
    } .
\end{align*}

We use a Lyapunov norm adapted to the effective slow matrix
\[
    \Delta
    \triangleq
    J_{yy}-J_{yx}J_{xx}^{-1}J_{xy}.
\]
Since $-\Delta$ is Hurwitz by Assumption~B.\ref{assump : derivability}, there exists
a unique symmetric positive definite matrix $\LyapG$ solving the following Lyapunov equation: 
\begin{align}
    \LyapG\Delta+\Delta^\top\LyapG=I .
    \label{eq:Lyapunov-equation-slow}
\end{align}
\newcontent{
Since we are assuming B.\ref{assump : derivability}, and we have $\beta \leq \beta^*$ as defined in \eqref{eq: condition beta contraction} that also satisfies ${\beta \over \alpha} < \min \{\epsStar{0}, \epsStar{1}\}$ then we can apply  Lemma~\ref{lemma: Lyapunov's slow contraction} to have:
\begin{align*}
\LyapGNorm{I - \beta (\Delta-  \J{yx}(U - U_0))}  \leq e^{-d\beta }
\end{align*}
where $d = (8 \opnorm{\LyapG})^{-1}$, see Lemma~\ref{lemma: Lyapunov's slow contraction} and its proof for full derivation. Equivalently,
\begin{align*}
    \norm{A_y^{\ell}}_{\LyapG}
    \leq
    e^{-d\beta \ell},
    \qquad \forall \ell\geq 0.
\end{align*}
}
we obtain, for every $\ell\geq 0$,
\[
    \opnorm{A_y^\ell}
    \leq \scondG \LyapGNorm{A_y^\ell}\leq 
    \scondG e^{-d\beta \ell}.
\]
Therefore,
\begin{align}
    \norm{s^y_{k+1}}
    \leq\;&
    e^{-d\beta(k+1)}\scondG\norm{s^y_0}
     \nonumber\\
     & +
    \alpha\opnorm{X}\scondG
    \sum_{j=0}^k
    e^{-d\beta(k-j)}
    \Expect{\norm{R_f(\tildex{j},\tildey{j})}}
    \label{eq:slow-bias-term-1}
    \\
    &\quad
    +
    \beta\opnorm{I'}\scondG
    \sum_{j=0}^k
    e^{-d\beta(k-j)}
    \Expect{\norm{R_g(\tildex{j},\tildey{j})}}
    \label{eq:slow-bias-term-2}
    \\
    &\quad
    +
    \alpha
    \norm{
        \sum_{j=0}^k
        A_y^{k-j}X\Expect{\psi_j}
    }
    +
    \beta
    \norm{
        \sum_{j=0}^k
        A_y^{k-j}I'\Expect{\varphi_j}
    } .
    \label{eq:slow-bias-term-3}
\end{align}

We now bound the Taylor remainder terms. As in the fast component analysis, there
exists a constant $L_R>0$ such that
\[
    \Expect{\norm{R_f(\tildex{j},\tildey{j})}}
    \leq
    2L_R d_1
    \Expect{
        \norm{x_j-h(y_j)}^2+\norm{\tildey{j}}^2
    },
    \qquad
    d_1\triangleq \max\{2,2L_h^2+1\}.
\]
The same bound holds for $R_g$, up to increasing $L_R$ if necessary.

By Lemma~\ref{lemma : Exp Sum Of the MSEs},
\begin{align*}
    \sum_{j=0}^k
    e^{-d\beta(k-j)}
    \Expect{\norm{R_f(\tildex{j},\tildey{j})}}
    \leq
    2L_Rd_1
    \left[
        \left(1+\frac{1}{d\beta}\right)\mathcal M_{\alpha,\beta}
        +
        u_k^{d,\beta}
        +
        w_k^{d,\beta}
    \right].
\end{align*}
At first sight, the factor $1/\beta$ is problematic, \newcontent{however, the term \eqref{eq:slow-bias-term-1} still has a multiplication by $\alpha \opnorm{X}$. Since we are assuming B.\ref{assump : derivability}, we showed previously in \eqref{eq:alpha-X-bound} that \[
\alpha \opnorm{X} \leq \beta \radiusX
\]
with $\radiusX =  2 \opnorm{\J{yx}} \opnorm{ \J{xx}^{-1}}$.} Consequently,
\begin{align*}
    \eqref{eq:slow-bias-term-1}
    \leq\;&
    \beta \underbrace{ 2 \radiusX\scondG L_Rd_1}_{\triangleq C_{y,1}}
    \left(
        \mathcal M_{\alpha,\beta}
        +
        u_k^{d,\beta}
        +
        w_k^{d,\beta}
    \right)
    \nonumber\\
    &\quad
    +
    \underbrace{2\radiusX \scondG L_Rd_1d^{-1}}_{\triangleq C_{y,2}}
    \mathcal M_{\alpha,\beta}.
\end{align*}
The term \eqref{eq:slow-bias-term-2} is controlled similarly. 
\newcontent{For ${\beta \over \alpha }< \min\{\epsStar{0},\epsStar{1}, \epsStar{2}, \epsStar{3}\}$  we have \[
\opnorm{I'} \leq \opnorm{I} + \opnorm{X}\opnorm{U}\leq 1 + \radiusX(\Rbound) \triangleq M_{I'}
\]}
Lemma~\ref{lemma : Exp Sum Of the MSEs}
gives
\begin{align*}
    \eqref{eq:slow-bias-term-2}
    \leq\;&
    \beta \underbrace{2 M_{I'}\scondG L_Rd_1}_{\triangleq C_{y,3}}
    \left(
        \mathcal M_{\alpha,\beta}
        +
        u_k^{d,\beta}
        +
        w_k^{d,\beta}
    \right)
    \nonumber\\
    &\quad
    +
    \underbrace{2M_{I'}\scondG L_Rd_1d^{-1}}_{\triangleq C_{y,4}}
    \mathcal M_{\alpha,\beta}.
\end{align*}

It remains to bound the weighted Markovian-noise terms in
\eqref{eq:slow-bias-term-3}. By Lemma~\ref{lemma:bias-poisson},
\begin{align*}
    \eqref{eq:slow-bias-term-3}
    \leq
    \alpha(C_2+D_4)
    +
    \beta(C_1+D_1)
    +
    \alpha\beta(C_4+D_3)
    +
    \beta^2(C_3+D_2).
\end{align*}
We define : 
\begin{align*}
    F_y^{(1)}&\coloneqq (C_2 + D_4) \vee  (C_1 + D_1) \;, \\
    F_y^{(2)} &\coloneqq C_{y,2} + C_{y,4} + \beta^* (C_{y,1} + C_{y,3})  \;, \\
    F_y^{(3)}&\coloneqq C_{y,3} + C_{y,1}  \;,  \\
    R^y_{\alpha, \beta}&\coloneqq \alpha\beta(C_4+D_3)
    +
    \beta^2(C_3+D_2)\;. 
\end{align*}
We wrap all the bounds in the final one
\begin{align*}
    \norm{s^y_{k+1}}
    \leq\;&
    e^{-d\beta(k+1)}\scondG\norm{s^y_0}
    +
    (\alpha+\beta) F_y^{(1)}
    +
    F_y^{(2)}\mathcal M_{\alpha,\beta}
    +
    \beta F_y^{(3)}(u^{d,\beta}_k+ w^{d,\beta}_k)\nonumber  \\
    & + R^y_{\alpha, \beta}
\end{align*}
we merge the bounds of $\norm{s^x_{k}}$ and $\norm{s^y_{k}}$ to get 
\begin{align*}
    \norm{s^x_{k+1}} + \norm{s^y_{k+1}} \leq & e^{-(k+1)c\alpha}\scondL\norm{s^x_0}
    +  e^{-d\beta(k+1)}\scondG\norm{s^y_0}  \\
    & + \underbrace{F_x^{(1)}+  F_y^{(1)}}_{\triangleq F_1}(\alpha + \beta )
    + \underbrace{(F_x^{(2)} + F_y^{(2)})}_{\triangleq F_2}\mathcal M_{\alpha,\beta}
    +
    (\alpha +\beta ) F_x^{(3)}  (u_k^{c,\alpha}+w_k^{c,\alpha}) \nonumber \\
    &
    +
    \beta F_y^{(3)}(u^{d,\beta}_k+ w^{d,\beta}_k)\nonumber + \underbrace{R^x_{\alpha, \beta} + R^y_{\alpha, \beta}}_{\triangleq R_{\alpha, \beta}}
\end{align*}
We define the new sequence $\mathrm{\hat S^{MSE}}_k$ to be: 
\begin{align*}
    \mathrm{\hat S^{MSE}}_k = \beta F_y^{(3)}(u^{d,\beta}_k+ w^{d,\beta}_k) + (\alpha +\beta ) F_x^{(3)}  (u_k^{c,\alpha}+w_k^{c,\alpha})
\end{align*}
where we have showed that by Lemma~\ref{lemma : Exp Sum Of the MSEs} that the limit exist and it is zero 
\begin{align*}
    \lim_{k\to \infty }\mathrm{\hat S^{MSE}}_k = 0
\end{align*}
which yields our bound: 
\begin{align*}
    \norm{s^x_{k+1}} + \norm{s^y_{k+1}} \leq & e^{-(k+1)c\alpha}\scondL\norm{s^x_0}
    +  e^{-d\beta(k+1)}\scondG\norm{s^y_0}  \\
    & + F_1(\alpha + \beta )
    + F_2\mathcal M_{\alpha,\beta}
    + \mathrm{\hat S^{MSE}}_k + R_{\alpha, \beta}
\end{align*}
where $R_{\alpha, \beta}$ contains terms of higher order $O((\alpha + \beta)^2)$. Moreover, from the definition of $u_k^{c,\alpha}, w_k^{c,\alpha}$ in Lemma~\ref{lemma : Exp Sum Of the MSEs} there exists $a,b, \hat c, \hat d$ such that:
\begin{align}
    \mathrm{\hat S^{MSE}}_k\leq a e^{-\alpha \hat c k }\norm{\x{0}}^2 + b  e^{-\beta \hat d k } \norm{\y{0}}^2\label{eq:sequence-sum-mse}
\end{align}
So our final bound will be
\begin{align*}
    \norm{s^x_{k+1}} + \norm{s^y_{k+1}} \leq & e^{-(k+1)c\alpha}\scondL\norm{s^x_0}
    +  e^{-d\beta(k+1)}\scondG\norm{s^y_0} + a e^{-\alpha \hat c k }\norm{\x{0}}^2 + b  e^{-\beta \hat d k } \norm{\y{0}}^2  \\
    & + F_1(\alpha + \beta )
    + F_2\mathcal M_{\alpha,\beta}
    + R_{\alpha, \beta}
\end{align*}

\paragraph{Unwrapping the results.} It remains to return to the original bias coordinates. By construction,
\[
    s_k
    =
    \sylvesterbias{k}
    =
    S T \bias{k}.
\]
Since the matrices $S$ and $T$ are invertible, we can write equivalently,
\[
    \bias{k}
    =
    T^{-1}S^{-1}s_k .
\]
Therefore, by applying the norm on both sides,
\begin{align*}
    \norm{\bias{k}}
    &\leq
    \norm{T^{-1}S^{-1}s_k}
    \\
    &\leq
    \opnorm{T^{-1}S^{-1}}\norm{s_k}
    \\
    &\leq
    \opnorm{T^{-1}S^{-1}}
    \left(
        \norm{s_k^x}+\norm{s_k^y}
    \right).
\end{align*}
where in the last inequality we have used the triangular inequality since 
\begin{align*}
    \norm{\sylvesterbias{k}}  = \norm{ \vector{s_k^x}{0} +  \vector{0}{s_k^y }} \leq \norm{s_k^x}+ \norm{s_k^y}
\end{align*}
Plugging the bounds obtained above for the fast and slow
transformed components gives the desired bias bound:
\[
    \max\left\{
        \norm{b_k^x},
        \norm{{b_k^y}}
    \right\}
    \leq
    \opnorm{T^{-1}S^{-1}}
    \left(
        \norm{s_k^x}+\norm{s_k^y}
    \right).
\]
which concludes the proof of our theorem.
\end{proof}
\section{Supporting Lemmas}\label{section: appendix lemmas}
The first lemma is merely a collection of various Lipschitzness/boundedness properties induced by the assumptions we make in the paper, introducing several constants that we use in the proofs of the main theorems.
\begin{lemma}\label{lemma: bounds}
Under assumptions A.\ref{assumption: Boundedness}-\ref{assumption: Lip and Mono of g}, there exist finite non-negative constants $\boundF, \boundG, \boundH, \boundXi,\boundXY$ dependent on the problem settings satisfying :
\begin{enumerate}
    \item $\|\barF{k}\| \leq 2\Lf\boundXY+\|\overline f(0,0)\|\leq \boundF$
    \item $\|\barG{k}\| \leq 2\Lg\boundXY+ \|\overline g(0,0)\|\leq \boundG$
    \item $\|\F{k}\| \leq 2\Lf\boundXY+\boundXY \leq \boundF$
    \item $\|\G{k}\| \leq  2\Lg\boundXY+ \boundXY \leq \boundG$
    \item $\|\Dh{k}\| \leq \beta \Lh\boundG \leq \beta \boundH$
    \item $\|\psi_k\| \leq \boundXi$ 
    \item $\|\H{k}\| \leq \Lh \boundXY + \norm{\H{0}}\leq \boundH$
    \item Given assumption A.\ref{assumption: All Kernel Properties}-\ref{assumption: Lip and Mono of g}, for every $x_1, x_2 \in \Rx , y_1, y_2 \in \Ry$, $v_f$ and $v_g$ (as defined in \eqref{eq: poisson solution fast} and \eqref{eq: poisson solution slow}) are $L_{v_f}$ and $L_{v_g}$-Lipschitz continuous:
    \begin{align*}
        \| v_f(x_1,y_1,\xi)- v_f(x_2,y_2,\xi)\|\leq \Lvf  (\|x_1-x_2\|+\|y_1 -y_2\|)\\
        \| v_g(x_1,y_1,\xi)- v_g(x_2,y_2,\xi)\|\leq \Lvg  (\|x_1-x_2\|+\|y_1 -y_2\|)
    \end{align*}
    \item $\|v_f(x_1,y_1,\xi)\|\leq  2\Lf\boundXY + f_0 \leq \boundXi $
    \item $\inner{\x{k}}{v_f(x_k,y_k,\xi_{k+1}) }\leq  \boundXi$
\end{enumerate}
\end{lemma}
\begin{proof}
We provide one line proof for each point:
\begin{enumerate}
    \item $\|\barF{k}\| \leq \|\barF{k} - \overline f(0,0)\| + \|\overline f(0,0)\|\leq \Lf (\|x_k\|+\|y_k\|) + \|\overline f(0,0)\| \leq 2\Lf\boundXY+\|\overline f(0,0)\|\leq \boundF$
    \item $\|\barG{k}\| \leq 2\Lg\boundXY+ \|\overline g(0,0)\|\leq \boundG$
    \item $\|\F{k}\| \leq 2\Lf\boundXY+\boundXY \leq \boundF$
    \item $\|\G{k}\| \leq  2\Lg\boundXY+ \boundXY \leq \boundG$
    \item $\|\Dh{k}\|=\|\H{k+1}-\H{k})\|\leq \Lh \|\beta \G{k}\|\leq \beta  \boundH$
    \item $\|\psi_k\| = \|\F{k} - \barF{k}\| \leq \|\F{k}\| + \|\barF{k}\| \leq 2 \boundF \leq  \boundXi$ 
    \item $\|\H{k}\| \leq \Lh \boundXY + \norm{\H{0}} \leq \boundH$
    \item See \cite[Lemma 8]{allmeier2024computingbiasconstantstepstochastic} and \cite[Lemma D.4]{haque2025tightfinitetimebounds}.
    \item $\|v_f(x_1,y_1,\xi)\|\leq  2\Lf\boundXY + f_0 \leq  \boundXi $
    \item $\inner{\x{k}}{ \vf{k}}\leq \norm{\x{k}} \norm{\vf{k}} \leq \boundXi$
\end{enumerate}
\end{proof}
The next lemma proves bounds on the terms containing the weighted sum of deviations of the dynamics estimators from their stationary averages $\F{k} - \barF{k}$ and $\G{k}-\barG{k}$ in Theorem~\ref{thm : mse}.

\begin{lemma} 
    \label{lemma : poisson mse}
Taking the sequence $(x_k,y_k)_{k\geq 0}$ defined by \eqref{eq: TTSA update} and under assumptions A.\ref{assumption: Boundedness}-\ref{assumption: Lip and Mono of f}, for every $\alpha  \muf <1 $ and $\beta \mug <1$:
    \begin{align*} 
        \boundMSEFastPoisson  \\
        \boundMSESlowPoisson
    \end{align*}
    where $\psi_k = \F{k}  - \barF{k}$ and $\newcontent{\phi_k = \G{k} - \barG{k}}$.
\end{lemma}

\begin{proof} 
To prove this lemma, we use the same technique in \cite[Lemma 8]{allmeier2024computingbiasconstantstepstochastic} and \cite{kaledin2020finite} by invoking the solution of a Poisson equation that allow us to decompose the deviation terms $\F{j}-\barF{j}$ inside the sum into a Martingale Difference Sequence (MDS), a bounded term and a telescoping term. Since our terms in the sum are weighted by $(1-\alpha\muf)^{k-j}$ we don't have exactly the telescoping terms that we want, instead, we have additional residuals that we needed to deal with. \newcontent{Moreover, since our transition kernel $K$ is depending on the iterates, we will have a drift between the Poisson solution time step and the required term to build the Martingale Difference}. In this proof we first start by introducing the core idea of the Poisson equation, then we derive each component individually to get their respective bounds. Moreover, throught this lemma we bound images of the functions $g$ and $\bar g$ by continuity and compactness and not using A.~\ref{assumption: Lip and Mono of g} and we use the same upper bound constant $\boundG$. 
Under assumption A.\ref{assumption: All Kernel Properties}, by \cite[Definition 21.2.1]{MC_2018_moulines} and \cite[Lemma 8]{allmeier2024computingbiasconstantstepstochastic} $v_f$ is a solution of the following Poisson equation:

\begin{align}
    \newcontent{\vf{k} - \sum_{\xi'\in \Xi}\Kernel{k}(\xi_{k+1}, \xi' ) v_f(x_k,y_k,\xi') = \F{k} - \barF{k} \label{eq: poisson solution fast}}
\end{align}
\newcontent{In the following, we use the notation, with $K_k \triangleq  \Kernel{k}$ \[\expvf{k}{k+1}{k} \triangleq \sum_{\xi'\in \Xi}\Kernel{k}(\xi_{k+1}, \xi' )  v_f(x_k,y_k,\xi').\] }
It is shown in \cite[Lemma 8]{allmeier2024computingbiasconstantstepstochastic} that $v_f$ is Lipschitz-continuous. Our goal from introducing this solution is to be able to write this decomposition: 
\begin{align*}
    \sum_{j=0}^{k} (1-\alpha\muf)^{k-j}\inner{\x{j}}{\psi_j} =\sum_{j=0}^{k} \text{MDS + Telescoping Term+ Bounded Term + Kernel Drift}
\end{align*}
First, we start by substituting \eqref{eq: poisson solution fast} in the sum:
\begin{align*}
 \sum_{j=0}^{k} (1-\alpha\muf)^{k-j}\inner{\x{j}}{\psi_j} =\sum_{j=0}^{k} (1-\alpha\muf)^{k-j}\inner{\x{j}}{ \vf{j} - \newcontent{\expvf{j}{j+1}{j}]}}
\end{align*}
\newcontent{We notice that the term $\expvf{j}{j+1}{j}$ does not make a Martingale Difference with the term $\vf{j}$ since $\xi'$ is sampled at the step $j+2$, we add and subtract $\poissf{j}{j+2}$:
\begin{align*}
\poissf{j}{j+1} - \expvf{j}{j+1}{j} =& \poissf{j}{j+2} - \expvf{j}{j+1}{j} \\
& + \poissf{j}{j+1} - \poissf{j}{j+2}
\end{align*}
However, this is not enough since $\xi_{j+2}\sim \Kernel{j+1}(\xi_{j+1}, \cdot)$ and not $\Kernel{j}(\xi_{j+1}, \cdot)$. To obtain a Martingale Difference we add and subtract $\expvf{j+1}{j+1}{j}$. Finally, to use Lipschitzness of $v_f$ and get a telescoping term we add and subtract $\poissf{j+1}{j+2}$. This will give us the following decomposition of the Markovian noise deviation}
\begin{align}
    \inner{\x{j}}{\psi_j} = \Big{\langle} \x{j},&v_f(x_j,y_j,\xi_{j+2}) - \newcontent{\expvf{j+1}{j+1}{j}}\label{eq: poisson fast 1}\\
    & + v_f(x_{j+1},y_{j+1},\xi_{j+2}) - v_f(x_{j},y_{j},\xi_{j+2})\label{eq: poisson fast 2}\\  
    & +\vf{j} - v_f(x_{j+1},y_{j+1},\xi_{j+2}) \label{eq: poisson fast 3}\\
    & + \newcontent{\expvf{j+1}{j+1}{j}-\expvf{j}{j+1}{j}} 
    \Big{\rangle}\label{eq: poisson fast 4}
\end{align}
the first term \eqref{eq: poisson fast 1} is an MDS, so it reduces to zero under expectation: 
\begin{align*}
    \Expect{ \sum_{j=0}^{k} (1-\alpha\muf)^{k-j}\inner{\x{j}}{ v_f(x_j,y_j,\xi_{j+2}) - \newcontent{\expvf{j+1}{j+1}{j}}}} = 0
\end{align*}
The second term \eqref{eq: poisson fast 2} scales linearly with the sum of step-sizes, since $f$ and $K$ are Lipschitz, so $v_f$ will be $\Lvf$-Lipschitz too for some positive constant $\Lvf$ and this is due to Lemma \ref{lemma: bounds}. So we can bound \eqref{eq: poisson fast 2} by:
\begin{align*}
     \abs{\eqref{eq: poisson fast 2}}
     \leq &  \norm{\x{j}} \norm { v_f(x_{j+1},y_{j+1},\xi_{j+2}) - v_f(x_{j},y_{j},\xi_{j+2})}\\
     = & \norm {x_j - h(y_j)} \norm { v_f(x_{j+1},y_{j+1},\xi_{j+2}) - v_f(x_{j},y_{j},\xi_{j+2})}\\
     \leq & \Lvf \norm {x_j - h(y_j)} (\norm{x_{j+1} -x_j} + \norm {y_{j+1} - y_j})\\
     \leq & \Lvf (\norm {x_j} +\norm{h(y_j)} )(\norm{x_{j+1} -x_j} + \norm {y_{j+1} - y_j})\\
     \leq & \Lvf (\boundXY +  \boundH) ( \norm{\alpha \F{j} } + \norm{\beta \G{j} })\\
     \leq & \Lvf (\boundXY +  \boundH) ( \alpha \boundF +  \beta \boundG )
\end{align*}
For brevity we set $z_j = \inner{\x{j}}{ v_f(x_{j+1},y_{j+1},\xi_{j+2}) - v_f(x_{j},y_{j},\xi_{j+2})}$. Now we apply our bound on the sum to have 
\begin{align*}
    \abs{\alpha\sum_{j=0}^{k} (1-\alpha\muf)^{k-j}z_j} \leq & \alpha\sum_{j=0}^{k} (1-\alpha\muf)^{k-j}\abs{z_j} \\  \leq & \alpha\sum_{j=0}^{k}(1-\alpha\muf)^{k-j}\Lvf (\boundXY +  \boundH) ( \alpha \boundF +  \beta \boundG )\\
    = & \alpha \Lvf (\boundXY +  \boundH) ( \alpha \boundF +  \beta \boundG ) \parentheses{1 - (1-\alpha \muf)^{k+1} \over \muf \alpha }\\
    \leq &  {\Lvf (\boundXY +  \boundH)\over \muf} ( \alpha \boundF +  \beta \boundG )
\end{align*}
\newcontent{ To see why the term \eqref{eq: poisson fast 4} is bounded we use our assumptions  A~\ref{assumption: All Kernel Properties} and A~\ref{assumption: Lip and Mono of f} about Lipschitzness of $K$ and $f$ 
\begin{align*}
    \eqref{eq: poisson fast 4} &= \norm {\sum_{\xi' \in \Xi} (K_{j+1}(\xi_{j+1}, \xi') - K_{j}(\xi_{j+1}, \xi'))v_f(x_j,y_j,\xi')}\\
    & \leq \sum_{\xi' \in \Xi} \norm {K_{j+1}(\xi_{j+1},\xi') - K_{j}(\xi_{j+1}, \xi')}\norm { v_f(x_j,y_j,\xi')}\\
    & \leq \sum_{\xi' \in \Xi} \Lk \parentheses{\norm{x_{j+1} - x_j}+ \norm{y_{j+1} - y_j}} \boundXi\\
    & \leq \Lk \boundXi(\alpha \boundF + \beta \boundG) |\Xi|
\end{align*}
which gives the bound 
\begin{align*}
    \alpha \abs{ \sum_{j=0}^k(1-\muf \alpha)^{k-j}\inner{\x{j}}{\expvf{j+1}{j+1}{j}-\expvf{j}{j+1}{j}}} \\ \leq {|\Xi| \boundXY \Lk \boundXi(\alpha \boundF + \beta \boundG)  \over \muf}
\end{align*}
}
Finally, we handle the third term \eqref{eq: poisson fast 3} which is supposed to be a telescoping term inside the sum:
\begin{align}
      \alpha \sum_{j=0}^{k} (1-\alpha\muf)^{k-j} \inner {\x{j}}{ v_f(x_j,y_j,\xi_{j+1}) - v_f(x_{j+1},y_{j+1},\xi_{j+2})} \label{eq: main poisson sum for x}
\end{align}
This is not perfectly telescoping. This hurdle emerges from two quantities: first, in the inner product itself we have $\x{j}$ which depends on $j$ and second, the weights $(1-\alpha \muf)^{k-j}$ also depend on $j$. To get over this, we will proceed in two steps. First, we fix the telescoping term inside the inner product by using the update rule in \eqref{eq: TTSA update}. Due to this step, we will have a remainder term dependent on the dynamics of $x_j$ and controlled by $\alpha$. The second step, is to handle the weights to have our telescoping terms. These two steps, will introduce remainders that we will need to control. Now we detail every step, so we start with the inner product and we set $v_j = \vf{j}$ so we can easily write the computations:
\begin{align*}
    \inner{\x{j}}{ v_{j+1} - v_{j}} 
    & = \inner{x_j- \H{j}}{v_{j+1}} - \inner{\x{j}}{v_{j}}\\
    &= \inner{x_{j+1} + \alpha \F{j} - \H{j}}{v_{j+1}} -\inner{\x{j}}{v_{j}}\\
    & = \inner{\x{j+1} + \alpha \F{j} + \Dh{j}}{v_{j+1}} -\inner{\x{j}}{v_{j}}\\
    & = \inner{\x{j+1}}{v_{j+1}}  - \inner{\x{j}}{v_{j}} + \inner{\alpha \F{j} +\Dh{j}} {v_{j+1}}
\end{align*}
where $\Dh{j}= \H{j+1}-\H{j}$. Now, we have a telescoping term plus a remainder. Substituting the last line into \eqref{eq: main poisson sum for x}: 
\begin{align}
    \eqref{eq: main poisson sum for x} &=  \alpha \sum_{j=0}^{k}  (1-\alpha\muf)^{k-j}( \inner{\x{j+1}}{v_{j+1}}  - \inner{\x{j}}{v_{j}}  ) \label{eq: almost telescoping}\\ 
     &+ \alpha^2 \sum_{j=0}^{k}  (1-\alpha\muf)^{k-j}  \inner{\F{j}}{v_{j+1}} \nonumber\\
     &+ \alpha \sum_{j=0}^{k}  (1-\alpha\muf)^{k-j} \inner{\Dh{j}}{ v_{j+1}} \nonumber
\end{align}
Starting with the first term \eqref{eq: almost telescoping}. Because of the weights in the sum $(1-\alpha \muf)^{k-j}$ we can't directly have a telescoping sum. An extra step has to be done for this term to be able to make it telescope with the cost of introducing new remainders. Hopefully, these remainders can be controlled. We proceed by:
\begin{align*}
    \eqref{eq: almost telescoping}=& \alpha \sum_{j=0}^{k}  (1-\alpha\muf)^{k-j} \inner{\x{j+1}}{v_{j+1}}  - \alpha \sum_{j=0}^{k}  (1-\alpha\muf)^{k-j} \inner{\x{j}}{v_{j}}    \\
     =& \alpha (1-\alpha\muf)\sum_{j=0}^{k}  (1-\alpha\muf)^{k-(j+1)} \inner{\x{j+1}}{v_{j+1}}  -\alpha  \sum_{j=0}^{k}  (1-\alpha\muf)^{k-j} \inner{\x{j}}{v_{j}} \\ 
    =& \alpha \sum_{j=0}^{k}  (1-\alpha\muf)^{k-(j+1)} \inner{\x{j+1}}{v_{j+1}}  - \alpha \sum_{j=0}^{k}  (1-\alpha\muf)^{k-j} \inner{\x{j}}{v_{j}} \\
    & -\alpha^2\muf\sum_{j=0}^{k}  (1-\alpha\muf)^{k-(j+1)} \inner{\x{j+1}}{v_{j+1}}\\ 
    = & \alpha (1-\alpha\muf)^{-1}\inner{\x{k+1}}{v_{k+1}} - \alpha (1-\alpha\muf)^k \inner{\x{0}}{v_{0}}  \\
    &-\alpha^2\muf\sum_{j=0}^{k}  (1-\alpha\muf)^{k-(j+1)} \inner{\x{j+1}}{v_{j+1}}
\end{align*}
where in the second equality we factored out $1-\alpha \muf$ to make the inside weight match the weight in the other sum. In the last line, we removed the telescoping terms to end up with the final three terms. Now, into \eqref{eq: main poisson sum for x} we apply absolute value to get: 
\begin{align}
   \abs{\eqref{eq: main poisson sum for x}} \leq & \abs{\alpha \Big((1-\alpha\muf)^{-1}\inner{\x{k+1}}{v_{k+1}}- (1-\alpha\muf)^k \inner{\x{0}}{v_{0}} \Big) \label{eq: fast poisson telescoping remainder 1}} \\
   &+ \alpha \sum_{j=0}^{k}  (1-\alpha\muf)^{k-j}\abs{ \inner{\Dh{j}}{ v_{j+1}} }\label{eq: fast poisson telescoping remainder 2}\\
    &  +\alpha^2\muf\sum_{j=0}^{k}  (1-\alpha\muf)^{k-(j+1)} \abs{ \inner{\x{j+1}}{v_{j+1}}} \label{eq: fast poisson telescoping remainder 3}\\ 
    &+ \alpha^2 \sum_{j=0}^{k}  (1-\alpha\muf)^{k-j} \abs{ \inner{\F{j}}{v_{j+1}}} \label{eq: fast poisson telescoping remainder 4}
\end{align}
We bound each term independently. Starting with the first term \eqref{eq: fast poisson telescoping remainder 1}: 
\begin{align*}
    \eqref{eq: fast poisson telescoping remainder 1} \leq \alpha{\boundXi \over 1 - \alpha \muf} + \alpha {\boundXi  } = \alpha \boundXi \parentheses{ 2 -\alpha \muf \over 1 - \alpha \muf}
\end{align*}
the second term \eqref{eq: fast poisson telescoping remainder 2}
\begin{align*}
   \eqref{eq: fast poisson telescoping remainder 2} &\leq \alpha\sum_{j=0}^{k}  (1-\alpha\muf)^{k-j} \|\Dh{j}\|\|v_{j+1}\| \leq  \alpha  \beta \boundXi \boundH \sum_{j=0}^{k}  (1-\alpha\muf)^{j} \\
     & = \alpha \beta \boundXi \boundH {1 - (1-\alpha\muf)^{k+1} \over \alpha \muf} \leq \beta {  \boundXi \boundH \over \muf} 
\end{align*}
the third term \eqref{eq: fast poisson telescoping remainder 3}
\begin{align*}
   \eqref{eq: fast poisson telescoping remainder 3} &\leq   \alpha ^2 \muf \boundXi \sum_{j=0}^{k} (1-\alpha\muf)^{k-(j+1)}  =  \alpha ^2 \muf \boundXi (1-\alpha\muf)^{-1}\sum_{j=0}^{k} (1-\alpha\muf)^{k-j} \\
    & \leq \alpha {\boundXi \over 1 - \alpha \muf}
\end{align*}
and the last term \eqref{eq: fast poisson telescoping remainder 4}
\begin{align*}
     \eqref{eq: fast poisson telescoping remainder 4}&\leq \alpha^2 \boundXi \boundF \sum_{j=0}^{k}  (1-\alpha\muf)^{j} \leq  \alpha { \boundXi \boundF \over \muf} 
\end{align*}
After bounding each term individually, using bounds from Lemma~\ref{lemma: bounds}, we apply these bounds on the main sum \eqref{eq: main poisson sum for x}
\begin{align*}
    \abs{\eqref{eq: main poisson sum for x} }\leq  \alpha \boundXi \parentheses{ {3 - \alpha \muf \over 1 - \alpha \muf} + {\boundF\over \muf}}   + \beta {  \boundXi \boundH \over \muf} 
\end{align*}
Since expectation preserves order, we apply it on both sides to get the final result:
\begin{align*}
\boundMSEFastPoisson
\end{align*}
which concludes the proof of the first part.\\
As above, let $v_g$ be the solution of the Poisson equation:
    \begin{align}
     \newcontent{ \G{j} - \barG{j}  =  \vg{j} - \expvg{j}{j+1}{j}  }
    \label{eq: poisson solution slow}
\end{align}
we rewrite:
\begin{align}
    \inner{\y{j}}{\phi_j} = \langle \y{j} &, \newcontent{ \poissg{j}{j+2} - \expvg{j+1}{j+1}{j}} \label{eq: mse slow poisson 1}\\
    &  + \poissg{j+1}{j+2} - \poissg{j}{j+2}\label{eq: mse slow poisson 2}\\
    &  +\poissg{j}{j+1} - \poissg{j+1}{j+2}\label{eq: mse slow poisson 3}\\
    & + \newcontent{\expvg{j+1}{j+1}{j} - \expvg{j}{j+1}{j}}\rangle \label{eq: mse slow poisson 4}
\end{align}
the first term \eqref{eq: mse slow poisson 1} is 0 under expectation because it is a Martingale Difference. \newcontent{The last term \eqref{eq: mse slow poisson 4} is treated in the same way as in \eqref{eq: poisson fast 4} and it has the following bound
\begin{align*}
    \beta \abs{\sum_{j=0}^k (1-\mug \beta)^{k-j}\inner{\y{j}}{\expvg{j+1}{j+1}{j} - \expvg{j}{j+1}{j}}} \\
    \leq {\Lk \boundXi \boundXY |\Xi|(\alpha \boundF + \beta \boundG ) \over \mug}
\end{align*}
}
The second term \eqref{eq: mse slow poisson 2} is of order $O( \alpha + \beta)$:
\begin{align*}
     \abs{\eqref{eq: mse slow poisson 2}}
     \leq &  \norm{\y{j}} \norm {v_g(x_{j+1}, y_{j+1},\xi_{j+2}) - v_g(x_{j}, y_{j},\xi_{j+2})}\\
     \leq & \Lvg \boundXY  (\norm {y_{j+1} - y_j} + \norm {x_{j+1} - x_j} ) \\
     \leq & \Lvg \boundXY (\norm{\alpha \F{j}} + \norm{\beta \G{j}})\\
     \leq & \Lvg \boundXY( \alpha \boundF + \beta \boundG  )
\end{align*}
where in the third line we have used the definition of the updates and in the last line we have used the bounds from Lemma~\ref{lemma: bounds}. So we have: 
\begin{align*}
    \abs{ \beta \sum_{j=0}^k(1-\beta \mug)^{k-j} \inner{\y{j}}{v_g(x_{j+1}, y_{j+1},\xi_{j+2}) - v_g(x_{j}, y_{j},\xi_{j+2})}}\leq& \beta \sum_{j=0}^k(1-\beta \mug)^{k-j}\abs{\eqref{eq: mse slow poisson 2} }
    \\ \leq & {\Lvg \boundXY( \alpha \boundF + \beta \boundG  )\over \mug}
\end{align*}
where in the first inequality we have used the triangular inequality. Moving to the third term \eqref{eq: mse slow poisson 3}. Let $w_j = \vg{j}$, the term in the sum becomes:
\begin{align*}
    -\eqref{eq: mse slow poisson 3}=\inner{\y{j}}{w_{j+1} - w_j} &= \inner{\y{j}}{w_{j+1}} -  \inner{\y{j}}{w_j} \\
    &= \inner{\y{j+1} +\beta \G{j}}{w_{j+1}} -  \inner{\y{j}}{w_j}\\
    &= \inner{\y{j+1}}{w_{j+1}} -  \inner{\y{j}}{w_j} + \inner{\beta \G{j}}{w_{j+1}}
\end{align*}
We substitute in the sum:
\begin{align}
    \beta \sum_{j=0}^{k}  (1-\beta\mug)^{k-j}\inner{\y{j}}{w_{j+1} - w_j} &=  \beta \sum_{j=0}^{k}  (1-\beta\mug)^{k-j}(  \inner{\y{j+1}}{w_{j+1}} -  \inner{\y{j}}{w_j}) \label{eq: poisson mse slow 1 telescoping}\\ 
     &+ \beta^2 \sum_{j=0}^{k}  (1-\beta\mug)^{k-j} \inner{\G{j}}{w_{j+1}} \nonumber
\end{align}
treating the first term \eqref{eq: poisson mse slow 1 telescoping}:
\begin{align*}
     \eqref{eq: poisson mse slow 1 telescoping}=& \beta (1-\beta\mug)^{-1}\inner{\y{k+1}}{w_{k+1}} - \beta (1-\beta\mug)^k \inner{\y{0}}{w_{0}}\\
     &-\beta^2\mug\sum_{j=0}^{k}  (1-\beta\mug)^{k-(j+1)} \inner{\y{j+1}}{w_{j+1}}
\end{align*}
substituting: 
\begin{align}
\abs{\beta \sum_{j=0}^{k}  (1-\beta\mug)^{k-j}\inner{\y{j}}{w_{j+1} - w_j}} =&  \beta \abs{(1-\beta\mug)^{-1}\inner{\y{k+1}}{w_{k+1}} - (1-\beta\mug)^k \inner{\y{0}}{w_{0}}}  \label{eq: slow poisson telescoping remainder 1}\\
     & +\beta^2 \mug \abs{\sum_{j=0}^{k}  (1-\beta\mug)^{k-(j+1)} \inner{\y{j+1}}{w_{j+1}}}\label{eq: slow poisson telescoping remainder 2}\\
     &+ \beta^2 \abs{\sum_{j=0}^{k}  (1-\beta\mug)^{k-j} \inner{\G{j}}{w_{j+1}}}\label{eq: slow poisson telescoping remainder 3}
\end{align}
we bound the term \eqref{eq: slow poisson telescoping remainder 1}:
\begin{align*}
    \eqref{eq: slow poisson telescoping remainder 1} &\leq  \beta \boundXi \parentheses{ 2 - \beta \mug \over 1 - \beta \mug}
\end{align*}
the second term \eqref{eq: slow poisson telescoping remainder 2}:
\begin{align*}
   \eqref{eq: slow poisson telescoping remainder 2} &\leq \beta {\boundXi \over 1 - \beta \mug}
\end{align*}
and the last term \eqref{eq: slow poisson telescoping remainder 3}:
\begin{align*}
     \eqref{eq: slow poisson telescoping remainder 3} \leq  \beta { \boundXi \boundG \over \mug} 
\end{align*}
where we have used the bounds from Lemma~\ref{lemma: bounds} as we have done in detail for the previous section regarding the fast iterate. The last bound can be written as:
\begin{align*}
\boundMSESlowPoisson
\end{align*}
which concludes the proof.
\end{proof}
\newcontent{The following two lemmas are describing a uniform upper bounds of the solutions of the Riccati equation  \eqref{eq:riccati-equation} and Sylvester equation \eqref{eq:sylvester-equation-X} for a range of values of $\beta \over \alpha$.
\begin{lemma}\label{lemma: Uniform Bound Riccati}
Let $\varepsilon > 0$ and let $J_{yx} \in \mathbb{R}^{m\times n}, J_{xx}\in \mathbb{R}^{n\times n}, J_{yy} \in \mathbb{R}^{m\times m}, J_{xy} \in \mathbb{R}^{n\times m}$ with $J_{xx}$ being invertible, for every $\radiusU> 0 $ we define $\epsStar{0}, \epsStar{1}$ as follows : 
\begin{align}
    \epsStar{0} & = {\radiusU \over  \opnorm{  \J{xx}^{-1}}  (\Rbound) ( \opnorm{\J{yy}} +(\Rbound) \opnorm{\J{yx}})} ,\label{eq: Riccati eps for the ball}\\
    \epsStar{1} & = \frac{1}{\opnorm{ \J{xx}^{-1} } \left( \opnorm{ \J{yy} } + 2 \opnorm{ \J{yx} } (\Rbound) \right)}\label{eq: Riccati eps for the contraction}
\end{align}
with $U_0 = J_{xx}^{-1}J_{xy}$. Then for every $\varepsilon < \min\{\epsStar{0}, \epsStar{1}\}$ there exists a matrix $U \in \mathbb{R}^{n\times m}$ that belongs to the closed ball $\ballU = \{U : \opnorm{U - U_0} \leq \radiusU\}$ that is the unique solution of the following Riccati equation: 
\begin{align} 
     \varepsilon U J_{yx} U
     + J_{xx} U
     - \varepsilon U J_{yy}
     =
     J_{xy} \label{eq: general riccati equation}
\end{align}
\end{lemma}
\begin{proof}
We multiply the Riccati equation \eqref{eq: general riccati equation} by $\J{xx}^{-1}$,since we assumed it is invertible:
\[
\eps \J{xx}^{-1} U \J{yx} U + U - \eps \J{xx}^{-1} U \J{yy} = \underbrace{\J{xx}^{-1} \J{xy}}_{\triangleq U_0}
\]
which gives the fixed point equation:
\begin{equation}
U = U_0 + \eps \left( \J{xx}^{-1} U \J{yy} - \J{xx}^{-1} U \J{yx} U \right) 
\end{equation}
We define $T$ to be the operator of the fixed point equation.
We define the closed ball:
\[
\ballU = \{ U : \opnorm{ U - U_0 } \leq \radiusU \}
\]
First let's find $\epsStar{0}$ such that $T(\ballU) \subseteq \ballU$. Let $M \in \ballU$:
\begin{align*}
     T(M) = U_0 + \eps \left( \J{xx}^{-1} M \J{yy} - \J{xx}^{-1} M \J{yx} M \right) 
\end{align*}
calculating the norm 
\begin{align*}
 \opnorm{T(M) - U_0} &= \eps \opnorm{  \J{xx}^{-1} M \J{yy} - \J{xx}^{-1} M \J{yx} M  }\\
 & \leq \eps (\opnorm{  \J{xx}^{-1} M \J{yy}} +\opnorm{ \J{xx}^{-1} M \J{yx} M  })\\
 & = \eps (\opnorm{  \J{xx}^{-1} (M - U_0 + U_0) \J{yy}} +\opnorm{ \J{xx}^{-1} (M - U_0 + U_0) \J{yx} (M - U_0 + U_0)  })\\
 &\leq \eps (\opnorm{  \J{xx}^{-1}}  (\opnorm{M - U_0  } +\opnorm{ U_0}) \opnorm{\J{yy}} \\
 &\quad +\opnorm{  \J{xx}^{-1}} (\opnorm{M - U_0  } +\opnorm{ U_0})^2 \opnorm{\J{yx}})\\
 &\leq \eps (\opnorm{  \J{xx}^{-1}}  (R +\opnorm{ U_0}) \opnorm{\J{yy}} +\opnorm{  \J{xx}^{-1}} (\radiusU+\opnorm{ U_0})^2 \opnorm{\J{yx}})
\end{align*}
now we condition on $\varepsilon$ to make $T(M) \in \ballU$, we need $\varepsilon \leq \varepsilon_0^* $ such that : 
\begin{align*}
    \varepsilon_0^* = {\radiusU \over  \opnorm{  \J{xx}^{-1}}  (\Rbound) ( \opnorm{\J{yy}} +(\Rbound) \opnorm{\J{yx}})}
\end{align*}
Now we prove $T$ is a contraction. Let $U_1, U_2 \in \ballU$, we calculate the norm of their difference:
\begin{align*}
\opnorm{ T(U_1) - T(U_2) } &= \opnorm{ \eps \left( \J{xx}^{-1} U_1 \J{yy} - \J{xx}^{-1} U_1 \J{yx} U_1 \right) - \eps \left( \J{xx}^{-1} U_2 \J{yy} - \J{xx}^{-1} U_2 \J{yx} U_2 \right) }
\end{align*}
We have:
\[
U_1 A U_1 - U_2 A U_2 = U_1 A U_1 - U_1 A U_2 + U_1 A U_2 - U_2 A U_2 = U_1 A (U_1 - U_2) + (U_1 - U_2) A U_2
\]
We replace:
\begin{align*}
\opnorm{ T(U_1) - T(U_2) } &= \eps \opnorm{ \J{xx}^{-1} \left( (U_1 - U_2) \J{yy} - ( (U_1 - U_2) \J{yx} U_2 + U_1 \J{yx} (U_1 - U_2) ) \right) } \\
\opnorm{ T(U_1) - T(U_2) } &\leq \eps \opnorm{ \J{xx}^{-1} } \opnorm{ U_1 - U_2 } \left( \opnorm{ \J{yy} } + \opnorm{ \J{yx} } ( \opnorm{ U_1 } + \opnorm{ U_2 } ) \right)
\end{align*}
Bounding the norms:
\[
\norm{ U_1 } \leq \norm{ U_1 - U_0 } + \norm{ U_0 }  \leq \Rbound
\]
\[
\norm{ U_2 } \leq \norm{ U_2 - U_0 } + \norm{ U_0 }  \leq \Rbound
\]
\[
\norm{ T(U_1) - T(U_2) } \leq \underbrace{\eps \left( \opnorm{ \J{xx}^{-1} } \left( \opnorm{ \J{yy} } + 2 \opnorm{ \J{yx} } (\Rbound) \right) \right)}_{r(\eps)} \opnorm{ U_1 - U_2 }
\]
To make $r(\eps) < 1$ we need:
\[
\eps < \frac{1}{\opnorm{ \J{xx}^{-1} } \left( \opnorm{ \J{yy} } + 2 \opnorm{ \J{yx} } (\Rbound) \right)} = \epsStar{1}
\]
For $\eps < \min\{\epsStar{0},\epsStar{1}, \}$, the solution $U$ of the Riccati equation is unique and is inside the ball $\ballU$. It can be obtained by successive application of $T$ on $U_0$. This shows that the uniform bound on $\norm{ U } \leq \Rbound$ is proved.
\end{proof}
}
\newcontent{
\begin{lemma}\label{lemma: Uniform Bound Sylvester}
    Let $\eps > 0$ and let $J_{yx} \in \mathbb{R}^{m\times n}, J_{xx}\in \mathbb{R}^{n\times n}, J_{yy} \in \mathbb{R}^{m\times m}, J_{xy} \in \mathbb{R}^{n\times m}$ with $J_{xx}$ being invertible, define $U_0 = \J{xx}^{-1}\J{xy}$, and for all $\radiusU > 0 $ there exists $\eps^*$ such that for all $\eps < \eps^*$ the matrix $U_\eps \in \mathbb{R}^{n\times m}$ is bounded by $\opnorm{U_\eps} \leq \Rbound$, we set $\radiusX = 2 \opnorm{\J{yx}} \opnorm{ \J{xx}^{-1}}$ and we define   $\epsStar{2}, \epsStar{3}$  : 
    \begin{align}
    \epsStar{2} & = {1  \over  \opnorm{\J{yy}}  \opnorm{\J{xx}^{-1}} + 2 (\Rbound) \opnorm{ \J{yx}} \opnorm{ \J{xx}^{-1}}},\label{eq: eps condition sylvester ball}\\
    \epsStar{3} & = \frac{1}{\opnorm{ \J{xx}^{-1} } \left( \opnorm{ \J{yy} } + 2 \opnorm{ \J{yx} } (\Rbound) \right)}\label{eq: eps condition sylvester contraction}
\end{align}
then for all $\eps < \min\{\eps^* , \epsStar{2}, \epsStar{3} \}$ there exists a matrix  $X\in \mathbb{R}^{m\times n}$ that belongs to the closed ball  $\ballX = \{X : \opnorm{X} \leq \eps \radiusX \}$ which is the unique solution of the following Sylvester equation: 
\[
X (\J{xx} + \eps U_\eps \J{yx}) - \eps (\J{yy} - \J{yx} U_\eps) X = - \eps \J{yx}
\]
\end{lemma}
\begin{proof}
    To make the notation lighter we drop $\eps$ from $U_\eps$ such that $U = U _\eps$. We multiply the Sylvester equation by $\J{xx}^{-1}$ on the right:
\[
X + \eps X U \J{yx} \J{xx}^{-1} - \eps \J{yy} X \J{xx}^{-1} + \eps \J{yx} U X \J{xx}^{-1} = - \eps \J{yx} \J{xx}^{-1}
\]
which gives the fixed point equation:
\[
X = \underbrace{\eps \left( \J{yy} X \J{xx}^{-1} - X U \J{yx} \J{xx}^{-1} - \J{yx} U X \J{xx}^{-1} - \J{yx} \J{xx}^{-1} \right)}_{T(X)}
\]
We define the ball $\ballX = \{ X : \opnorm{ X } \leq \varepsilon \radiusX \}$ where $\radiusX = 2 \opnorm{\J{yx}} \opnorm{ \J{xx}^{-1}}$ (This choice is explained in the next lines). Let $M \in \ballX$ and $U \in B^{U_0}_{\radiusU}$:
\begin{align*}
    \opnorm{T(M)} &= \varepsilon \opnorm{\J{yy} M \J{xx}^{-1} - M U \J{yx} \J{xx}^{-1} - \J{yx} U M \J{xx}^{-1} - \J{yx} \J{xx}^{-1}} \\
    & \leq \varepsilon (\opnorm{\J{yy}} \opnorm{ M} \opnorm{\J{xx}^{-1}} + 2\opnorm{M} \opnorm{ U} \opnorm{ \J{yx}} \opnorm{ \J{xx}^{-1}} + \opnorm{\J{yx}} \opnorm{ \J{xx}^{-1}})\\
    & = \varepsilon ( \opnorm{ M}(\opnorm{\J{yy}}  \opnorm{\J{xx}^{-1}} + 2 \opnorm{ U} \opnorm{ \J{yx}} \opnorm{ \J{xx}^{-1}}) + \opnorm{\J{yx}} \opnorm{ \J{xx}^{-1}})\\
    & \leq \varepsilon \radiusX {  \varepsilon \radiusX (\opnorm{\J{yy}}  \opnorm{\J{xx}^{-1}} + 2 (\Rbound) \opnorm{ \J{yx}} \opnorm{ \J{xx}^{-1}}) + \opnorm{\J{yx}} \opnorm{ \J{xx}^{-1}} \over \radiusX}
\end{align*}
a first condition to set is that 
\[
{  \varepsilon \radiusX (\opnorm{\J{yy}}  \opnorm{\J{xx}^{-1}} + 2 (\Rbound) \opnorm{ \J{yx}} \opnorm{ \J{xx}^{-1}}) + \opnorm{\J{yx}} \opnorm{ \J{xx}^{-1}} \over \radiusX} \leq 1 
\]
to make the operator $T$ maps into $\ballX$ which is equivalent to 
\[
{  \varepsilon  } < {\radiusX - \opnorm{\J{yx}} \opnorm{ \J{xx}^{-1}} \over \radiusX (\opnorm{\J{yy}}  \opnorm{\J{xx}^{-1}} + 2 (\Rbound) \opnorm{ \J{yx}} \opnorm{ \J{xx}^{-1}})}
\]
where we extract two conditions , the first is that $\radiusX > \opnorm{\J{yx}} \opnorm{ \J{xx}^{-1}}$ since $\eps$ is strictly positive, here we justify the choice of the radius $\radiusX = 2 \opnorm{\J{yx}} \opnorm{ \J{xx}^{-1}}$, once we have done that we can have a condition on $\varepsilon \leq \epsStar{2}$ to ensure that image of $T$ stays in the ball $\ballX$ such that:
\begin{align*}
    \epsStar{2} ={1  \over  \opnorm{\J{yy}}  \opnorm{\J{xx}^{-1}} + 2 (\Rbound) \opnorm{ \J{yx}} \opnorm{ \J{xx}^{-1}}}
\end{align*}
Now we move to the contraction part of $T$. Let $X_1, X_2 \in \ballX$:
\begin{align*}
\opnorm{ T(X_1) - T(X_2) } &\leq \eps \opnorm{ \J{xx}^{-1} } \left( \opnorm{ \J{yy} } + \opnorm{ U \J{yx} } + \opnorm{ \J{yx} U } \right) \opnorm{ X_1 - X_2 } \\
&\leq \eps \opnorm{ \J{xx}^{-1} } \left( \opnorm{ \J{yy} } + 2 \opnorm{ \J{yx} } (\Rbound) \right) \opnorm{ X_1 - X_2 }
\end{align*}
The condition for the contraction is the same as the Riccati equation
\[
\epsStar{3} = \epsStar{1}
\]
then by Banach fixed point theorem, the map $T$ has a fixed point solution $X$ that is unique inside the ball $\ballX$ for every $\eps < \min\{\epsStar{2}, \epsStar{3}\}$.
\end{proof}}
In the next lemmas, we extract contractions of the bias dynamics using Lyapunov norms suitable for a given range of step sizes. The choice of the Lyapunov equations \eqref{eq: lyapunov equation fast} and \eqref{eq: lyapunov equation slow} appears in \cite[Chapter 7]{kokotovic1999singular} to study the stability of the equivalent singular perturbation system. The bound on the $\LyapF$-norm of $I - \gamma A$ is common and appears in \cite{haque2025tightfinitetimebounds, durmus2021tighthighprobabilitybounds, kwon2024twotimescalelinearstochasticapproximation,kaledin2020finite}. However, in our case we deal with nonlinear systems so we have extra perturbations emerging from the coupling between the two variables, so we extend the classical way to prove the contraction property under the suitable Lyapunov norms to account for the new added perturbations which are dependent on the solutions of the Riccati and the Sylvester equations that we showed previously to exist and to be bounded under certain assumptions on the step sizes. 
\newcontent{
\begin{lemma}\label{lemma: Lyapunov's fast contraction}
Let $-J_{xx}$ be a Hurwitz matrix in $\mathbb R ^{d\times d}$, $J_{yx}\in \mathbb R ^{n \times d}$ such that $\opnorm{\J{yx}} \leq M_J$ for some positive constant $M_J$, $\LyapF \in \mathbb R ^{d\times d}$ a symmetric positive definite matrix a solution to the following Lyapunov equation:
\begin{align}
J_{xx}^T \LyapF + \LyapF J_{xx} = I. \label{eq: lyapunov equation fast}
\end{align}
Let $\radiusU> 0 $, there exists $\alpha^* = \alpha^*(\radiusU)$ where $\alpha^*(\radiusU)$ is defined by:
\begin{align}
    \alpha^*(\radiusU) = {1\over 4 \opnorm{\LyapF}(\LyapFNorm{ \J{xx} }^2+ {\epsStar{4}}^2 \kappa(L)(\Rbound)^2\LyapFNorm{\J{yx} }^2)} \label{eq: alpha condition fast Lyapunov}
\end{align}
such that for every $  0 < \alpha \leq \alpha^*$ and for every $\beta$ such that ${\beta \over \alpha} < \min\{\epsStar{0}, \epsStar{1},\epsStar{5}\}$ where $\epsStar{0}, \epsStar{1},\epsStar{5}$ depend on $\radiusU$ and are fully defined in \eqref{eq: Riccati eps for the ball}, \eqref{eq: Riccati eps for the contraction}, \eqref{eq: eps condition for fast lyapunov contraction},
we define the matrix $U_{\beta \over \alpha} \in \mathbb R ^{d\times n}$ to be bounded by $\|U_{\beta \over \alpha }\| \leq \Rbound$, then we have: 
\begin{align*}
    \LyapFNorm{ I - \alpha \J{xx} - \beta U_{\beta\over\alpha} \J{yx} }  \leq e^{-c\alpha }
\end{align*}
where $c = (8 \opnorm{\LyapF})^{-1}$.
\end{lemma}
\begin{proof}
Under the Lyapunov norm:
\begin{align*}
\LyapFNorm{ I - \alpha \J{xx} - \beta U \J{yx} }^2 &= \sup_{\LyapFNorm{u} = 1} u^T (I - \alpha \J{xx} - \beta U \J{yx})^T\LyapF(I - \alpha \J{xx} - \beta U \J{yx}) u \\
&= \sup_{\LyapFNorm{u} = 1} \{ u^T (L - \alpha \J{xx}^T\LyapF- \beta (U \J{yx})^T\LyapF- \alpha\LyapF\J{xx} + \alpha^2 \J{xx}^T\LyapF\J{xx} \\
&\hspace{1.6cm}+ \alpha \beta \J{yx}^T U^T\LyapF\J{xx} - \beta\LyapF U \J{yx} + \alpha \beta \J{xx}^T\LyapF U \J{yx} + \beta^2 \J{yx}^T U^T\LyapF U \J{yx}) u \} \\
&= \sup_{\LyapFNorm{u}=1} \{ u^T\LyapF u - \alpha u^T(\J{xx}^T\LyapF +\LyapF \J{xx})u + \alpha^2 u^T \J{xx}^T\LyapF \J{xx} u \\
&\hspace{1.6cm} - \beta u^T (\underbrace{\J{yx}^T U^T\LyapF +\LyapF U \J{yx}}_{A_1})u + \alpha \beta u^T (\underbrace{\J{yx}^T U^T\LyapF \J{xx} + \J{xx}^T\LyapF U \J{yx}}_{A_2}) u \\
&\hspace{1.6cm}+ \beta^2 u^T (\J{yx}^T U^T\LyapF U \J{yx}) u \} \\
&= 1 + \sup_{\LyapFNorm{u}=1} \{ - \alpha u^T u + \alpha^2 u^T \J{xx}^T\LyapF \J{xx} u - \beta u^T A_1 u + \alpha \beta u^T A_2 u \\
&\hspace{2.2cm}+ \beta^2 u^T (\J{yx}^T U^T\LyapF U \J{yx}) u \} \\
&\leq 1 + \sup_{\LyapFNorm{u}=1} \{ - \alpha u^T u\} + \alpha^2 \LyapFNorm{ \J{xx} }^2+ \beta^2 \LyapFNorm{ U \J{yx} }^2 + \alpha \beta \opnorm{ \Lsim{A_2} }  \\
&\quad + \beta \sup_{\LyapFNorm{u} = 1} -u^T A_1 u\\
&\leq 1 - \alpha (2\opnorm{\LyapF })^{-1} + \alpha^2 \LyapFNorm{ \J{xx} }^2+ \beta^2 \LyapFNorm{ U \J{yx} }^2\\
&\quad  - \alpha (2\opnorm{\LyapF })^{-1} + \alpha \beta \opnorm{ \Lsim{A_2} }  + \beta \sup_{\LyapFNorm{u} = 1} -u^T A_1 u
\end{align*}
We have used the fact that $\inf_{\LyapFNorm{u}=1} \{ u^T u \} = \opnorm{\LyapF}^{-1} =(2\opnorm{\LyapF})^{-1} + (2\opnorm{\LyapF})^{-1}$.We will use the symmetry of $A_1$ to bound the last term such that 
\begin{align*}
    \beta \sup_{\LyapFNorm{u} = 1} -u^T A_1 u \leq  \beta \sup_{\LyapFNorm{u} = 1} |u^T A_1 u| = \beta \sup_{\norm{z} = 1} |z^T \LsimSym{A_1} z| = \beta \opnorm{ \LsimSym{A_1} }
\end{align*}
we have 
\begin{align*}
\LyapFNorm{ I - \alpha \J{xx} - \beta U \J{yx} }^2 
&\leq 1 - \alpha (2\opnorm{\LyapF })^{-1} + \alpha^2 (\LyapFNorm{ \J{xx} }^2+ {\beta^2\over \alpha ^2} \LyapFNorm{ U \J{yx} }^2)\\
&\quad  - \alpha (2\opnorm{\LyapF })^{-1} + \beta  \alpha \opnorm{ \Lsim{A_2} }  + \beta \opnorm{ \LsimSym{A_1} }
\end{align*}
Since $U$ is a solution of the Riccati equation \eqref{eq:riccati-equation} then by Lemma \ref{lemma: Uniform Bound Riccati} for all $\radiusU>0$ there exists $\epsStar{0}$ and $\epsStar{1}$ defined in \eqref{eq: Riccati eps for the ball} and \eqref{eq: Riccati eps for the contraction}, respectively, such that for ${\beta \over \alpha} < \epsStar{4} \triangleq \min\{\epsStar{0},\epsStar{1}\}  $ we have $\norm{U} \leq \Rbound$. We apply the result of lemma:
\begin{align*}
    \LyapFNorm{ I - \alpha \J{xx} - \beta U \J{yx} }^2 
&\leq 1 - \alpha (2\opnorm{\LyapF })^{-1} + \alpha^2 (\LyapFNorm{ \J{xx} }^2+ {\epsStar{4}}^2 \kappa(L)(\Rbound)^2\LyapFNorm{\J{yx} }^2)\\
&\quad  - \alpha (2\opnorm{\LyapF })^{-1} + \beta  \alpha \opnorm{ \Lsim{A_2} }  + \beta \opnorm{ \LsimSym{A_1} }\\ 
& \leq 1 - \alpha (2\opnorm{\LyapF })^{-1} + \alpha^2 (\LyapFNorm{ \J{xx} }^2+ {\epsStar{4}}^2 \kappa(L)(\Rbound)^2\LyapFNorm{\J{yx} }^2)\\
&\quad  - \alpha (2\opnorm{\LyapF })^{-1} + 2 \beta  \alpha \kappa (L)(\opnorm{\J{yx}}\opnorm{\J{xx}}(\Rbound))  \\
& \quad + 2 \beta \kappa(L)^{1/2}(M_J (\Rbound)) 
\end{align*}
First, we look for values of $\alpha$ such that the first line is less than one to get the contraction property. This is true for the choice of $\alpha \leq \alpha^*(\radiusU)$ defined in the body of the lemma so we have:
\begin{align*}
    \LyapFNorm{ I - \alpha \J{xx} - \beta U \J{yx} }^2
& \leq 1 - c'\alpha  \\
&\quad  - \alpha (2\opnorm{\LyapF })^{-1} + 2 \beta  \alpha^* \kappa (L)(\opnorm{\J{yx}}\opnorm{\J{xx}}(\Rbound))  \\
& \quad + 2 \beta \kappa(L)^{1/2}(M_J(\Rbound)) 
\end{align*}
with $c' = (4 \opnorm{\LyapF})^{-1} $. Second, we need the term in the last two lines to be negative so this will imply a condition of the type: 
\begin{align}
    {\beta \over \alpha } \leq \epsStar{5} \triangleq {1 \over 2 \opnorm{\LyapF} (\Rbound) (2   \alpha^* \kappa (L)\opnorm{\J{yx}}\opnorm{\J{xx}}  + 2 \kappa(L)^{1/2} M_J )} \label{eq: eps condition for fast lyapunov contraction}
\end{align}
Once this is established we have:
\begin{align*}
\LyapFNorm{ I - \alpha \J{xx} - \beta U \J{yx} }^2  \leq 1 -c' \alpha \leq e^{-c\alpha }
\end{align*}
this implies that: 
\begin{align*}
\LyapFNorm{ I - \alpha \J{xx} - \beta U \J{yx} }  \leq e^{-{c'\over 2}\alpha }
\end{align*}
which establishes the proof.
\end{proof}}
\newcontent{
\begin{lemma}\label{lemma: Lyapunov's slow contraction}
Let $-\Delta$ be a Hurwitz matrix in $\mathbb R ^{n\times n}$, $J_{yx}\in \mathbb R ^{n \times m}$ such that $\opnorm{J_{yx}} \leq M_J$ for some positive constant $M_J$, $\LyapG \in \mathbb R ^{n\times n}$ a symmetric positive definite matrix a solution to the following Lyapunov equation:
\begin{align}
\Delta^T \LyapG + \LyapG \Delta = I. \label{eq: lyapunov equation slow}
\end{align}
Let $\radiusU = \radiusUfixed$, we define $\beta^* = \beta^*(\radiusU)$ where $\beta^*(\radiusU)$ is defined by:
\begin{align}
    \beta^* = {1 \over 4 \opnorm{G} ( 2\kappa(G) \opnorm{\Delta} M_J\radiusU+  \LyapGNorm{\Delta }^2 + \kappa(G) M_J^2 \radiusU^2)}\label{eq: condition beta contraction}
\end{align}
for all $\alpha >0 $ and $\beta \leq \beta^*$ with ${\beta \over \alpha} < \min\{\epsStar{0},\epsStar{1}\}$ where $\epsStar{0},\epsStar{1}$ are dependent on $\radiusU$ and defined in \eqref{eq: Riccati eps for the ball} and \eqref{eq: Riccati eps for the contraction}, the matrix $U_{\beta \over \alpha} \in \ballU$ is a solution of the Riccati equation defined in \eqref{eq: general riccati equation} with $\eps = {\beta \over \alpha }$ and $U_0 = \J{xx}^{-1} \J{xy}$, and we have:
\begin{align*}
    \LyapGNorm{I - \beta (\Delta-  \J{yx}(U_{\beta\over \alpha } - U_0))}  \leq e^{-d\beta }
\end{align*}
where $d = (8 \opnorm{\LyapG})^{-1}$.
\end{lemma}
\begin{proof}
    Taking the Lyapunov norm 
\begin{align*}
    \LyapGNorm{I - \beta (\Delta-  \J{yx}(U - U_0))}^2 &= \sup_{\LyapGNorm{u}=1}\{u^T (I - \beta \Delta- \beta \J{yx}(U - U_0))^TG (I - \beta \Delta- \beta \J{yx}(U - U_0))u 
    \} \\
    & =\sup_{\LyapGNorm{u}=1} \{ u^T Gu 
    - \beta u^T G  \Delta u 
    - \beta u^T G  \J{yx}(U - U_0)u 
    -\beta  u^T \Delta^TG u 
    \\
    & \hspace{1.5cm}+ \beta^2 u^T  \Delta^TG  \Delta u  + \beta^2 u^T  \Delta^TG \J{yx}(U - U_0)u 
    \\
    & \hspace{1.5cm}- \beta  u^T (U - U_0)^T\J{yx}^TG u  + \beta^2 u^T  (U - U_0)^T\J{yx}^TG \Delta u\\
    &\hspace{1.5cm} + \beta^2 u^T (U - U_0)^T\J{yx}^TG  \J{yx}(U - U_0)u \}\\
    & =1 + \sup_{\LyapGNorm{u}=1} \{  
    - \beta u^T u 
    - 2\beta u^T G  \J{yx}(U - U_0) u  + 2\beta^2 u^T  \Delta^TG \J{yx}(U - U_0)u 
     \\
     & \hspace{2.1cm}+ \beta^2 u^T  \Delta^TG  \Delta u+ \beta^2 u^T (U - U_0)^T\J{yx}^TG  \J{yx}(U - U_0)u \}\\
     & \leq 1  - \beta \opnorm{G}^{-1}
    + 2\beta\sup_{\LyapGNorm{u}=1} \{-  u^T G  \J{yx}(U - U_0) u\}
     \\
    & \quad+ 2\beta^2 \sup_{\LyapGNorm{u}=1} \{u^T  \Delta^TG \J{yx}(U - U_0)u \}
     + \beta^2 \LyapGNorm{\Delta }^2 + \beta^2 \LyapGNorm{\J{yx}(U - U_0)}^2 \\
     & \leq 1  - \beta \opnorm{G}^{-1}
    + 2\beta \opnorm{ G^{1/2}  \J{yx}(U - U_0) G^{-1/2}}
     \\
    & \quad+ 2\beta^2  \opnorm{ G^{-1/2} \Delta^TG \J{yx}(U - U_0)G^{-1/2} }
     + \beta^2 \LyapGNorm{\Delta }^2 + \beta^2 \LyapGNorm{\J{yx}(U - U_0)}^2 \\
    & \leq 1  - \beta \opnorm{G}^{-1}
    + 2\beta  \kappa(G)^{1/2} M_J \opnorm{ U - U_0 }
     \\
    & \quad+ 2\beta^2  \opnorm{ G^{-1/2} \Delta^TG \J{yx}(U - U_0)G^{-1/2} }
     + \beta^2 \LyapGNorm{\Delta }^2 + \beta^2 \LyapGNorm{\J{yx}(U - U_0)}^2 
\end{align*}
Here we set the radius of convergence of the Riccati equation to calibrate for the negative term in the first line of the last inequality above. We want:
\begin{align*}
    2  \kappa(G)^{1/2} M_J \opnorm{ U - U_0 } \leq  (2\opnorm{G})^{-1}
\end{align*}
which is equivalent to:
\begin{align*}
 \opnorm{ U - U_0 } \leq  {1\over 4 \opnorm{G} \kappa(G)^{1/2} M_J } = \radiusU
\end{align*}
So under this condition we have 
\begin{align*}
    \LyapGNorm{I - \beta (\Delta-  \J{yx}(U - U_0))}^2 
    & \leq 1  - \beta (2\opnorm{G})^{-1} + \beta^2 ( 2\kappa(G) \opnorm{\Delta} M_J \radiusU +  \LyapGNorm{\Delta }^2 + \kappa(G) M_J^2 \radiusU^2)
\end{align*}
So the contraction is established for $\beta \leq \beta^*$ where $\beta^*$ is defined in the body of the lemma and for $d' = (4 \opnorm{G})^{-1}$ we have 
\begin{align*}
    \LyapGNorm{I - \beta (\Delta-  \J{yx}(U - U_0))}^2 
    & \leq 1  - d\beta  \leq e^{-d'\beta }
\end{align*}
this implies that: 
\begin{align*}
    \LyapGNorm{I - \beta (\Delta-  \J{yx}(U - U_0))}  \leq e^{-{d'\over 2}\beta }
\end{align*}
\end{proof}}
In the next lemma we handle the norms of the sum of deviations appearing in the proof of Theorem~\ref{thm : bias}.
\begin{lemma}\label{lemma:bias-poisson} Taking the sequence $\{x_k,y_k\}_{k\geq 0}$ defined by \eqref{eq: TTSA update} and under assumptions A.\ref{assumption: Boundedness}-\ref{assumption: All Kernel Properties} and B.\ref{assump : derivability}. Assume also that the step sizes $\alpha$ and $\beta$ satisfy $\alpha<\alpha^*$, $\beta < \beta^*$ and ${\beta \over \alpha }< \omega^*$ such that $\alpha^*,\beta^*, \omega^*$ are defined in \eqref{eq: alpha star condition for bias thm},\eqref{eq: beta star condition for bias thm},\eqref{eq: omega star condition for bias thm} then, the term \eqref{eq:fast-bias-term-3} is bounded by:
\begin{align*}
  \eqref{eq:fast-bias-term-3} \leq & \alpha A_1 +\beta (A_2 + B_1)+ \alpha ^2 A_3+ \beta^2 B_2 + \alpha \beta (A_4 + B_3) +{\beta ^2 \over \alpha} B_4
\end{align*}
in addition the term \eqref{eq:slow-bias-term-3} is bounded by: 
\begin{align*}
    \eqref{eq:slow-bias-term-3} \leq & \alpha (C_2 + D_4) + \beta (C_1 + D_1 )+ \alpha \beta (C_4 + D_3) + \beta^2 (C_3 + D_2)
\end{align*}
where the constants are independent of $\alpha$ and $\beta$ and are detailed in the proof.
\end{lemma}
\begin{proof} We start with bounding the first term: 
\begin{align}
    \alpha \norm{\sum_{j=0}^k (I-\alpha J_{xx}^{\alpha, \beta})^{k-j}\Expect{\psi_j}} \label{eq: sum of fast contracted markovian sum 1}
\end{align}
In Lemma \ref{lemma : poisson mse} we have seen that there exists a function $v_f$ the solution of the Poisson equation \eqref{eq: poisson solution fast}, which allowed us to write:
\begin{align*}
   \Expect{\psi_s} = \Expect{\vf{s} - \expvf{s}{s+1}{s}}
\end{align*}
we could write the term inside the expectation as:
\begin{align}
    \psi_s =& v_f(x_s, y_s, \xi_{s+2}) - \newcontent{\expvf{s+1}{s+1}{s}} \label{eq: poisson fast 1 (bias)}\\
    &+  v_f(x_{s+1}, y_{s+1}, \xi_{s+2}) -  v_f(x_s, y_s, \xi_{s+2}) \label{eq: poisson fast 2 (bias)}\\
    &  + \vf{s} - v_f(x_{s+1}, y_{s+1}, \xi_{s+2})\label{eq: poisson fast 3 (bias)}\\
    & + \newcontent{\expvf{s+1}{s+1}{s} - \expvf{s}{s+1}{s}} \label{eq: poisson fast 4 (bias)}
\end{align}
where the first term \eqref{eq: poisson fast 1 (bias)} is a Martingale Difference Sequence and \newcontent{we denote $K_s \triangleq \Kernel{s}{}$}. By taking the expectation, we get:
\begin{align*}
    \Exp[v_f(x_s, y_s, \xi_{s+2}) - \newcontent{\expvf{s+1}{s+1}{s}} ] =&  \Exp \big [ \newcontent{\expvf{s+1}{s+1}{s}} \\  & - \newcontent{\expvf{s+1}{s+1}{s}} \big] \\
    =& 0
\end{align*}
moving to the second term \eqref{eq: poisson fast 2 (bias)} and for legibility, we set $v_f(x_s,y_s, \xi_{s+2}) = v_ s$ and $v_f(x_{s+1},y_{s+1}, \xi_{s+2}) = v_{s+1}$ we have:
\begin{align*}
    \alpha \norm{\sum_{s=0}^k (I- \alpha \jxxab)^{k-s} (v_s - v_{s+1})}\leq & \alpha  \sum_{s=0}^k \norm{I- \alpha \jxxab}^{k-s} \norm{v_s - v_{s+1}}\\
    \leq & \alpha\scondL \sum_{s=0}^k e^{-c\alpha (k-s)}\norm{v_s - v_{s+1}}\\
    \leq & \alpha\scondL \Lvf \sum_{s=0}^k e^{-c\alpha (k-s)}(\norm{x_s - x_{s+1}}+ \norm{y_s - y_{s+1}})\\
    \leq &  \alpha\scondL\Lvf  (\alpha \boundF + \beta \boundG) \sum_{s=0}^k e^{-c\alpha s} \\
    \leq & \alpha\scondL \Lvf (\alpha \boundF + \beta \boundG)  ( 1 + \int _0^\infty  e^{-c\alpha t}dt)\\
    = &  \alpha c^{-1}\scondL \Lvf  \boundF   +  \beta c^{-1}\scondL \Lvf  \boundG \\
    & + \alpha^2 \scondL \Lvf \boundF +  \alpha \beta \scondL \Lvf  \boundG 
\end{align*}
Next, we handle the pseudo-telescoping term \eqref{eq: poisson fast 3 (bias)}, but before that we set $\vf{s}= v_s$, $v_f(x_{s+1}, y_{s+1}, \xi_{s+2}) = v_{s+1}$, $\jxxab = J_{xx}+ {\beta \over \alpha }U J_{yx}$ and $A_{\alpha,\beta} = I - \alpha \jxxab$. 
\newcontent{From the Neumann series convergence, $A_{\alpha,\beta}$ is invertible if we have a condition of the type \[\alpha \opnorm{\jxxab} < 1 \]
since we are assuming 
\begin{align}
    \alpha < {1\over \opnorm{\J{xx}} + (\Rbound)\opnorm{\J{yx}}} \label{eq: alpha for invertible fast bias}
\end{align}
then the condition is satisfied and $A_{\alpha,\beta}$ is invertible and the inverse can be given by the Neumann series $A_{\alpha,\beta}^{-1} = \sum_{k=0}^{\infty} \parentheses{\alpha \jxxab}^k $. However, we will not need the explicit form of the inverse, so we use it directly:}
\begin{align*}
    \sum_{s=0}^k A_{\alpha,\beta}^{k-s}(v_s- v_{s+1}) =& \sum_{s=0}^k A_{\alpha,\beta}^{k-s}v_s- A_{\alpha,\beta} A_{\alpha,\beta}^{k-(s+1)}v_{s+1}\\
    =&\sum_{s=0}^k A_{\alpha,\beta}^{k-s}v_s- A_{\alpha,\beta}^{k-(s+1)}v_{s+1}+ \alpha \jxxab A_{\alpha,\beta}^{-1} \sum_{s=0}^{k}  A_{\alpha,\beta}^{k-s}v_{s+1}\\
    = & A_{\alpha,\beta}^kv_0 - A_{\alpha,\beta}^{-1}v_{k+1} + \alpha \jxxab A_{\alpha,\beta}^{-1} \sum_{s=0}^{k}  A_{\alpha,\beta}^{k-s}v_{s+1}
\end{align*}
taking the norm:
\begin{align*}
    \alpha \norm{\sum_{s=0}^k A_{\alpha,\beta}^{k-s}(v_s- v_{s+1})} \leq  &\alpha^2 \opnorm{\jxxab} \opnorm{A_{\alpha,\beta}^{-1}} \sum_{s=0}^{k}  \opnorm{A_{\alpha,\beta}^{k-s}} \norm{v_{s+1}} + \alpha \opnorm{A_{\alpha,\beta}^k}\norm{v_0} \\&+ \alpha \opnorm{A_{\alpha,\beta}^{-1}}\norm{v_{k+1}} \\
    \leq &\alpha^2 \scondL \opnorm{\jxxab} \opnorm{A_{\alpha,\beta}^{-1}}\boundXi\sum_{s=0}^{k}e^{-c\alpha s}+ \alpha e^{-c\alpha k} \scondL\boundXi\\& + \alpha \opnorm{A_{\alpha,\beta}^{-1}} \boundXi \\
    \leq & \scondL\opnorm{\jxxab} \opnorm{A_{\alpha,\beta}^{-1}}\boundXi\parentheses{\alpha^2 + {\alpha \over c }}+ \alpha e^{-c\alpha k} \scondL \boundXi \\
    &+ \alpha \opnorm{A_{\alpha,\beta}^{-1}} \boundXi \\
    = & \alpha^2 \scondL\opnorm{\jxxab} \opnorm{A_{\alpha,\beta}^{-1}}\boundXi + \alpha e^{-c\alpha k} \scondL \boundXi\\
    &+ \alpha c^{-1}\scondL\opnorm{\jxxab} \opnorm{A_{\alpha,\beta}^{-1}}\boundXi + \alpha \opnorm{A_{\alpha,\beta}^{-1}} \boundXi 
\end{align*}
where in the second inequality we used $\opnorm{D} \leq \scondL \LyapFNorm{D}$. 
\newcontent{
The last term \eqref{eq: poisson fast 4 (bias)} is capturing the drift in the kernel and is bounded in the same way as \eqref{eq: poisson fast 4} by 
\begin{align*}
    \norm{\expvf{s+1}{s+1}{s} - \expvf{s}{s+1}{s}} \leq \boundXi \boundXY \Lk |\Xi| (\alpha \boundF + \beta \boundG)
\end{align*}
which yields the bound 
\begin{align*}
    \alpha \norm{ \sum_{s=0}^k (I- \alpha \jxxab)^{k-s}\eqref{eq: poisson fast 4 (bias)}} \leq \scondL \boundXY \boundXi \Lk |\Xi|c^{-1}(\alpha \boundF + \beta \boundG)
\end{align*}
}
We pack all the bounds together, under our assumptions on the step sizes there is $C_A, C_J$ that satisfy $\opnorm{A_{\alpha,\beta}^{-1}} \leq C_A$ and  $\opnorm{\jxxab} \leq C_J$ so we have: 
\begin{align*}
    \eqref{eq: sum of fast contracted markovian sum 1} \leq & \alpha \scondL \underbrace{\parentheses{c^{-1} \boundF (\Lvf+  \newcontent{\boundXY \boundXi \Lk |\Xi|}) +\boundXi\parentheses{1+C_A \parentheses{c^{-1}C_J+\scondL^{-1} }}}}_{\triangleq A_1}
    \\& + \beta \underbrace{\parentheses{c^{-1}\scondL \boundG(\Lvf+  \newcontent{\boundXY \boundXi \Lk |\Xi|})}}_{\triangleq A_2}  + \alpha ^2 \underbrace{\scondL \parentheses{ \Lvf \boundF  + C_J C_A\boundXi}}_{\triangleq A_3}\\
    & + \alpha \beta \underbrace{\parentheses{ \scondL \Lvf  \boundG} }_{\triangleq A_4}
\end{align*}
The second part of \eqref{eq:fast-bias-term-3} concerns bounding the term: 
\begin{align}
    \beta \norm{\sum_{j=0}^k (I - \alpha \jxxab)^{k-j}  U\Expect{\varphi_j}} \label{eq: sum of fast contracted markovian sum 2}
\end{align}
that is similar to the first one, except that now we have a multiplication by $\beta$ and the deviations concern the dynamics of the slow variable. There exists a function $v_g$ solution of the Poisson equation: 
\begin{align*}
    \vg{j} - \newcontent{\expvg{j}{j+1}{j}}= \varphi_j
\end{align*}
so we can write: 
\begin{align}
    \varphi_j =& v_g(x_j, y_j, \xi_{j+2}) - \newcontent{\expvg{j+1}{j+1}{j}} \label{eq: poisson slow 1 (bias)}\\
    &+  v_g(x_{j+1}, y_{j+1}, \xi_{j+2}) -  v_g(x_j, y_j, \xi_{j+2}) \label{eq: poisson slow 2 (bias)}\\
    &  + \vg{j} - v_g(x_{j+1}, y_{j+1}, \xi_{j+2}) \label{eq: poisson slow 3 (bias)}\\
    & +\newcontent{\expvg{j+1}{j+1}{j}}  - \newcontent{\expvg{j}{j+1}{j}} \label{eq: poisson slow 4 (bias)}
\end{align}
where term \eqref{eq: poisson slow 1 (bias)} is an MDS, under expectation it will equal zero. The second term \eqref{eq: poisson slow 2 (bias)}, similarly to the previous part, we set $v_g(x_j,y_j, \xi_{j+2}) = v_j$ and $v_g(x_{j+1},y_{j+1}, \xi_{j+2}) = v_{j+1}$ we have 

\begin{align*}
    \beta \norm{\sum_{j=0}^k (I - \alpha \jxxab)^{k-j}  U(v_j - v_{j+1})}\leq & \beta^2  \scondL \opnorm{U}\Lvg  \boundG + \alpha \beta  \scondL \opnorm{U}\Lvg  \boundF  \\
    &+ {\beta^2  \over c\alpha } \scondL \opnorm{U}\Lvg  \boundG+ \beta  c^{-1} \scondL \opnorm{U}\Lvg  \boundF
\end{align*}
for the third term \eqref{eq: poisson slow 3 (bias)} it will follow the same bound as the one used in \eqref{eq: poisson fast 3 (bias)}, we just need to multiply it by $\beta$, we set $\vg{j}= v_j$, $v_g(x_{j+1}, y_{j+1}, \xi_{j+2}) = v_{j+1}$:
\begin{align*}
    \beta \norm{\sum_{j=0}^k A_{\alpha,\beta}^{k-j}U(v_j- v_{j+1})} \leq  & \alpha \beta \scondL\opnorm{\jxxab}\opnorm{U} \opnorm{A_{\alpha,\beta}^{-1}}\boundXi + \beta e^{-c\alpha k} \scondL \opnorm{U}\boundXi\\
    &+ \beta c^{-1}\scondL\opnorm{\jxxab}\opnorm{U} \opnorm{A_{\alpha,\beta}^{-1}}\boundXi + \beta \opnorm{A_{\alpha,\beta}^{-1}} \opnorm{U}\boundXi 
\end{align*}
\newcontent{Finally, we bound the term  
\begin{align*}
    \beta \norm{ \sum_{s=0}^k (I- \alpha \jxxab)^{k-s} U \eqref{eq: poisson slow 4 (bias)}} \leq {\beta \over \alpha } \opnorm{U} \scondL \boundXY \boundXi \Lk |\Xi|c^{-1}(\alpha \boundF + \beta \boundG)
\end{align*}}
Putting all the bounds together, we get: 
\begin{align*}
    \eqref{eq: sum of fast contracted markovian sum 2}\leq & \beta \underbrace{ \opnorm{U} \scondL  \parentheses{ c^{-1} (\Lvg + \newcontent{\boundXY \boundXi \Lk |\Xi|})  \boundF + \boundXi\parentheses{ C_A\parentheses{c^{-1}C_J + \scondL ^{-1}} + 1  } }}_{\triangleq B_1}\\
    & + \beta^2  \underbrace{\scondL \opnorm{U}\Lvg  \boundG }_{\triangleq B_2}+ \alpha \beta \underbrace{\scondL \parentheses{C_J\opnorm{U} C_A\boundXi +  \opnorm{U}\Lvg  \boundF }}_{\triangleq B_3}  \\
    & + {\beta^2  \over \alpha } \underbrace{\scondL \opnorm{U} \boundG c^{-1}(\Lvg + \newcontent{\boundXY \boundXi \Lk |\Xi|})}_{\triangleq B_4}
\end{align*}
We merge the major two bounds together to get: 
\begin{align*}
\eqref{eq: sum of fast contracted markovian sum 1} +\eqref{eq: sum of fast contracted markovian sum 2} \leq &  \alpha A_1 +\beta (A_2 + B_1)+ \alpha ^2 A_3+ \beta^2 B_2 + \alpha \beta (A_4 + B_3) +{\beta ^2 \over \alpha} B_4
\end{align*}
ending the first part of the proof. Now we start with the second part to bound \eqref{eq:slow-bias-term-3}. We start by the first term: 
\begin{align}
    \alpha \norm{\sum_{j=0}^k  (I - \beta \jyyab)^{k-j} X\Expect{\psi_j}}
    \label{eq: poisson sum contraction cross}
\end{align}
as we have discussed above about rewriting the inside terms with the solutions of the Poisson equation, we set $v_f(x_s,y_s, \xi_{s+2}) = v_ s$ and $v_f(x_{s+1},y_{s+1}, \xi_{s+2}) = v_{s+1}$ we have:

\begin{align*}
    \alpha \norm{\sum_{s=0}^k (I- \beta  \jyyab)^{k-s} X (v_s - v_{s+1})}\leq & \alpha \sum_{s=0}^k \norm{I- \beta \jyyab}^{k-s} \norm{X} \norm{v_s - v_{s+1}}\\
    \leq & \alpha\scondG\norm{X} \sum_{s=0}^k e^{-d\beta (k-s)}\norm{v_s - v_{s+1}}\\
    \leq & \alpha\scondG\Lvf\norm{X} (\alpha \boundF + \beta \boundG)  ( 1 + {1 \over d\beta })\\
    \leq & \beta \radiusX   \scondG \Lvf(\alpha \boundF + \beta \boundG)  ( 1 + {1 \over d\beta })\\
    = & \alpha \beta \radiusX   \scondG \Lvf  \boundF  + \beta^2 \radiusX   \scondG \Lvf  \boundG  \\
    & + \alpha  d^{-1} \radiusX   \scondG \Lvf  \boundF + \beta d^{-1} \radiusX   \scondG \Lvf  \boundG
\end{align*}
where in the fourth inequality we have used the bound \eqref{eq:alpha-X-bound}, now we target the almost telescoping term. We set $\vf{s}= v_s$, $v_f(x_{s+1}, y_{s+1}, \xi_{s+2}) = v_{s+1}$, $\jyyab = \Delta -  J_{yx}(U-U_0) $ and $B_{\alpha,\beta} = I - \beta \jyyab$. \newcontent{Since we are assuming 
\begin{align}
\beta < {1 \over \opnorm{\Delta} +\radiusU \opnorm{\J{yx} }} \label{eq: beta invertibility condition for slow bias}
\end{align}
then $B_{\alpha,\beta}$ is invertible using the Neumann series, we use the inverse to have:}
\begin{align*}
    \sum_{s=0}^k B_{\alpha,\beta}^{k-s}X(v_s- v_{s+1}) =& \sum_{s=0}^k B_{\alpha,\beta}^{k-s}Xv_s- B_{\alpha,\beta} B_{\alpha,\beta}^{k-(s+1)}Xv_{s+1}\\
    =&\sum_{s=0}^k B_{\alpha,\beta}^{k-s}Xv_s- B_{\alpha,\beta}^{k-(s+1)}Xv_{s+1}+ \beta \jyyab B_{\alpha,\beta}^{-1} \sum_{s=0}^{k}  B_{\alpha,\beta}^{k-s}Xv_{s+1}\\
    = & B_{\alpha,\beta}^k X v_0 - B_{\alpha,\beta}^{-1}Xv_{k+1} + \beta \jyyab B_{\alpha,\beta}^{-1} \sum_{s=0}^{k}  B_{\alpha,\beta}^{k-s}Xv_{s+1}
\end{align*}
applying the norm: 
\begin{align*}
    \alpha \norm{\sum_{s=0}^k B_{\alpha,\beta}^{k-s}X(v_s- v_{s+1})} \leq & \alpha \opnorm{B_{\alpha,\beta}^k} \opnorm{ X} \norm{ v_0}  + \alpha \opnorm{B_{\alpha,\beta}^{-1}} \opnorm{X}\norm{v_{k+1}} \\
    & + \alpha \beta \opnorm{\jyyab} \opnorm{ B_{\alpha,\beta}^{-1}} \opnorm{X}\sum_{s=0}^{k}  \opnorm{B_{\alpha,\beta}}^{k-s}\norm{v_{s+1}}\\
    \leq &  \beta \scondG \radiusX e^{-d\beta k} \boundXi + \beta \scondG \radiusX  \norm{B_{\alpha,\beta}^{-1}}  \boundXi \\
    & +  \beta^2 \scondG \radiusX  \norm{\jyyab} \norm{ B_{\alpha,\beta}^{-1}} \boundXi \sum_{s=0}^{k}  e^{-d\beta (k-s)}\\
    \leq &  \beta \radiusX  \scondG \boundXi \parentheses{ e^{-d\beta k} +  \norm{B_{\alpha,\beta}^{-1}} + \norm{\jyyab} \norm{ B_{\alpha,\beta}^{-1}} d^{-1}} \\
    & \beta^2 \radiusX  \scondG \boundXi \norm{\jyyab} \norm{ B_{\alpha,\beta}^{-1}} 
\end{align*}
and we get our bound for the first part, under our assumptions on the step sizes there is $C_A, C_J$ that satisfy $\norm{ B_{\alpha,\beta}^{-1}} \leq C_B$ and  $\opnorm{\jyyab} \leq C_J$ so we have: 
\begin{align*}
    \eqref{eq: poisson sum contraction cross}\leq & \beta \underbrace{\radiusX  \scondG \parentheses{d^{-1}(\Lvf + \newcontent{\boundXY \boundXi \Lk |\Xi|}) \boundG  + \boundXi \parentheses{ 1 +  C_B + \norm{\jyyab} C_B d^{-1}}}}_{\triangleq C_1} \\
    & + \alpha  \underbrace{\radiusX   \scondG (\Lvf + \newcontent{\boundXY \boundXi \Lk |\Xi|})  \boundF d^{-1} }_{\triangleq C_2}+ \beta^2 \underbrace{ \radiusX  \scondG \parentheses{\boundXi C_J C_B +  \Lvf  \boundG }}_{\triangleq C_3}  \\
    & + \alpha \beta  \underbrace{ \radiusX   \scondG \Lvf  \boundF}_{\triangleq C_4}
\end{align*}
our last bound is for the term : 
\begin{align}
    \beta \norm{\sum_{j=0}^k  (I - \beta \jyyab)^{k-j} I'\Expect{\varphi_j}} \label{eq: poisson sum contraction slow}
\end{align}
from the Poisson decomposition the first term will equal zero under expectation. For the second term \eqref{eq: poisson slow 2 (bias)}, similarly to the previous part, we set $v_g(x_j,y_j, \xi_{j+2}) = v_j$ and $v_g(x_{j+1},y_{j+1}, \xi_{j+2}) = v_{j+1}$ we have:
\begin{align*}
    \beta \norm{\sum_{j=0}^k (I - \beta \jyyab)^{k-j}  I'(v_j - v_{j+1})}
    \leq & \alpha \beta \opnorm{I'} \scondG \Lvg  \boundF + \beta d^{-1} \opnorm{I'} \scondG \Lvg  \boundG \\
    & + \alpha  d^{-1}\opnorm{I'} \scondG \Lvg  \boundF + \beta^2 \opnorm{I'}\scondG \Lvg  \boundG 
\end{align*}
and the last term, we set $v_j = \vg{j}, v_{j+1} = v_g(x_{j+1}, y_{j+1},\xi_{j+2})$:
\begin{align*}
    \beta \norm{\sum_{s=0}^k B_{\alpha,\beta}^{k-s}I'(v_s- v_{s+1})} \leq &  \beta \opnorm{I'} \boundXi \parentheses{ \scondG e^{-d\beta k} +  \opnorm{B_{\alpha,\beta}^{-1}} + \scondG\opnorm{\jyyab} \opnorm{ B_{\alpha,\beta}^{-1}} d^{-1}  }\\
    & +  \beta^2 \opnorm{I'} \boundXi \scondG\opnorm{\jyyab} \opnorm{ B_{\alpha,\beta}^{-1}}
\end{align*}
which implies the bound on the term: 
\begin{align*}
    \eqref{eq: poisson sum contraction slow} \leq &  \beta \underbrace{ \opnorm{I'} \boundXi \parentheses{ C_B + \scondG \parentheses{ 1 + d^{-1} \parentheses{C_J C_B +\opnorm{I'} (\Lvg + \newcontent{\boundXY \boundXi \Lk |\Xi|}) \boundG }}}}_{\triangleq D_1}\\
    & +  \beta^2 \underbrace{\parentheses{\opnorm{I'} \boundXi \scondG C_J C_B+\opnorm{I'}\scondG \Lvg  \boundG }}_{\triangleq D_2}\\
    & + \alpha \beta  \underbrace{\opnorm{I'} \scondG \Lvg  \boundF}_{\triangleq D_3}  + \alpha \underbrace{ d^{-1}\opnorm{I'} \scondG (\Lvg + \newcontent{\boundXY \boundXi \Lk |\Xi|})  \boundF }_{\triangleq D_4}
\end{align*}
which implies the following bound: 
\begin{align*}
     \eqref{eq: poisson sum contraction cross} + \eqref{eq: poisson sum contraction slow} \leq \alpha (C_2 + D_4) + \beta (C_1 + D_1 )+ \alpha \beta (C_4 + D_3) + \beta^2 (C_3 + D_2) 
\end{align*}
which concludes the proof.
\end{proof}
Lastly, this lemma is needed to bound the weighted sum of the MSEs of the fast and slow iterates. We need this result in the proof of Theorem~\ref{thm : bias}.
\begin{lemma}\label{lemma : Exp Sum Of the MSEs} Taking the sequence $\{x_k,y_k\}_{k\geq 0}$ defined by \eqref{eq: TTSA update} and under assumptions A.\ref{assumption: Boundedness}-\ref{assumption: Lip and Mono of g}, let $q,\gamma$ be positive real numbers, we have: 
\begin{align*}
    \sum_{s=0}^k e^{-q\gamma(k-s)} \Expect{ \norm{\y{s}}^2+\norm{\x{s}}^2}  \leq\parentheses{1 + {1\over q\gamma}} \parentheses{\MSEx + \MSEy}+ u_k^{q,\gamma}+w_k^{q,\gamma}
\end{align*}
\end{lemma}
where $\MSEx$ and $\MSEy$ are defined in Theorem \ref{thm : mse} and 
\[
u_k^{q,\gamma}= \norm{\x{0}}^2 e^{-kq\gamma}\sum_{s=0}^{k}(e^{q\gamma}(1-\alpha \mu_f))^{s}
\]
\[
 w_k^{q,\gamma}=\sum_{s=0 }^k e^{-q \gamma (k-s)} \parentheses{1 - \beta \mug}^s \norm{\y{0}}^2 + \sum_{s=0 }^k e^{-q \gamma (k-s)} D_1 \beta s (1 - \alpha \muf \wedge \beta \mug )^s\norm{\x{0}}^2
\]
are vanishing sequences as $k \to \infty$.
\begin{proof}
First, we start by analyzing the fast iterate, we invoke the result of Theorem \ref{thm : mse} to bound its MSE, we start by the sum of MSEs of the fast iterate:
\begin{align}
\sum_{s=0}^k e^{-q\gamma(k-s)}  \Expect{\norm{\x{s}}^2}  \leq   & \sum_{s=0}^{k}e^{-q\gamma(k-s)}  \left ( (1-\alpha \muf)^{s}\norm{\x{0}}^2 + \MSEx  \right ) \nonumber \\
  \leq & \sum_{s=0}^{k}e^{-q\gamma(k-s)}  \left ( (1-\alpha \mu_f)^{s} \norm{\x{0}}^2\right ) +  \MSEx   \parentheses{1 +\int_0^\infty e^{-q\gamma t}dt} \nonumber \\
  = & \sum_{s=0}^{k}e^{-q\gamma(k-s)} \left ( (1-\alpha \mu_f)^{s} \norm{\x{0}}^2\right ) +  \MSEx  \parentheses{1+{1 \over q\gamma }} \nonumber \\
 \leq &  \norm{\x{0}}^2 e^{-kq\gamma}\sum_{s=0}^{k}\parentheses{e^{q\gamma}(1-\alpha \mu_f)}^{s} \label{eq:sum-mse-fast-1}\\
& +  \parentheses{1+{1 \over q\gamma }}\MSEx \nonumber
\end{align}
Now from our assumptions we have $ 1- \alpha\muf \in (0,1] $. If $e^{q\gamma}(1-\alpha \mu_f) = 1 $ then the sum equals $e^{-kq\gamma }\sum_{s=0}^{k}(e^{q\gamma}(1-\alpha \mu_f))^{s} =e^{-kq\gamma }(k+1)$ which will go to 0 as $k\to\infty$.  For the other case $e^{q\gamma}(1-\alpha \mu_f) \neq 1$ so the sum : 
\[
e^{-kq\gamma }\sum_{s=0}^{k}(e^{q\gamma}(1-\alpha \mu_f))^{s} = {e^{-kq\gamma} - e^{q\gamma} (1-\alpha \muf)^{k+1} \over 1 - e^{q\gamma} (1-\alpha \muf)}
\]
which will also disappear as k goes to infinity. So we set the sequence $u_k^{q,\gamma}$ to be defined by:
\[
u_k^{q,\gamma} = \norm{\x{0}}^2 e^{-kq\gamma }\sum_{s=0}^{k}\parentheses{e^{q\gamma}(1-\alpha \mu_f)}^{s}
\]
such that 
\[
\lim_{k\rightarrow \infty} u_k^{q,\gamma} = 0
\]
Same reasoning goes with the MSE of the slow iterate, from Theorem \ref{thm : mse} following the same steps in the fast MSE above we get:
\begin{align}
\sum_{s=0}^k e^{-q\gamma(k-s)} \Expect{\norm{\y{s}}^2}  & \leq  \sum_{s=0 }^k e^{-q \gamma (k-s)} \parentheses{1 - \beta \mug}^s \norm{\y{0}}^2 \nonumber\\
& + \sum_{s=0 }^k e^{-q \gamma (k-s)} D_1 \beta s (1 - \alpha \muf \wedge \beta \mug )^s\norm{\x{0}}^2\label{eq:sum-mse-slow-2}\\
& + \sum_{s=0 }^k e^{-q \gamma (k-s)} \MSEy \label{eq:sum-mse-slow-3}
\end{align}
The first term follows the same reasoning as the term \eqref{eq:sum-mse-fast-1}, we focus on the second term \eqref{eq:sum-mse-slow-2}: 
\begin{align*}
    \eqref{eq:sum-mse-slow-2} \leq & \norm{\x{0}}^2 D_1 \beta ke^{-q \gamma k}  \sum_{s=0 }^k \parentheses{ e^{q \gamma}  (1 - \alpha \muf \wedge \beta \mug )}^s
\end{align*}
if $e^{q \gamma}  (1 - \alpha \muf \wedge \beta \mug ) =1$:
\begin{align*}
    ke^{-q \gamma k}  \sum_{s=0 }^k \parentheses{ e^{q \gamma}  (1 - \alpha \muf \wedge \beta \mug )}^s = e^{-q\gamma k} k(k+1)
\end{align*}
which vanishes as $k\to \infty$.

otherwise, if $e^{q \gamma}  (1 - \alpha \muf \wedge \beta \mug ) \neq1$:
\begin{align*}
    ke^{-q \gamma k}  \sum_{s=0 }^k \parentheses{ e^{q \gamma}  (1 - \alpha \muf \wedge \beta \mug )}^s = {k e^{-q \gamma k} - k e^{q\gamma} (1 - \beta \mug \wedge \alpha \muf)^{k+1}\over 1 -e^{q\gamma} (1 - \beta \mug \wedge \alpha \muf) }
\end{align*}
which also vanishes as $k\to \infty$. We define the new sequence  $w_k^{q,\gamma}$ as:
\[
 w_k^{q,\gamma}=\sum_{s=0 }^k e^{-q \gamma (k-s)} \parentheses{1 - \beta \mug}^s \norm{\y{0}}^2 + \sum_{s=0 }^k e^{-q \gamma (k-s)} D_1 \beta s (1 - \alpha \muf \wedge \beta \mug )^s\norm{\x{0}}^2
\]
such that 
\[
\lim_{k\rightarrow \infty} w_k^{q,\gamma}= 0
\]
and the last term \eqref{eq:sum-mse-slow-3} is bounded by 
\begin{align*}
    \eqref{eq:sum-mse-slow-3} \leq \parentheses{1 + {1\over q\gamma}} \MSEy
\end{align*}
now we can write the final bound:
\begin{align*}
    \sum_{s=0}^k e^{-q\gamma(k-s)} \Expect{ \norm{\y{s}}^2+\norm{\x{s}}^2}  \leq \parentheses{1 + {1\over q\gamma}} \parentheses{\MSEx + \MSEy}+ u_k^{q,\gamma}+w_k^{q,\gamma}
\end{align*}
\end{proof}

\section{Illustrations}\label{apx : Additional illustrations}

\subsection{First example}

\begin{figure}[b]
    \centering
    \begin{tabular}{@{}c@{}c@{}c@{}}
        \includegraphics[width=0.32\linewidth]{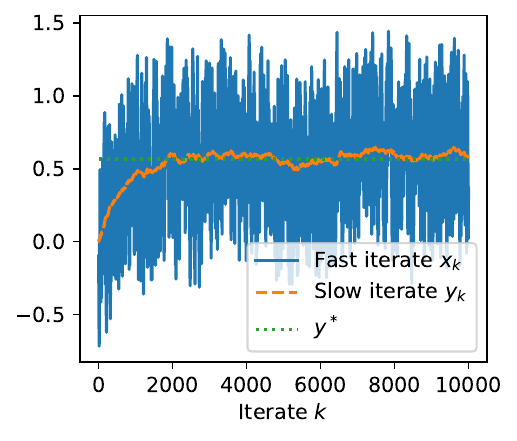}
        &\includegraphics[width=0.32\linewidth]{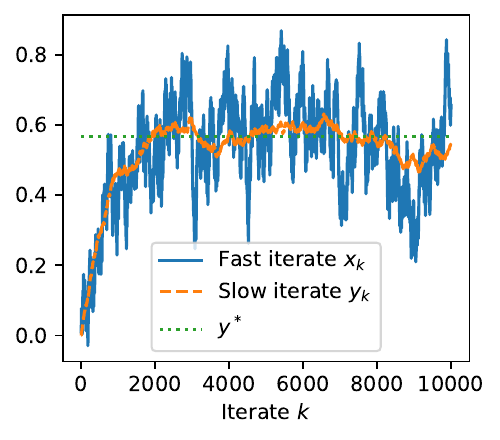}
        &\includegraphics[width=0.32\linewidth]{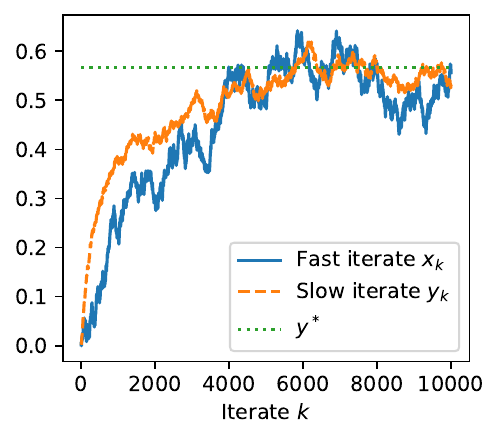}\\
        (a) $\alpha=0.1$, $\beta=0.001$
        &(b)  $\alpha=0.01$, $\beta=0.001$
        &(c) $\alpha=0.002$, $\beta=0.001$
    \end{tabular}
    \caption{Sample trajectory of the TTSA defined in \eqref{eq:example_smooth_apx}.}
    \label{fig:example1-traj}
\end{figure}

We investigate the empirical performance of a TTSA where both variables are $1$-dimensional: $x_k\in\Reals$ and $y_k\in\Reals$ and satisfy the following recursion
\begin{align}
    \label{eq:example_smooth_apx}
    \begin{aligned}
        x_{k+1} &= x_k -\alpha (x_k - y_k + \xi^x_{k+1} )\\
        y_{k+1} &= y_k -\beta (2 y_k - e^{-y_k} - x_k + \xi^y_{k+1} ),
    \end{aligned}
\end{align}
where the noise terms $\xi^x_{k}, \xi^y_k \in \{-1,1\}$ form the Markovian noise with transition probability kernel $\mathbb{P}[\xi_{k+1} = s' \mid \xi_k = s] = 0.3$ for $s \neq s'$. Initially, we set $x_0=y_0=0$.  For this example, $h(y)=y$ and $y^*$ is the solution of $y=\exp(-y)$ which is equal to $W_0(1)\approx0.567$, where $W_0$ is the standard Lambert function.

In Figure~\ref{fig:example1-traj}, we display a trajectory of the TTSA for three values of $\alpha\in\{0.1, 0.01, 0.002\}$ and a fixed $\beta=0.001$. We observe two phenomena that are shown in Theorem~\ref{thm : mse}:
\begin{enumerate}
    \item First, the vanishing term vanishes at a slower rate as $\alpha$ gets smaller. For instance, when $\alpha=0.002$ (Figure~\ref{fig:example1-traj}(c)), the fast variable is lagging behind $h(y_k)=y_k$ where when $\alpha=0.01$ (Figure~\ref{fig:example1-traj}(b)), the variable $x_k$ has no lag $y_k$.
    \item Second, the larger is $\alpha$, the larger are the oscillations of $x_k$ and $y_k$: in Figure~\ref{fig:example1-traj}(a), the variable $x_k$ is not lagging behind $h(y_k)=y_k$ but it is oscillating a lot. 
\end{enumerate} 

In Figure~\ref{fig:example1-MSE-bias}, we focus on the asymptotic behavior of the MSE and of the bias as a function of the step size $\alpha$ and for the case where there is a time scale separation $\beta=\alpha^{1.5}$.  For each value, we simulate a trajectory of length $k=1.2\cdot 10^5$ steps and compute the average values of $x_k$ for the last half of the trajectory. We perform $5000$ independent trajectories to obtain the confidence intervals (they are shown on both panels but are too small to be visible on the MSE plot in Figure~\ref{fig:example1-MSE-bias}(a)). This figure shows that for this example,  $\mathrm{MSE^x}$ seems to be of the order of $\alpha$ whereas $\mathrm{MSE^y}$ and the bias terms are of the order of $\beta=\alpha^{1.5}$. 

\begin{figure}[ht]
    \centering
    \begin{tabular}{cc}
        \includegraphics[width=0.45\linewidth]{Figures/MSE_simu.pdf}
        &\includegraphics[width=0.45\linewidth]{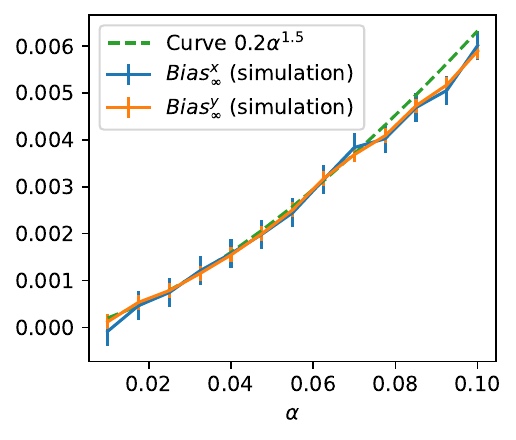}\\
        (a) MSE as a function of $\alpha$ (for $\beta=\alpha^{3/2}$)
        & (b) Bias as a function of $\alpha$ (for $\beta=\alpha^{3/2}$)\\
    \end{tabular}
    \caption{MSE and Bias of the TTSA defined in \eqref{eq:example_smooth_apx} for $\beta=\alpha^{3/2}$}
    \label{fig:example1-MSE-bias}
\end{figure}

\subsection{Example showing a lower bound in $\Omega(\beta^2/ \alpha^2)$}

The example used in the paper (Figure~\ref{fig:various illustrations}(c)) to illustrate the $\Omega(\beta^2/\alpha^2)$ term in the MSE is a TTSA whose recurrence equation is given by
\begin{align}
    \label{eq: beta^2/alpha^2}
    \begin{aligned}
        x_{k+1} &= x_k - \alpha (x_k-y_k)\\
        y_{k+1} &= y_k - \beta \frac{(x_k - y_k)}{|x_k-y_k|^\gamma}, 
    \end{aligned}
\end{align}
with $x_0=0$ and $y_0=1$.

In this example, $x_k$ chases the variable $y_k$ ($h(y_k)=y_k$), while $y_k$ tries to escape $x_k$. This example can be easily analyzed:  Rewriting $\delta_k\triangleq y_k-x_k$, one obtains 
\newcontent{
\begin{align*}
    \delta_{k+1} = (1-\alpha)\delta_k + \beta \frac{(\delta_k)}{|\delta_k|^\gamma}
\end{align*}}
As $k$ goes to infinity, this shows that $\lim_{k\to\infty}\delta_k =(\beta/\alpha)^{1/\gamma}$. This shows that:
\begin{align*}
    \lim_{k\to\infty}\mathrm{MSE^x} = \left(\frac{\beta}{\alpha}\right)^{2/\gamma}.
\end{align*}
When $\gamma$ is close to $1$, this is essentially equal to $(\beta/\alpha)^2$. This is what is illustrated in Figure~\ref{fig:1D_lower bound}: for this example, the quantity $\mathrm{MSE}^x$ does not depend on $\alpha$ and is equal to $(\beta/\alpha)^{2/\gamma}$ for all $\alpha$, $\beta$ and $\gamma$. 

\begin{figure}[ht]
    \centering
    \begin{tabular}{ccc}
        \includegraphics[width=0.3\linewidth]{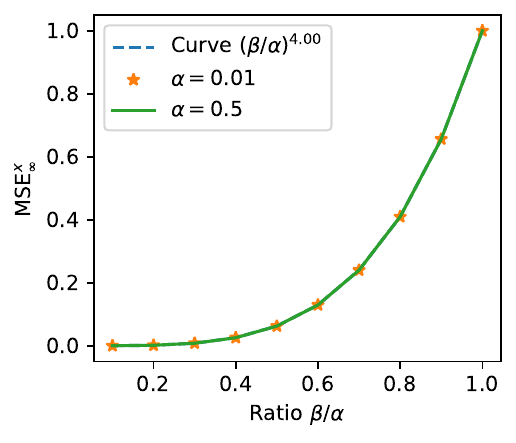}
        &\includegraphics[width=0.3\linewidth]{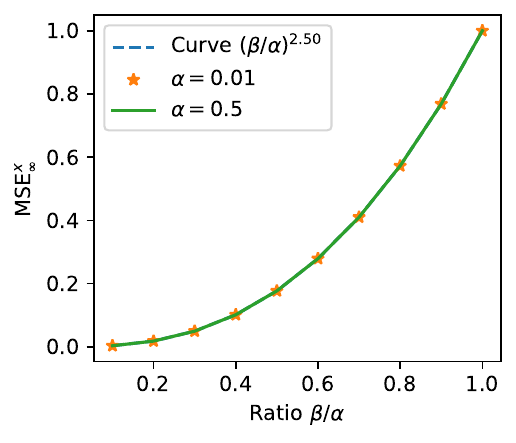}
        &\includegraphics[width=0.3\linewidth]{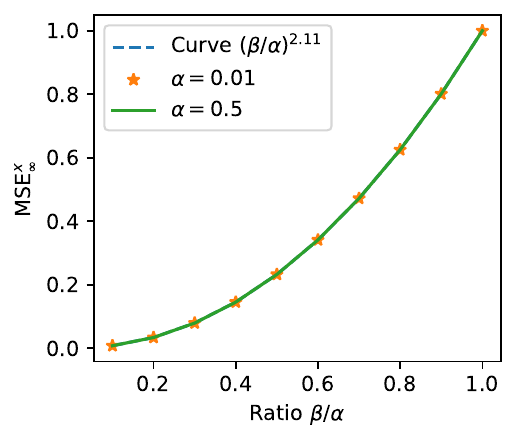}\\
        (a) $\gamma=0.5$
        &(b) $\gamma=0.8$
        &(c) $\gamma=0.95$
    \end{tabular}
    \caption{$\mathrm{MSE^x}$ of the TTSA given by \eqref{eq: beta^2/alpha^2} as a function of the ratio $\beta/\alpha$. We observe that $\mathrm{MSE}^x$ does not depend on $\alpha$ and is equal to $(\beta/\alpha)^{2/\gamma}$ for all $\alpha$, $\beta$ and $\gamma$.  }
    \label{fig:1D_lower bound}
\end{figure}

\paragraph{Discussion on the assumptions} The TTSA \eqref{eq: beta^2/alpha^2} \newcontent{satisfies the assumptions about the kernel and the fast iterate, since here we care about the tracking error of the fast iterate that is of order $\beta^2 \over \alpha^2$ and we got it from the MSE analysis of x without using information about Lipschitzness nor the strong monotonicity of $g$ so A.\ref{assumption: Lip and Mono of g} was not needed in that analysis. In addition to that we relax the assumption on the boundedness of the iterates as we could not find an example that guarantees boundedness.}

The fact that $x_k$ and $y_k$ do not live in a compact set is more problematic but can be solved by slightly modifying the above example. Note that this new example cannot be easily solved in closed form which is why we decided to keep the TTSA \eqref{eq: beta^2/alpha^2} as the main example. To construct this new example, we consider $2$-dimensional variables $x_k, y_k\in\Reals^2$ with $x_0=[1,0]$, $y_0=[1,.1]$ and
\begin{align}
    \label{eq:example circle}
    \begin{aligned}
        x_{k+1} &= x_k - \alpha (x_k - y_k + 0.1 \xi_{k+1})\\
        y_{k+1} &= y_k - \beta \left(\frac{(x_k - y_k)}{|x_k-y_k|^\gamma} + y_k \tanh(\norm{y_k}^2-1) \right),
    \end{aligned}
\end{align}
where $\xi_{k+1}$ is a sequence of Rademacher random variables. 

This example is essentially the same as the TTSA in \eqref{eq: beta^2/alpha^2} with two differences. First, we add a small white noise in the dynamics of $x_k$ (this does not change anything, but shows that the counter-example does not depend on the dynamics being deterministic). Second, and more importantly, the evolution of $y_k$ contains the term $y_k \tanh(\norm{y_k}^2-1)$. The role of this term is to push $y_k$ towards the unit circle. As an example, we display sample trajectories of $(x_k,y_k)$ in Figure~\ref{fig:dynamics on the circle} for various values of $\alpha$ and $\beta=0.1\alpha$. We observe that both iterates $x_k$ and $y_k$ oscillate on the circle (with more noise when $\alpha$ is larger).

\begin{figure}[ht]
    \centering
    \begin{tabular}{ccc}
        \includegraphics[width=0.3\linewidth]{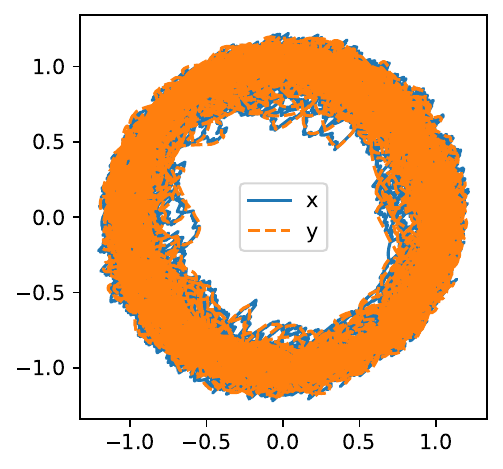}
        &\includegraphics[width=0.3\linewidth]{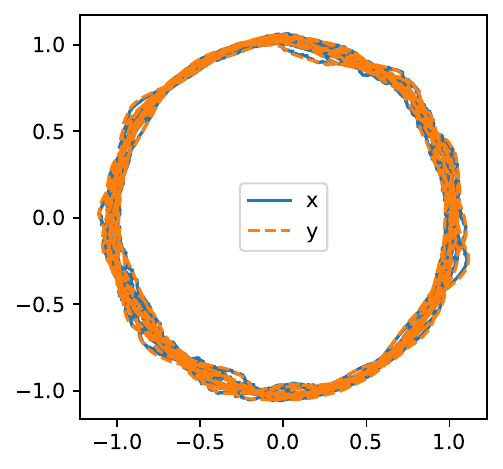}
        &\includegraphics[width=0.3\linewidth]{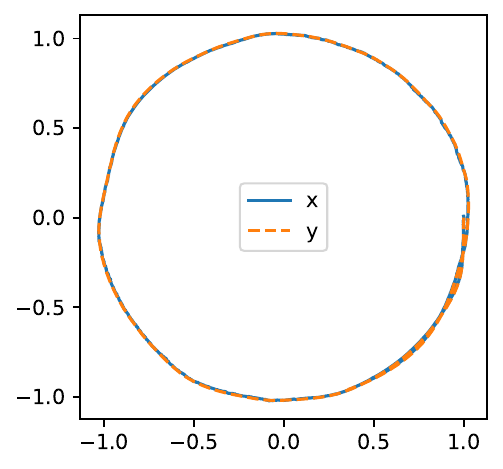}\\
        (a) $\alpha=0.1$
        &(b) $\alpha=0.01$, 
        &(c) $\alpha=0.001$, 
    \end{tabular}
    \caption{Sample trajectories of $x_k$ and $y_k$ for the TTSA \eqref{eq:example circle} for various values of $\alpha$ and $\beta=\alpha/10$}
    \label{fig:dynamics on the circle}
\end{figure}

For this dynamics, we again have $h(y_k)=y_k$. In Figure~\ref{fig:error on the circle}(a), we plot a sample trajectory of the tracking error $x_k-h(y_k)=x_k-y_k$ for various values of $\alpha$ while keeping the ratio $\beta/\alpha=0.1$ fixed.  We observe that, as $\alpha$ goes to $0$, the error does not vanish but seems to converge to a circular dynamics. Similarly to what was shown in Figure~\ref{fig:1D_lower bound} for the TTSA~\eqref{eq: beta^2/alpha^2}, we observe in Figure~\ref{fig:error on the circle}(b) that the MSE of $x$ grows like $(\beta/\alpha)^{2/\gamma}$, independently of $\alpha$. 
\begin{figure}
    \centering
    \begin{tabular}{cc}
        \includegraphics[width=0.4\linewidth]{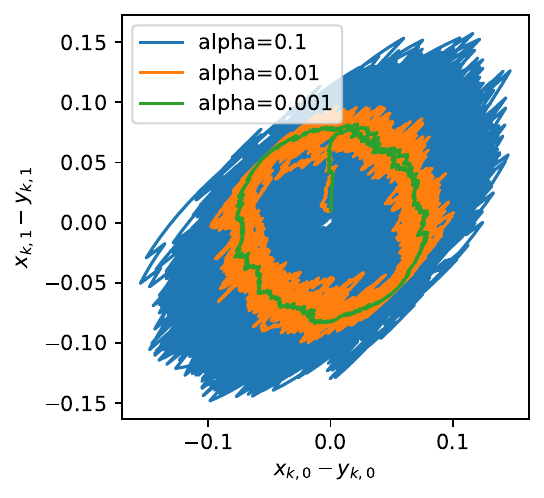}
        & \includegraphics[width=0.4\linewidth]{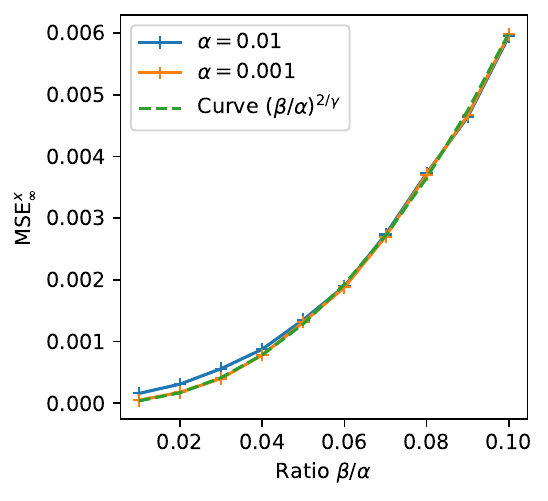} \\
        (a) Error $x_k-h(y_k)$ for $\beta=\alpha/10$
        & (b) MSE as a function of $\beta/\alpha$
    \end{tabular}
    \caption{Error and MSE of the TTSA  \eqref{eq:example circle}. }
    \label{fig:error on the circle}
\end{figure}

Note that this model does not satisfy the strong monotonicity assumption for $\bar{g}$ (A.\ref{assumption: Lip and Mono of g}). It is an open question if one can construct an example that satisfies A.\ref{assumption: Lip and Mono of g}, and whose asymptotic MSE for $x$ has a dependence on $(\beta/\alpha)$.

\end{document}